\documentclass{article} 
\usepackage{iclr2025_conference,times}

\usepackage{amsmath,amsfonts,bm}

\def\eqref#1{equation~\ref{#1}}

\def\1{\bm{1}}

\DeclareMathAlphabet{\mathsfit}{\encodingdefault}{\sfdefault}{m}{sl}
\SetMathAlphabet{\mathsfit}{bold}{\encodingdefault}{\sfdefault}{bx}{n}

\usepackage{hyperref}
\usepackage{url}

\usepackage{amsfonts}       
\usepackage{nicefrac}       
\usepackage{microtype}      
\usepackage{xcolor}         

\usepackage{microtype}
\usepackage{graphicx}
\usepackage{subfigure}
\usepackage{booktabs} 

\usepackage{multirow,colortbl}
\definecolor{grey}{rgb}{0.9,0.9,0.9}
\usepackage{wrapfig}

\usepackage{amsmath}
\usepackage{amssymb}
\usepackage{mathtools}
\usepackage{amsthm}
\usepackage{upgreek}

\usepackage[capitalize]{cleveref}
\usepackage{subcaption}
\usepackage{subfigure}
\usepackage{csquotes}

\usepackage{xspace}
\usepackage{hypcap} 
\usepackage{bbm}
\usepackage{amsthm}

\usepackage{textcase}

\theoremstyle{plain}
\newtheorem{theorem}{Theorem}
\newtheorem{proposition}{Proposition}

\theoremstyle{definition}

\theoremstyle{remark}

\newtheorem*{proposition*}{Proposition}
\newtheorem*{theorem*}{Theorem}

\graphicspath{{figure/}{figures/}}

\newcommand{\tabhead}[1]{{\bfseries#1}}
\newcommand{\md}{\text{d}}

\DeclareMathOperator{\lbint}{[\![}
\DeclareMathOperator{\rbint}{]\!]}
\DeclareMathOperator{\trainset}{\mathcal{D}_{\text{tr}}}

\DeclareMathOperator{\adaptset}{\mathcal{D}_{\text{ad}}}

\newcommand{\tone}{NCF-$\mathit{t}_1$\xspace}
\newcommand{\ttwo}{NCF-$\mathit{t}_2$\xspace}

\newcommand{\lv}{\textbf{LV}\xspace}
\newcommand{\go}{\textbf{GO}\xspace}
\newcommand{\sm}{\textbf{SM}\xspace}
\newcommand{\bt}{\textbf{BT}\xspace}
\newcommand{\gs}{\textbf{GS}\xspace}
\newcommand{\ns}{\textbf{NS}\xspace}

\usepackage{xfrac}
\usepackage{listings}
\usepackage{xcolor}

\usepackage{algorithm}
\usepackage{algorithmicx,algpseudocode}

\usepackage{multicol}

\algblockdefx{MRepeat}{EndRepeat}{\textbf{Repeat}}{}
\algnotext{EndRepeat}

\newcommand{\rebut}[1]{#1}
\definecolor{rebutcolor}{rgb}{0,0,0}

\definecolor{codegreen}{rgb}{0,0.6,0}
\definecolor{codegray}{rgb}{0.5,0.5,0.5}
\definecolor{codepurple}{rgb}{0.58,0,0.82}
\definecolor{backcolour}{rgb}{0.95,0.95,0.92}

\lstdefinestyle{mystyle}{
    backgroundcolor=\color{backcolour},   
    commentstyle=\color{codegreen},
    keywordstyle=\color{magenta},
    numberstyle=\tiny\color{codegray},
    stringstyle=\color{codepurple},
    basicstyle=\ttfamily\scriptsize,
    breakatwhitespace=false,         
    breaklines=true,                 
    captionpos=b,                    
    keepspaces=true,                 
    numbers=left,                    
    numbersep=5pt,                  
    showspaces=false,                
    showstringspaces=false,
    showtabs=false,                  
    tabsize=2,
}

\usepackage{float}

\usepackage{etoolbox}
\usepackage{pifont}
\usepackage{cancel}
\usepackage{placeins}
\usepackage{makecell}

\usepackage{tabularx,ragged2e}

\usepackage{titletoc}
\usepackage[subfigure]{tocloft}
\usepackage{textcase}

\newcommand\DoToC{%
  \startcontents
  \printcontents{}{1}{{\begin{center}\parbox{0.99\textwidth}{\centering\textsc{\textbf{\Large Supplementary Material for Neural Context Flows for Meta-Learning of Dynamical Systems}}}\end{center}\vskip3pt\hrule\vskip5pt}}
  \vskip4pt\hrule\vskip5pt
}

\renewcommand{\cite}{\citep}

\renewcommand*{\bibfont}{\small}

\title{Neural Context Flows for Meta-Learning of Dynamical Systems}

\author{Roussel Desmond Nzoyem \\
School of Computer Science\\
University of Bristol\\
\texttt{\footnotesize rd.nzoyemngueguin@bristol.ac.uk} \\
\And \hspace*{-0.6cm}
David A.W. Barton \\
\hspace*{-0.6cm} {\small School of Engineering Mathematics and Technology} \\
\hspace*{-0.6cm} University of Bristol\\
\hspace*{-0.6cm} \texttt{\footnotesize david.barton@bristol.ac.uk} \\
\And
Tom Deakin \\
School of Computer Science\\
University of Bristol\\
\texttt{\footnotesize tom.deakin@bristol.ac.uk} \\
}

\newcolumntype{K}[1]{>{\centering\arraybackslash}p{#1}}

\renewcommand{\cite}{\citep}

\iclrfinalcopy 
\begin{document}

\maketitle

\begin{abstract}
Neural Ordinary Differential Equations (NODEs) often struggle to adapt to new dynamic behaviors caused by parameter changes in the underlying physical system, even when these dynamics are similar to previously observed behaviors. This problem becomes more challenging when the changing parameters are unobserved, meaning their value or influence cannot be directly measured when collecting data. To address this issue, we introduce Neural Context Flow (NCF), a robust and interpretable Meta-Learning framework that includes uncertainty estimation. NCF uses Taylor expansion to enable contextual self-modulation, allowing context vectors to influence dynamics from other domains while also modulating themselves. After establishing theoretical guarantees, we empirically test NCF and compare it to related adaptation methods. Our results show that NCF achieves state-of-the-art Out-of-Distribution performance on 5 out of 6 linear and non-linear benchmark problems. Through extensive experiments, we explore the flexible model architecture of NCF and the encoded representations within the learned context vectors. Our findings highlight the potential implications of NCF for foundational models in the physical sciences, offering a promising approach to improving the adaptability and generalization of NODEs in various scientific applications.
\end{abstract}

\section{Introduction}
\label{introduction}

\begin{wrapfigure}[17]{r}{0.33\textwidth}
\vspace*{-0.60cm}
  \begin{center}
    \includegraphics[width=\linewidth]{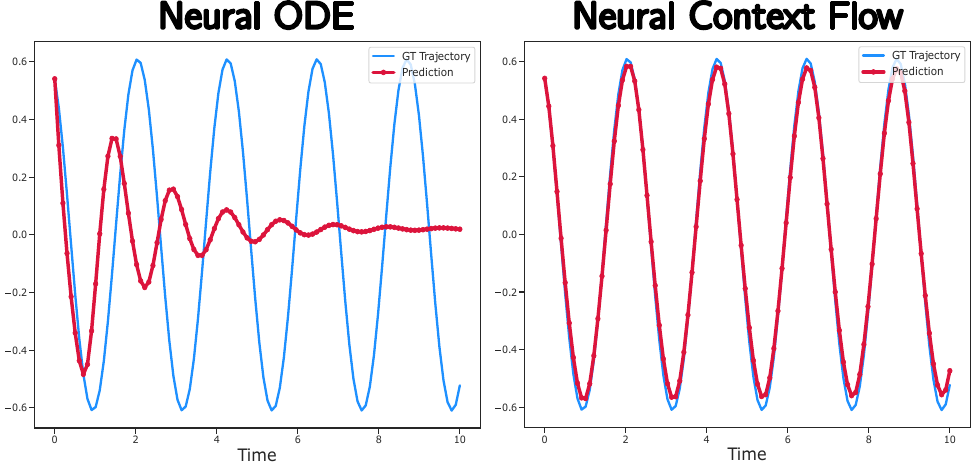}    
  \end{center}
\vspace*{-0.45cm}
  \caption{Predicted angle $\alpha$ for the simple pendulum $\frac{\md^2 \alpha}{\md t^2} + g \sin \alpha = 0 $. The OFA Neural ODE that disregards context fails to generalize its oscillation frequency $\sfrac{\sqrt{g}}{2 \pi}$ to unseen environments, due in part to merged inhomogeneous training data. Our work investigates Neural Context Flows and related Meta-Learning methods to overcome this issue.}
    \label{fig:trajscompare}
\end{wrapfigure}


A prototypical autonomous dynamical system describes the continuous change, through time $t \in \mathbb{R}^+$, of a quantity $x(t) \in \mathbb{R}^d$. Its dynamics are influenced by its parameters $c \in \mathbb{R}^{d_c}$, according to the (ordinary) differential equation
\begin{align} \label{eq:ode}
    \frac{\md x}{\md t}(t) = f(x(t), c),
\end{align}
where $f: \mathbb{R}^d \times \mathbb{R}^{d_c} \rightarrow \mathbb{R}^d $ is the \textbf{vector field}. Learning a dynamical system from data is synonymous with approximating $f$, a task neural networks have been remarkably good at in recent years \citep{chen2018neural,kidger2022neural}.

Consider the challenge of reconstructing the mechanical motion of an undamped pendulum given limited data from two distinct \textbf{environments}: Earth and Mars. Disregarding the physical variations between these environments (e.g., gravitational constants, ambient temperatures, etc.), one might employ a One-For-All (OFA) approach to learn a single environment-agnostic vector field from all available data \cite{yin2021leads}. This model would struggle to adequately fit such heterogeneous data and would face difficulties generalizing to data from novel environments, e.g., the Moon. \cref{fig:trajscompare} illustrates this issue, using gravity as the underlying physical parameter. Alternatively, learning individual vector fields for each environment with a One-Per-Env (OPE) approach would miss inter-domain commonalities, proving both time-intensive and inhibiting rapid adaptation to new environments \cite{yin2021leads}. Given these constraints, it becomes imperative to develop a methodology that can effectively \emph{learn what to learn} from the aggregate data while simultaneously accounting for the \emph{unique} properties of each environment.

In Scientific Machine Learning \cite{cuomo2022scientific} \rebut{(SciML)}, the problem of \textbf{generalization} has largely been tackled by injecting domain knowledge. It is commonly understood that adding a term in \cref{eq:ode} that captures as much of the dynamics as possible leads to lower evaluation losses \cite{yin2021augmenting}. For such terms to be added, however, it is essential to have knowledge of the physical parameters that change, which may then either be directly estimated, or predicted by a neural network within the vector field \cite{rackauckas2020universal}. We are naturally left to wonder how to efficiently learn a generalizable dynamical system when such physics are absent.


Under constantly changing experimental conditions, two major obstacles to learning the parameter dependence of a vector field can be \rebut{isolated}: \textbullet{} (P1) \textbf{Limited data} -- SciML models can be data-intensive \cite{hey2020machine,yin2021augmenting}, and limited data in each environment may not be enough to learn a vector field suitable for all environments; \textbullet{} (P2) \textbf{Unobserved parameters} -- during the data collection and modeling processes, one might be unfamiliar with the basic physics of the system.
Solving these two problems would contribute to the efficient generalization of the learned models, 
particularly to related but Out-of-Distribution (OoD) datasets. Fast OoD adaptation would massively reduce cost and complement recent efforts towards foundational models in the physical sciences \cite{subramanian2024towards,mccabe2023multiple,herde2024poseidon}.

Neural Ordinary Differential Equations (Neural ODEs) \cite{chen2018neural,weinan2017proposal,haber2017stable} have emerged as a \rebut{powerful backbone} for learning ordinary, stochastic, and controlled differential equations \cite{kidger2020neural}. Trained using Differentiable Programming techniques \cite{nzoyem2023comparison, ma2021comparison,rackauckas2020universal}, they have demonstrated \rebut{broad utility with outstanding results} in areas like chemical engineering \cite{owoyele2022chemnode}, geosciences \cite{shen2023differentiable}, and climate modeling \cite{kochkov2024neural}. \rebut{In the increasingly popular SciML subfield of solving parametric PDEs \cite{li2020fourier,takamoto2023learning,li2023identification,subramanian2024towards,koupai2024geps}, Neural ODEs occupy a place of choice due to their flexibility and performance \cite{yin2021leads,lee2021parameterized,kirchmeyer2022generalizing}. That said, existing methods seeking to generalize Neural ODEs to various parameter-defined environments fail to leverage information from environments other than the ones of interest. Not to mention the pervasive lack of interpretability and ways of accounting for the model's uncertainty.}

This work presents \textbf{Neural Context Flows} (NCFs), a novel approach for multi-environment generalization of dynamical systems based on Neural ODEs. By leveraging the regularity of the vector field with respect to unobserved parameters, NCFs parameterize an environment-agnostic vector field and environment-specific latent context vectors to modulate the vector field. The vector field is Taylor-expanded about these context vectors, effectively allowing information to flow between environments. Our contribution is threefold:

\begin{enumerate}
    \item[(1)] We introduce a Meta-Learning methodology for enhancing the generalizability of dynamical systems. Our approach effectively addresses problems (P1) and (P2), challenging the prevailing notion that standard Deep Learning techniques are inherently inefficient for Out-of-Distribution (OoD) adaptation \cite{mouli2023metaphysica}.

    \item[(2)] We present an \emph{interpretable} framework for Multi-Task representation learning, incorporating a straightforward method for \emph{uncertainty \rebut{estimation}}. This work extends the emerging trend of explainable linearly parameterized physical systems \cite{blanke2024interpretable} to non-linear settings, thus broadening its applicability. For affine systems, we provide a concise proof for the identifiability of their underlying parameters.

    \item[(3)] We provide a curated set of benchmark problems specifically designed for Meta-Learning of dynamical systems. This collection encompasses a diverse range of problems frequently encountered in the physical science literature, all accessible through a \emph{unified} interface.
\end{enumerate}


\section{Related Work}
\label{related}

Learning data-driven Neural ODEs that generalize across parameters is only a recent endeavor. To the best of the authors' knowledge, all attempts to solve (P1) and (P2)
have relied on Multi-Task and Meta-Learning to efficiently adapt to new parameter values, thus producing methods with varying levels of interpolation and extrapolation capabilities.

\paragraph{Multi-Task Learning.}
Multi-Task Learning (MTL) describes a family of techniques where a model is trained to jointly perform multiple tasks. In Scientific Machine Learning, one of the earliest methods to attack this generalization problem is LEADS \cite{yin2021leads}. In LEADS, the vector field is decomposed into shared dynamics $f_{\phi}$ and environment-specific $g^e_{\psi}$ components
\begin{align}
    \frac{\md x^e}{\md t} = f_{\phi}(x) + g^e_{\psi}(x),
\end{align}
where the superscript $e$ identifies the environment in which the dynamical system evolves, and $ \{\phi,\psi \}$ are learnable neural network weights.

While LEADS excels at interpolation tasks, it performs poorly during extrapolation \cite{kirchmeyer2022generalizing}. Furthermore, it requires retraining a new network $g^e_{\psi}$ each time a new environment is encountered, which can be constraining in scenarios where adaptation is frequently required. Before LEADS, other MTL approaches had been proposed outside the context of dynamical systems \cite{caruana1997multitask,rebuffi2017learning,rebuffi2018efficient,lee2021parameterized}. Still, they do not address the crucial adaptation to new tasks, which is our focus.


\paragraph{Meta-Learning.}
Another influential body of work looked at Meta-Learning: a framework in which, in addition to the MTL joint training scheme, shared representation is learned for rapid adaptation to unseen tasks with only minimum data \cite{wang2021bridging}. The seminal MAML \cite{finn2017model} popularized Gradient-Based Meta-Learning (GBML) by nesting an inner gradient loop in the typical training process. Since then, several variants aimed at avoiding over-fitting and reducing cost have been proposed, e.g. ANIL \cite{raghu2019rapid} and CAVIA \cite{zintgraf2019fast}. The latter is a contextual learning approach \cite{garnelo2018conditional} that partitions learnable parameters into some that are optimized on each environment, and others shared across environments, i.e., \emph{meta-trained}.

DyAd \cite{wang2022meta} is one of the earliest Meta-Learning approaches to target generalizable dynamical systems. It learns to represent time-invariant features of a trajectory by an \emph{encoder} network, followed by a \emph{forecaster} network to learn shared dynamics across the different environments. 
DyAd only performs well under weak supervision, that is when the underlying (observed) parameters are made known to the loss function via penalization.


Arguably the most successful Meta-Learning method for physical systems is CoDA \cite{kirchmeyer2022generalizing}, which assumes that the underlying system is described by parametrized differential equations whose form is shared by all environments. However, these equations differ by the values of the vector field's weights, which are produced by a (linear) \emph{hypernetwork}. For the environment indexed by $e$, these weights are computed by\footnote{The general CoDA formulation encompasses the GMBL adaptation rule \cite{kirchmeyer2022generalizing}. Furthermore, MTL models can be identified to \cref{eq:coda} with $\theta^c = \mathbf{0}$.}
\begin{align} \label{eq:coda}
    \theta^e = \theta^c + W\xi^e,
\end{align}
where $\theta^c$ and $W$ are shared across environments, and $\xi^e \in \mathbb{R}^{d_{\xi}}$ is an environment-specific latent context vector (or simply \textbf{context}). While it achieves state-of-the-art performance on many physical systems, the main limitation of CoDA is its hypernetwork approach, which might hinder parallelism and \rebut{memory} scaling to large root or target networks. In practice, methods based on hypernetworks require more computational resources to \rebut{backpropagate and} train, and exhibit a more complex optimization landscape \cite{chauhan2023brief}. 




\paragraph{Taylor Expansion.} This local approximation strategy finds numerous applications in SciML, notably helping establish single-trajectory geometrical requirements for linear and affine system identification \cite{duan2020identification}. Recently, \citet{blanke2024interpretable} modeled the variability of linearly parameterized dynamical systems with an affine function of low-dimensional environment-specific context vectors. They empirically showed that this improved interpretability, generalization abilities, and computation speed. In our work, we similarly explain \emph{non-linear} systems, thus generalizing existing work with higher-order Taylor expansion. Additionally, we extract benefits such as massive parallelizability and uncertainty \rebut{estimation}.

\section{Neural Context Flow}
\label{method}

A training dataset $\trainset := \big\{ x^e_{i}(\cdot)\big\}_{ i \in \lbint 1, S \rbint }^{e \in \lbint 1,m \rbint}$ is defined as a set of trajectories collected
from $m$ related environments, with $S$ trajectories per environment, each of length $N \in \mathbb{N}^*$ over a time horizon $T>0$. Given $\trainset$, we aim to find the neural network weights $\theta$ that parameterize a vector field $f_{\theta}$, along with several context vectors $ \{ \xi^e \}_{e \in \lbint 1,m \rbint}$ that modulate its behavior such that 
\begin{align} \label{eq:neuralode}
    \frac{\md x_i^e}{\md t}(t) = f_{\theta}(x_i^e(t), \xi^e), \qquad \forall t \in \left[ 0,T \right], \quad \forall i \in \lbint 1, S \rbint, \quad \forall e \in \lbint 1,m \rbint.
\end{align}
We learn a single vector field for all environments in our training set $\trainset$. The same vector field will be reused, unchanged, for future testing and adaptation to environments in \rebut{a similarly-defined} $\adaptset$. 

The vector field $f_{\theta}$ is assumed to not only be \emph{continuous}, but also \emph{smooth} in its second argument $\xi$. Exploiting this \rebut{constraint}, we ``collect'' information from other environments by Taylor-expanding $f_{\theta}$ around any other $\{\xi^j\}_{j \in \mathrm{P}}$, where $\mathrm{P} \subseteq \lbint 1,m \rbint$ is a \textbf{context pool}\footnote{We note that $\mathrm{P}$ might include $e$ itself, and is reconstituted at each evaluation of \cref{eq:loss}. Its size $p$ is constant, each element indexing a \emph{distinct} environment for computational efficiency (see \cref{subsec:whatsinapool}).} containing $p:=| \mathrm{P}| \geq 1$ environment indices. This gives rise, for fixed $e$ and $i$, to $p$ Neural ODEs
\begin{align} \label{eq:method}
    \begin{dcases}
    \frac{\md x^{e,j}_i}{\md t}(t) = T^k_{f_{\theta}}(x^{e,j}_i(t), \xi^e, \xi^j), \\
    x^{e,j}_i(0) = x^{e}_{i}(0), 
    \end{dcases}
    \qquad \forall j \in \mathrm{P},
\end{align}
where $x^{e,j}_i(\cdot) \in \mathbb{R}^d$, and $\xi^e,\xi^j \in \mathbb{R}^{d_\xi}$. $T^k_{f_{\theta}}(\cdot, \xi^e, \xi^j)$ denotes the $k$-th order Taylor expansion of $f_{\theta}$ at $\xi^e$ around $\xi^j$.
In particular, $T^1_{f_{\theta}}$ can be written as
\begin{align} \label{eq:talor_1}
    T^1_{f_{\theta}}(x^{e,j}_i, \xi^e, \xi^j) &= f_{\theta}(x^{e,j}_i, \xi^j) 
    + \nabla_{\xi} f (x^{e,j}_i, \xi^j)  (\xi^e - \xi^j) 
    + o(\Vert \xi^e - \xi^j \Vert ),
\end{align}
where $o(\cdot)$ captures negligible residuals. \cref{eq:talor_1} directly consists of a Jacobian-Vector Product (JVP), making its implementation memory-efficient. Since higher-order Taylor expansions of vector-valued functions do not readily display the same property, we provide the following proposition to facilitate the second-order approximation.


\begin{proposition}[Second-order Taylor expansion with JVPs] \label{prop:t2}
    Assume $f: \mathbb{R}^d \times \mathbb{R}^{d_{\xi}} \rightarrow \mathbb{R}^{d}$ is $\mathcal{C}^2$ wrt its second argument. Let $x \in \mathbb{R}^d, \xi \in \mathbb{R}^{d_{\xi}} $, and define $g: \bar \xi \mapsto \nabla_{\xi} f(x, \bar \xi)(\xi - \bar \xi) $. The second-order Taylor expansion of $f$ around any $ \tilde \xi \in \mathbb{R}^{\xi}$ is then expressed as
    \begin{align}
        f(x,\xi) &= f(x,\tilde \xi) + \frac{3}{2} g(\tilde \xi) 
        + \frac{1}{2} \nabla g(\tilde \xi) (\xi - \tilde \xi) 
        + o(\Vert \xi - \tilde \xi \Vert^2).
    \end{align}
\end{proposition}

\begin{proof}
    We refer the reader to \cref{app:proofs}.
\end{proof}

By setting $\xi^e := \xi$, and $\xi^j := \tilde \xi$, \cref{prop:t2} yields an expression for $T^2_{f_{\theta}}(\cdot,\xi^e, \xi^j)$ in terms of JVPs, an implementation of which we detail in \cref{app:code}. During training (as described below) trajectories from all $p$ Neural ODEs are used within the loss function.


\begin{figure}[t]
\begin{center}
\centerline{\includegraphics[width=0.9\columnwidth]{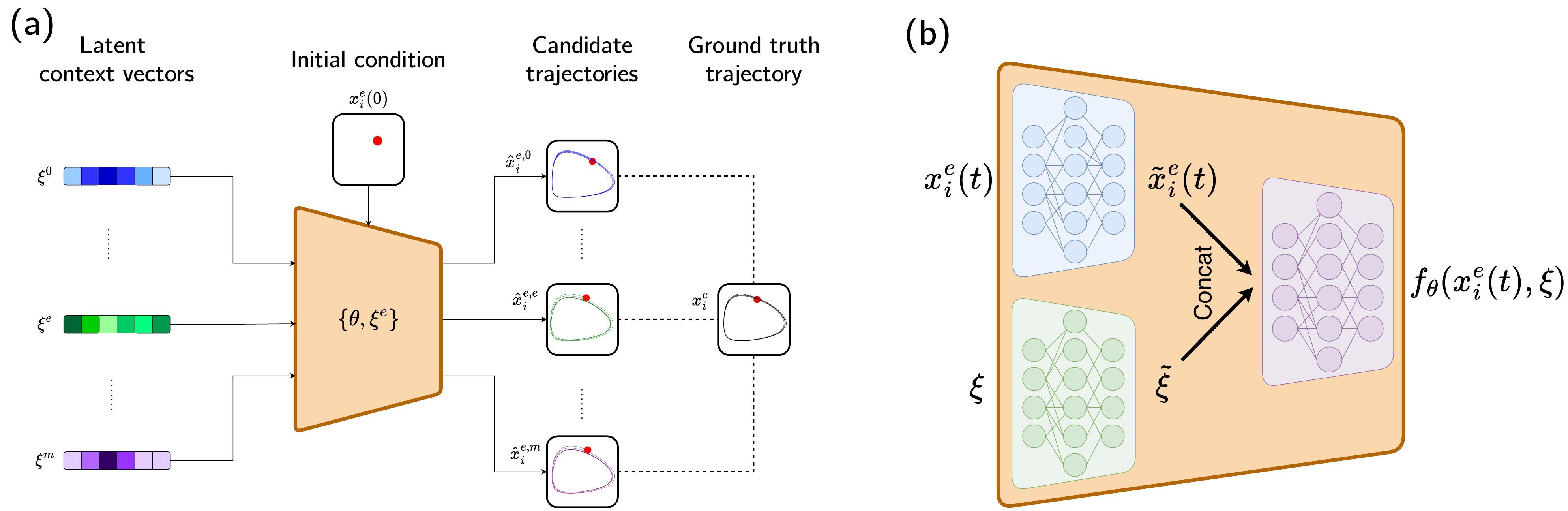}}
\caption{\textbf{(a)} Illustration of the Neural Context Flow (NCF). Given an initial condition $x_i^e(0)$ for a training trajectory $i$ in the environment $e$, NCF predicts in parallel, several candidate trajectories $ \{\hat{x}^{e,j}_{i} \}_{j \in \mathrm{P}}$ that are all compared to the ground truth $x_i^{e}$, upon which $\{\theta,\xi^e\}$ is updated. \textbf{(b)} Depiction of the 3-networks architecture for $f_{\theta}$, where the state vector and an arbitrary context vector are projected into the same representational space before they can interact inside the main network.}
\label{fig:method}
\end{center}

\vspace*{-0.6cm}
\end{figure}

This new framework is called \textbf{Neural Context Flow} (NCF) as is the resulting model. The ``flow'' term refers to the capability of the context from one environment, i.e., $j$, to influence predictions in another environment, i.e., $e$, by means of Taylor expansion. It allows the various contexts to not only modulate the behavior of the vector field, but to equally modulate theirs, since they are forced to \textcolor{teal}{remain close} for the Taylor approximations to be accurate. This happens while the same contexts are \textcolor{orange}{pushed apart} by the diversity in the data. This self-modulation process and the beneficial friction it creates is notably absent from other contextual Meta-Learning approaches like CAVIA \cite{zintgraf2019fast} and CoDA \cite{kirchmeyer2022generalizing}.

We depict the NCF framework in \cref{fig:method}, along with a \textbf{3-networks} architecture that \rebut{lifts} the contexts and the state vectors into the same representational space before they can interact. Concretely, $x_i^e(t)$ and the context $\xi$ are first processed independently into $\tilde x_i^e(t)$ (by a state network in blue) and $\tilde \xi$ (by a context network in green) respectively, before they are concatenated and fed to a main network (in purple) to produce $f_{\theta}(x_i^e(t), \xi)$. This explicitly allows the model to account for the potentially nonlinear relationship between the context and the state vector at each evaluation of the vector field.


\subsection{Meta-Training}

Starting from the same initial state $\hat x^{e,j}_i(0) := x^e_i(0)$, the Neural ODEs in \cref{eq:method} are integrated using a differentiable numerical solver \cite{kidger2022neural,nzoyem2023comparison,politorchdyn,torchdiffeq}:
\begin{align} \label{eq:candidates}
    \hat x^{e,j}_i(t) = \hat x^{e,j}_i(0) + \int_0^t T^k_{f_{\theta}}(\hat x^{e,j}_i(\tau), \xi^e, \xi^j) \, \md \tau, \quad 
    \forall j \in \mathrm{P}.
\end{align}
The resulting \textbf{candidate} trajectories are evaluated at specific time steps $\left\{ t_n \right\}_{n \in \lbint 1,N \rbint}$  such that $t_1 = 0 \text{ and } t_{N} = T$. We feed these to a supervised (inner) loss function
\begin{align} \label{eq:loss}
    \ell(\theta,\xi^e,\xi^j, \hat{x}^{e,j}_{i}, x^{e}_{i}) := \frac{1}{N \times d} \sum_{n=1}^{N}{\Vert \hat{x}^{e,j}_{i}(t_n) - x^{e}_{i}(t_n) \Vert_2^2} 
    +  \frac{\lambda_1}{d_{\xi}} \Vert \xi^e \Vert_1 
    +  \frac{\lambda_2}{d_{\theta}} \Vert \theta \Vert_2^2,
\end{align}
where 
$d, d_{\xi},$ and $d_{\theta}$ are the dimensions of the state space, context vectors, and flattened network weights, respectively. $\Vert \cdot \Vert_1$ and $\Vert \cdot \Vert_2$ denote the \rebut{regularizing} $L^1$ and $L^2$ norms, with $\lambda_1$ and $\lambda_2$ their penalty coefficients, respectively. 
To ensure its independence from environment count, context pool size, and trajectory count per environment, the overall MSE loss function $\mathcal{L}$ is expressed as in \cref{eq:bigloss}; after which it is minimized wrt both weights $\theta$ and contexts $\xi^{1:m} := \{ \xi^e \}_{e \in \lbint1,m \rbint}$ via gradient descent, alternating between updates:
\begin{align} \label{eq:bigloss}
    \mathcal{L}(\theta, \xi^{1:m}, \trainset) := \frac{1}{m \times S \times p} \sum_{e=1}^m \sum_{i=1}^S \sum_{j=1}^p  \ell(\theta,\xi^e,\xi^j, \hat{x}^{e,j}_{i}, x^{e}_{i}).
\end{align}



The most effective way to train NCFs is via \textbf{proximal} alternating minimization as described in \cref{alg:ncf_proximal}. Although more computationally demanding compared to ordinary alternating minimization (see \cref{alg:ncf}), it is adept at dealing with non-smooth loss terms like the $L^1$ norm in \cref{eq:loss} \cite{parikh2014proximal}. Not to mention that ordinary alternating minimization (\cref{alg:ncf}) can easily lead to sub-optimal convergence \cite{attouch2010proximal,li2019alternating}, while its proximal counterpart converges, with random initialization, almost surely to second-order stationary points provided mild assumptions are satisfied (see \cref{the:convergence}). We provide a short discussion on those assumptions in \cref{app:convproximal}.

The above comparison demands the definition of two variants of Neural Context Flows. \textbf{NCF-$\mathit{t}_2$} uses \emph{second}-order Taylor expansion in \eqref{eq:method}, and is trained via \emph{proximal} alternating minimization (\cref{alg:ncf_proximal}). \textbf{NCF-$\mathit{t}_1$} on the other hand, is implemented using a \emph{first}-order Taylor expansion, and trained using \emph{ordinary} alternating minimization (\cref{alg:ncf}). Although less expressive than \ttwo, \tone is faster and serves as a powerful baseline in our experiments.

\begin{theorem}[Convergence to second-order stationary points]
\label{the:convergence}
Assume that $\mathcal{L}(\cdot, \cdot, \trainset)$ satisfies the Kurdyka-Lojasiewicz (KL) property, is $L$ bi-smooth, and $\nabla \mathcal L(\cdot, \cdot, \trainset)$ is Lipschitz continuous on any bounded subset of domain $\mathbb{R}^{d_\theta} \times \mathbb{R}^{d_{\xi} \times m}$. Under those assumptions, let $(\theta_0, \xi^{1:m}_0)$ be a random initialization and $(\theta_q, \xi^{1:m}_q)$ be the sequence generated by \cref{alg:ncf_proximal}. If the sequence $(\theta_q, \xi^{1:m}_q)$ is bounded, then it converges to a second-order stationary point of $\mathcal{L}(\cdot, \cdot, \trainset)$ almost surely.
\end{theorem}

\begin{minipage}[t]{0.48\textwidth}

\vspace*{-0.6cm}
\begin{algorithm}[H]
   \caption{Proximal Alternating Minimization}
   \label{alg:ncf_proximal}
\begin{algorithmic}[1]
   \State {\bfseries Input:} 
   $ \trainset := \{ \mathcal{D}^e \}_{e\in \lbint 1,m \rbint}$
   \State $\,\, \theta_0 \in \mathbb{R}^{d_{\theta}}$ randomly initialized 
   \State $\,\, \xi^{1:m}_0 = \bigcup\limits_{e=1}^m \xi^e$, where $ \xi^e = \mathbf{0} \in \mathbb{R}^{d_{\xi}} $
   \State $\,\, q_{\text{max}} \in \mathbb{N}^*$; $\beta \geq L \in \mathbb{R}^+; \eta_{\theta}, \eta_{\xi} > 0$

   \For{$q \leftarrow 1,q_{\text{max}}$}

   \State $\mathcal{G}(\theta) := \mathcal{L}(\theta, \xi^{1:m}_{q-1}, \trainset) + \frac{\beta}{2} \Vert \theta - \theta_{q-1} \Vert_2^2$
   \State $\theta_q = \theta_{q-1}$
   
   \Repeat
   \State $\theta_q \leftarrow \theta_{q} - \eta_{\theta} \nabla \mathcal{G} (\theta_{q})$
   \Until{$\theta_q$ converges}

   \State $\mathcal{H}(\xi^{1:m}) := \mathcal{L}(\theta_q, \xi^{1:m}, \trainset) \newline \text{\qquad \qquad \qquad \qquad} + \frac{\beta}{2} \Vert \xi^{1:m} - \xi^{1:m}_{q-1} \Vert_2^2$
   \State $\xi^{1:m}_q = \xi^{1:m}_{q-1}$

   \Repeat
   \State $\xi^{1:m}_q \leftarrow \xi^{1:m}_q - \eta_{\xi} \nabla \mathcal{H}(\xi^{1:m}_q)$
   \Until{$\xi^{1:m}_q$ converges}

   \EndFor
\end{algorithmic}
\end{algorithm}

\vspace*{-0.3cm}

\end{minipage}
\hfill
\begin{minipage}[t]{0.48\textwidth}

\vspace*{-0.6cm}
\begin{algorithm}[H]
   \caption{Sequential Adaptation of NCF}
   \label{alg:ncf_adp_seq}
\begin{algorithmic}[1]
   \State {\bfseries Input:} 
   $ \adaptset := \{ \mathcal{D}^{e'} \}_{e'\in \lbint a,b \rbint}$
   \State $\,\, \theta \in \mathbb{R}^{d_{\theta}}$ learned 
   \State $ \,\, \xi^{e'} = \mathbf{0} \in \mathbb{R}^{d_{\xi}}, \forall e' \in \lbint a,b \rbint$
    \State $\,\, \eta > 0$
   \For{$e' \leftarrow a,b$}
   \State {$\mathcal{L}(\xi^{e'}, \mathcal{D}^{e'}) := \newline \text{\qquad \qquad } \frac{1}{S'} \sum\limits_{i=1}^{S'} \ell(\theta,\xi^{e'},\xi^{e'}, \hat{x}^{e',e'}_{i}, x^{e'}_{i}) $}
   \Repeat
   \State \quad $\xi^{e'} \leftarrow \xi^{e'} - \eta \nabla \mathcal{L}(\xi^{e'}, \mathcal{D}^{e'})$
   \Until{$\xi^{e'}$ converges}
   \EndFor
\end{algorithmic}
\end{algorithm}

\begin{proof}
The proof of \cref{the:convergence} is straightforward by adapting Assumptions 1 and 4, then Theorem 2 of \cite{li2019alternating}.   
\end{proof}

\end{minipage}

\subsection{Adaptation (or Meta-Testing)}
Few-shot adaptation to a new environment $e' \in \lbint a,b \rbint$ in or out of the meta-training distribution requires relatively less data, i.e., its trajectory count $S' \ll S $. Here, the network weights are frozen, and the goal is to find a context $\xi^{e'}$ such that
\begin{align} \label{eq:neuralode_adapt}
    \frac{\md x_i^{e'}}{\md t}(t) = f_{\theta}(x_i^{e'}(t), \xi^{e'}), \qquad \forall i \in \lbint 1, S' \rbint.
\end{align}

Our adaptation rule, as outlined in \cref{alg:ncf_adp_seq} is extremely fast, converging in seconds for trainings that took hours. In scenarios where we want to adapt to more than one environment, we outline a bulk version in \cref{app:bulk} (see \cref{alg:ncf_adp_bulk}). Although the bulk adaptation algorithm is parallelizable and framed in the same way as during meta-training, it does not allow for flow of contextual information, since this causes significant accuracy degradation. Most importantly, disabling the Taylor expansion at this stage limits the size of the context pool to $p=1$ and significantly improves memory efficiency, an important resource when adapting to a large number of environments.



\subsection{InD and OoD Testing}

We distinguish two forms of testing. For \textbf{In-Domain} (InD) testing, the environments are the same as the ones used during training. In InD testing data, the underlying parameters of the dynamical system we aim to reconstruct are unchanged, and so the meta-learned context vectors are reused. \textbf{Out-of-Distribution} (OoD) testing considers environments encountered during adaptation. The data is from the same dynamical system, but defined by different parameter values (either interpolated or extrapolated). In all forms of testing, only the main predicted trajectory corresponding to $j=e$ or $j=e'$ is used in the MSE and MAPE metrics computation (see \cref{subsec:lv} for definitions), thus returning to a standard Neural ODE. Other candidate trajectories are aggregated to ascertain the model's uncertainty. The initial conditions that define the trajectories are always unseen, although their distribution never changes from meta-training to meta-testing.


\section{Main Results}
\label{experiments}

In this section, we evaluate the effectiveness of our proposed framework by investigating two main questions. \textbf{(\textit{i})} How good are NCFs at resolving (P1) and (P2) for interpolation tasks (\cref{subsec:interpolation})? \textbf{(\textit{ii})} How does our framework compare to SoTA Meta-Learning baselines (\cref{subsec:extrapolation})? \rebut{Further questions regarding interpretability, scalability, and uncertainty estimation are formulated and addressed in \cref{subsec:interpret,subsec:uncertainty,app:scalability} respectively.}

\subsection{Experimental Setting}

The NCF framework is evaluated on seven seminal benchmarks. The Simple Pendulum (\textbf{SP}) models the periodic motion of a point mass suspended from a fixed point. The Lotka-Volterra (\textbf{LV}) system models the dynamics of an ecosystem in which two species interact. Additional ODEs include a simple model for yeast glycolysis: Glycolytic-Oscillator (\textbf{GO}) \cite{daniels2015efficient}, and the more advanced Sel'kov Model (\textbf{SM}) \cite{strogatz2018nonlinear,sel1968self}. Like GO, SM is non-linearly parameterized, but additionally exhibits starkly different behaviors when its key parameter is varied (see \cref{fig:attractors}). Finally, we consider three PDEs, all with periodic boundary conditions, and cast as ODEs via the method of lines: the non-linear oscillatory Brusselator (\bt) model for autocatalytic chemical reactions \cite{prigogine1968symmetry}, the Gray-Scott (\textbf{GS}) system also for reaction-diffusion in chemical settings \cite{pearson1993complex}, and Navier-Stokes (\ns) for incompressible fluid flow \cite{stokes1851effect}.


For all problems, the parameters and initial states are sampled from distributions representative of real-world problems observed in the scientific community, and the trajectories are generated using a time-adaptive 4th-order Runge-Kutta solver \cite{2020SciPy-NMeth}. For LV, GO, GS, and NS, we reproduce the original guidelines set in \cite{kirchmeyer2022generalizing}, while exposing the data for ODE and PDE problems alike via a common interface. Such use of synthetic data is a common practice in this emerging field of generalizable dynamical systems, where the search for unifying benchmarks remains an open problem \cite{massaroli2020dissecting}. This need for shared datasets and APIs has motivated the \texttt{\href{https://github.com/ddrous/gen-dynamics}{Gen-Dynamics}} open-source initiative, our third and final contribution with this paper. Further details, along with the data generation process, are given in \cref{app:datasets}.

We now highlight a few key practical considerations shared across experiments. We use the 3-networks architecture depicted in \cref{fig:method}b to suitably process the state and context variables. The dimension of the context vector, the context pool's size and filling strategy, and the numerical integration scheme vary across problems. For instance, we set $d_{\xi}=1024$ for LV, $d_{\xi}=202$ for NS, and $d_{\xi}=256$ for all other problems; while $p=2$ for LV and SM, $p=4$ for LV and GO, and $p=3$ for all PDE problems. Other hyperparameters are carefully discussed in \cref{app:detailes}.

\subsection{Interpolation Results}
\label{subsec:interpolation}

\begin{wraptable}[7]{R}{0.5\textwidth}
\vspace*{-0.5cm}
\caption{Training and adaptation testing MSEs ($\downarrow$) with OFA, OPE, and \tone on the SP problem.}
\label{tab:pendulum}
\begin{center}
\begin{small}
\begin{sc}
\vspace*{-0.4cm}
\begin{tabular}{lcc}
\toprule
 & Train ($\times 10^{-1}$) & Adapt ($\times 10^{-3}$) \\
\midrule
OFA    & 9.49 $\pm$ 0.04  & 115000  $\pm$ 3200\\
OPE    &  0.18 $\pm$ 0.02 & 459.0  $\pm$ 345.0\\
NCF  & 0.10 $\pm$ 0.03 & 0.0356 $\pm$ 0.001\\
\bottomrule
\end{tabular}
\end{sc}
\end{small}
\end{center}
\end{wraptable}

This experiment explores the SP problem discussed in \cref{fig:trajscompare}. During meta-training, we use 25 environments with the gravity $g$ regularly spaced in $[2,24]$. Each of these environments contains only 4 trajectories with the initial conditions $x^e_i(0) \sim \left( \mathcal{U}(-\frac{\pi}{3}, \frac{\pi}{3}), \mathcal{U}(-1,1) \right)^T$. During adaptation, we interpolate to 2 new environments with $g \in \{10.25,14.75\}$, each with 1 trajectory. \rebut{For both training and adaptation testing scenarios, we generate 32 separate trajectories.}

\cref{tab:pendulum} emphasizes the adaptation-time merits of Meta-Learning approaches like NCFs in place of baselines where one context-agnostic vector field is trained for all environments indiscriminately (One-For-All or OFA), \rebut{or for each environment independently (One-Per-Env or OPE)}. Additional results and further analysis of these differences is done in \cref{subsec:sp}.



\subsection{Extrapolation Results}
\label{subsec:extrapolation}

A Meta-Learning algorithm is only as good as its ability to extrapolate to unseen environments. In this section we tackle several one-shot generalization problems, i.e. $S'=1$. We consider the LV, GO, SM, GS, BT, and NS problems, which all involve adaptation environments outside their meta-training distributions. We compare both variants of NCF to two baselines: CAVIA \cite{zintgraf2019fast} which is conceptually the closest GBML method to ours, and CoDA-$\ell_1$ \cite{kirchmeyer2022generalizing}. Their hyperparameters, laid out in \cref{app:detailes}, are tuned for a balance of computational efficacy and performance, all the while respecting each baseline's key assumptions.

At similar parameter counts\footnote{This accounts for the total number of learnable parameters in the framework excluding the context vectors, whose size might vary with the method as explained in \cref{app:detailes}.}, \cref{tab:sota_results} shows that CAVIA is the least effective for learning all six physical systems, with a tendency to overfit on the few shots it receives during both meta-training and meta-testing \cite{mishra2017simple}. \ttwo achieves SoTA OoD results on 5 out of 6 problems. While CoDA retains its superiority on GS, we find that it struggles on non-linear problems, most notably on the SM problem whose trajectories go through 3 distinct attractors in a Hopf bifurcation.  
\begin{wrapfigure}[11]{L}{0.25\textwidth}
\vspace*{-0.5cm}
\centering
\includegraphics[width=\linewidth]{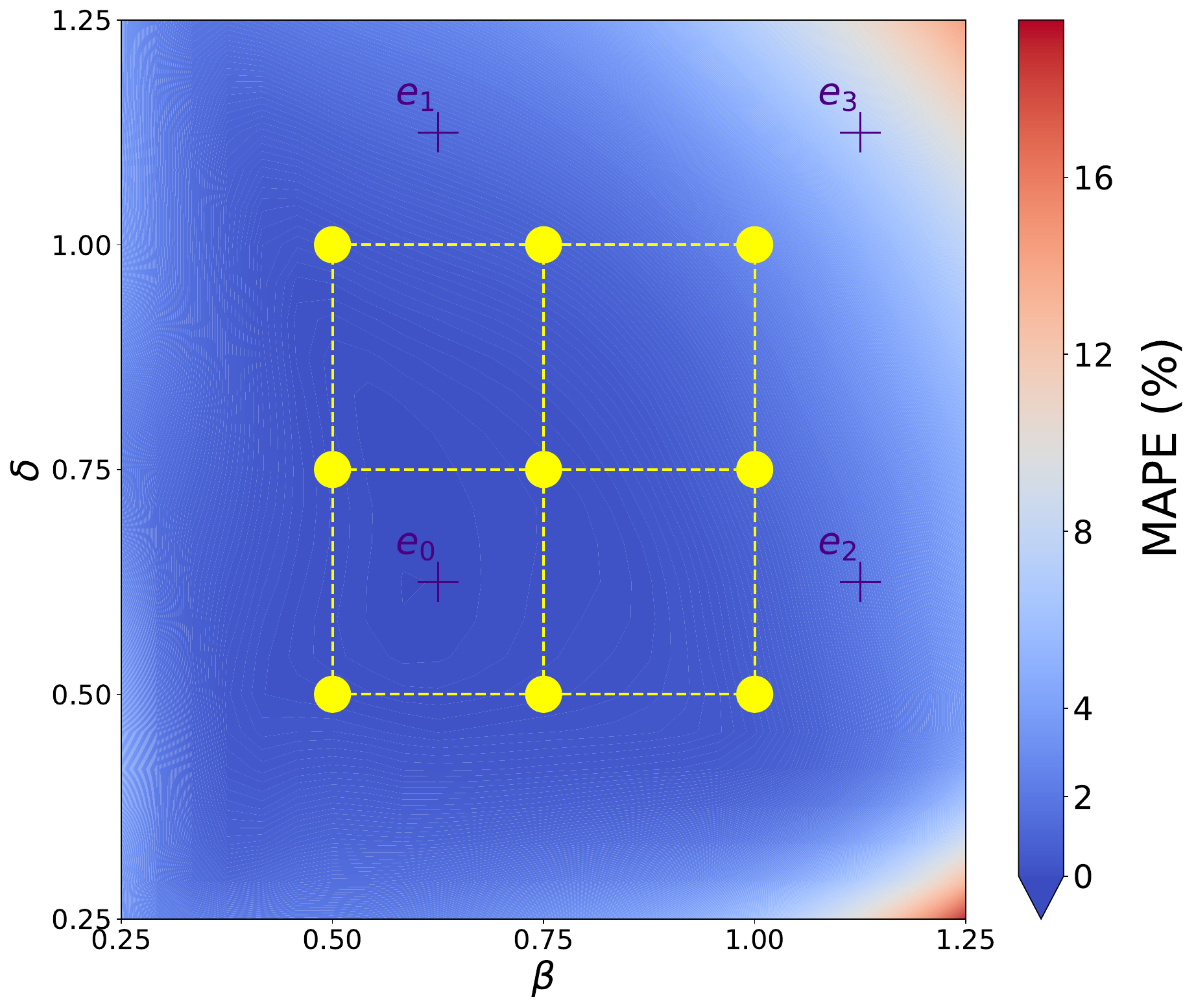}
\caption{Grid-wise adaptation on the LV problem showing low MAPEs ($\downarrow$).}
\label{fig:huge_adapt}
\end{wrapfigure}

This shows that CoDA's ability to automatically select low-rank adaptation subspaces via a linear hypernetwork decoder is limited, particularly for highly nonlinear systems. \rebut{We note that \tone equally fails to accurately resolve SM and other non-linear problems, thus illustrating the value of second-order Taylor expansion. This suggests that higher-order Taylor-based regularization prevents \ttwo from learning spurious associations, which is a problem commonly associated with poor OoD generalization \cite{mouli2023metaphysica}.}

\begin{table*}[t!]
\caption{In-Domain (InD) and adaptation (OoD) test MSEs ($\downarrow$) for the LV, GO, SM, BT, GS and NS problems. The best is reported in \textbf{bold}. The best of the two NCF variants is shaded in \colorbox{grey}{grey}.}
\vspace*{-0.3cm}
\label{tab:sota_results}
\begin{center}
\footnotesize
\begin{sc}
\begin{tabularx}{0.95\textwidth}{lcccccc}
\toprule
& \multicolumn{3}{c}{\textbf{LV} ($\times 10^{-5}$)} & \multicolumn{3}{c}{\textbf{GO} ($\times 10^{-4}$)} \\
\cmidrule(lr){2-4} \cmidrule(lr){5-7}
 & \#Params  & InD  & OoD &  \#Params  & InD & OoD   \\
\midrule
CAVIA      & 305246 & 91.0$\pm$63.6 & 120.1$\pm$28.3 & 130711 & 64.0$\pm$14.1 & 463.4$\pm$84.9 \\
CoDA      & 305793 & \textbf{1.40$\pm$0.13} & 2.19$\pm$0.78 & 135390 & 5.06$\pm$0.81 & 4.22$\pm$4.21 \\
\tone   & 308240 & 6.73$\pm$0.87 & 7.92$\pm$1.04 & 131149 & 40.3$\pm$9.1 & 19.4$\pm$1.24 \\
\ttwo   & 308240 & \cellcolor{grey} 1.68$\pm$0.32 &  \cellcolor{grey} \textbf{1.99$\pm$0.31} & 131149 & \cellcolor{grey} \textbf{3.33$\pm$0.14}  & \cellcolor{grey} \textbf{2.83$\pm$0.23} \\
\bottomrule
\end{tabularx}

\begin{tabularx}{0.95\textwidth}{lcccccc}
& \multicolumn{3}{c}{\textbf{SM} ($\times 10^{-3}$)} & \multicolumn{3}{c}{\textbf{BT} ($\times 10^{-1}$)} \\
\cmidrule(lr){2-4} \cmidrule(lr){5-7}
 & \#Params  & InD & OoD  &  \#Params  & InD  & OoD  \\
\midrule
CAVIA      & 50486 & 979.1$\pm$141.2 & 859.1$\pm$70.7 & 116665 & 21.93$\pm$1.8 & 22.6$\pm$7.22 \\
CoDA      & 50547 & 156.0$\pm$40.52 &  8.28$\pm$0.29 & 119679 & 25.40$\pm$9.5 &  19.47$\pm$11.6 \\
\tone   & 50000 & 680.6$\pm$320.1 & 677.2$\pm$18.7 & 117502 & 21.53$\pm$8.9 & 20.89$\pm$12.0 \\
\ttwo   & 50000 & \cellcolor{grey} \textbf{6.42$\pm$0.41}  & \cellcolor{grey} \textbf{2.03$\pm$0.12} & 117502 & \cellcolor{grey} \textbf{3.46$\pm$0.09}  & \cellcolor{grey} \textbf{3.77$\pm$0.15} \\
\bottomrule
\end{tabularx}

\begin{tabularx}{0.95\textwidth}{lcccccc}
& \multicolumn{3}{c}{\textbf{GS} ($\times 10^{-3}$)} & \multicolumn{3}{c}{\textbf{NS} ($\times 10^{-3}$)} \\
\cmidrule(lr){2-4} \cmidrule(lr){5-7}
 & \#Params  & InD & OoD  &  \#Params  & InD  & OoD  \\
\midrule
CAVIA      & 618245 & 69.9$\pm$21.2 & 68.0$\pm$4.2 & 310959 & 128.1$\pm$29.9 & 126.4$\pm$20.7 \\
CoDA      & 619169 & \textbf{1.23$\pm$0.14} & \textbf{0.75$\pm$0.65} & 309241 & 7.69$\pm$1.14 &  7.08$\pm$0.07 \\
\tone   & 610942 & 7.64$\pm$0.70 & 5.57$\pm$0.21 & 310955 & 2.98$\pm$0.09 & 2.83$\pm$0.06 \\
\ttwo   & 610942 & \cellcolor{grey} 6.15$\pm$0.24 &  \cellcolor{grey} 3.40$\pm$0.51 & 310955 & \cellcolor{grey} \textbf{2.92$\pm$0.08}  & \cellcolor{grey} \textbf{2.79$\pm$0.09} \\
\bottomrule
\end{tabularx}

\end{sc}
\end{center}
\vspace*{-0.3cm}
\end{table*}

\paragraph{Grid-wide adaptation on LV.}

We report in \cref{fig:huge_adapt} how well NCF performs on 625 adaptation environments obtained by varying its parameters $\beta$ and $\delta$ on a 25$\times$25 uniform grid. It shows a consistently low MAPE below 2\%, except in the bottom right corner where the MAPE rises to roughly 15\%, still a remarkably low value for this problem. We highlight the 9 meta-training environments in yellow, and the 4 environments used for OoD adaptation for \cref{tab:sota_results} in indigo. Remarkably, our training environment's convex-hull where adaptation is particularly low is much larger compared to CoDA's \cite[Figure 3]{kirchmeyer2022generalizing}.

\paragraph{Nonlinear adaptation on SM.}

We investigate how the models perform across the SM attractors depicted in \cref{fig:attractors}: training environments $e_1$ and $e_2$ fall into a limit cycle (L1), $e_3$ and $e_4$ collapse to a stable equilibrium (E), while $e_5$ and $e_6$ fall into another limit cycle (L2).  \cref{fig:erropergroup} presents the adaptation MSE for each environment in the OoD testing dataset. While CoDA equally fails to capture all 3 attractors, we observe that CAVIA and \ttwo tend to favor E. But unlike CAVIA, \ttwo succeeds in capturing limit cycles as well, as evidenced by the low MSE on all environments.

\begin{figure}[h]
\vspace*{-0.5cm}
\centering
\subfigure[Trajectories]{\includegraphics[width=0.18\textwidth]{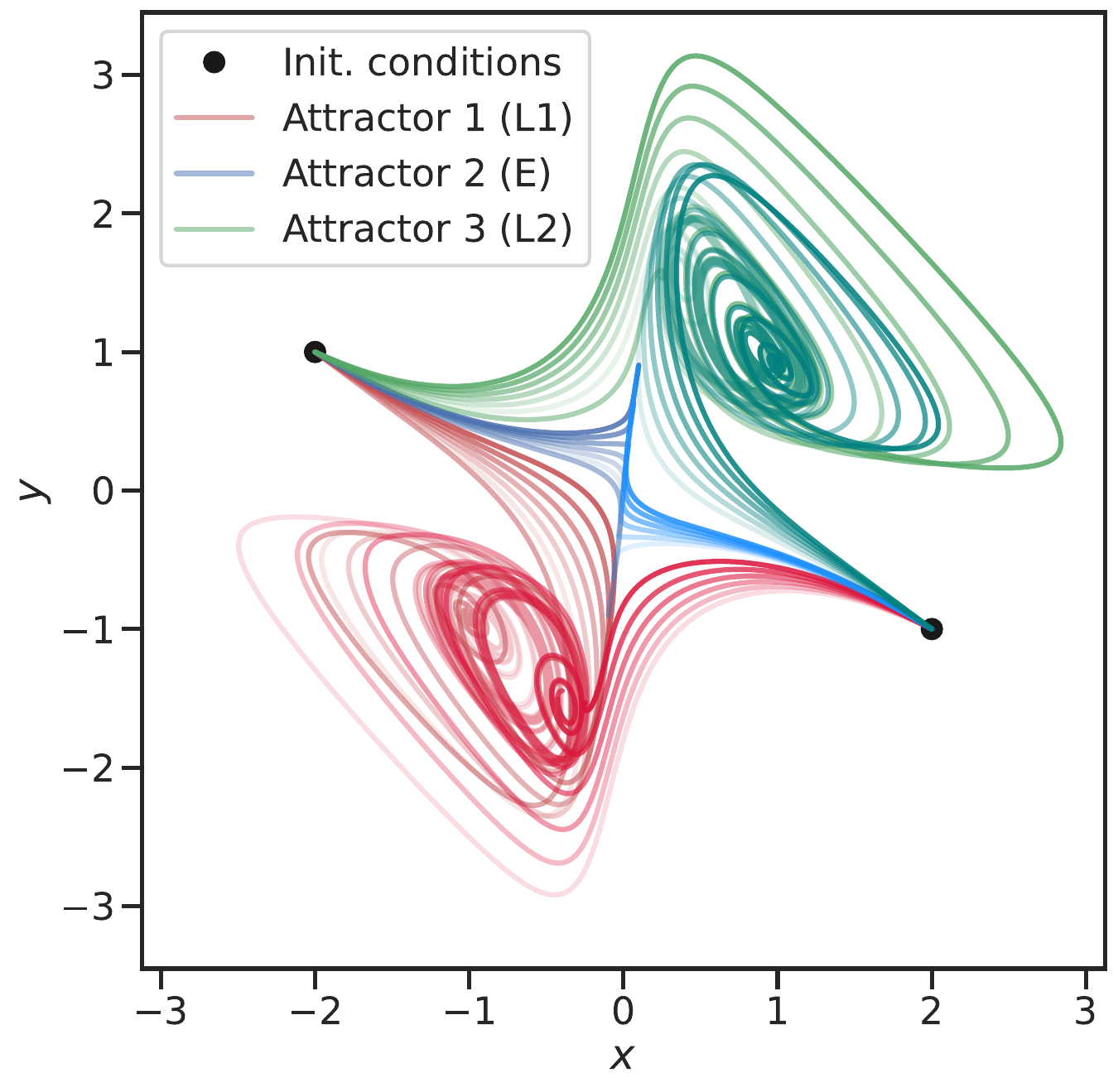}\label{fig:attractors}}
\subfigure[Training losses]{\includegraphics[width=0.35\textwidth]{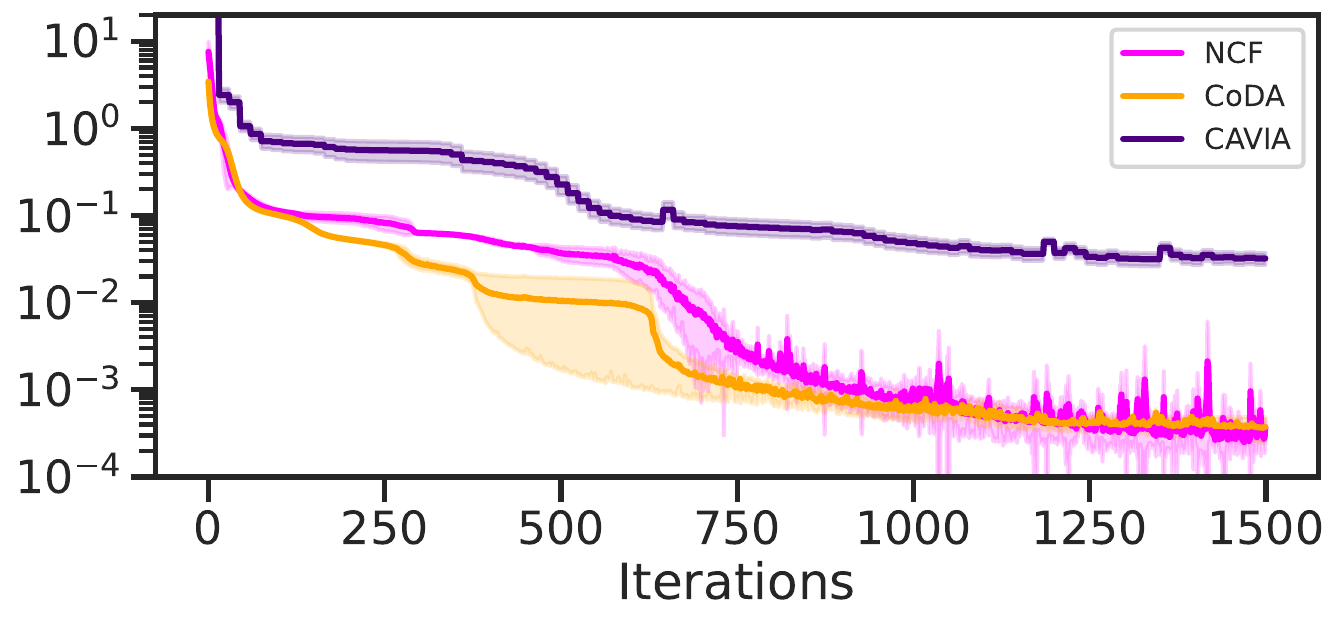}\label{fig:selkov_train_losses}}
\subfigure[OoD MSEs per environment]{\includegraphics[width=0.35\textwidth]{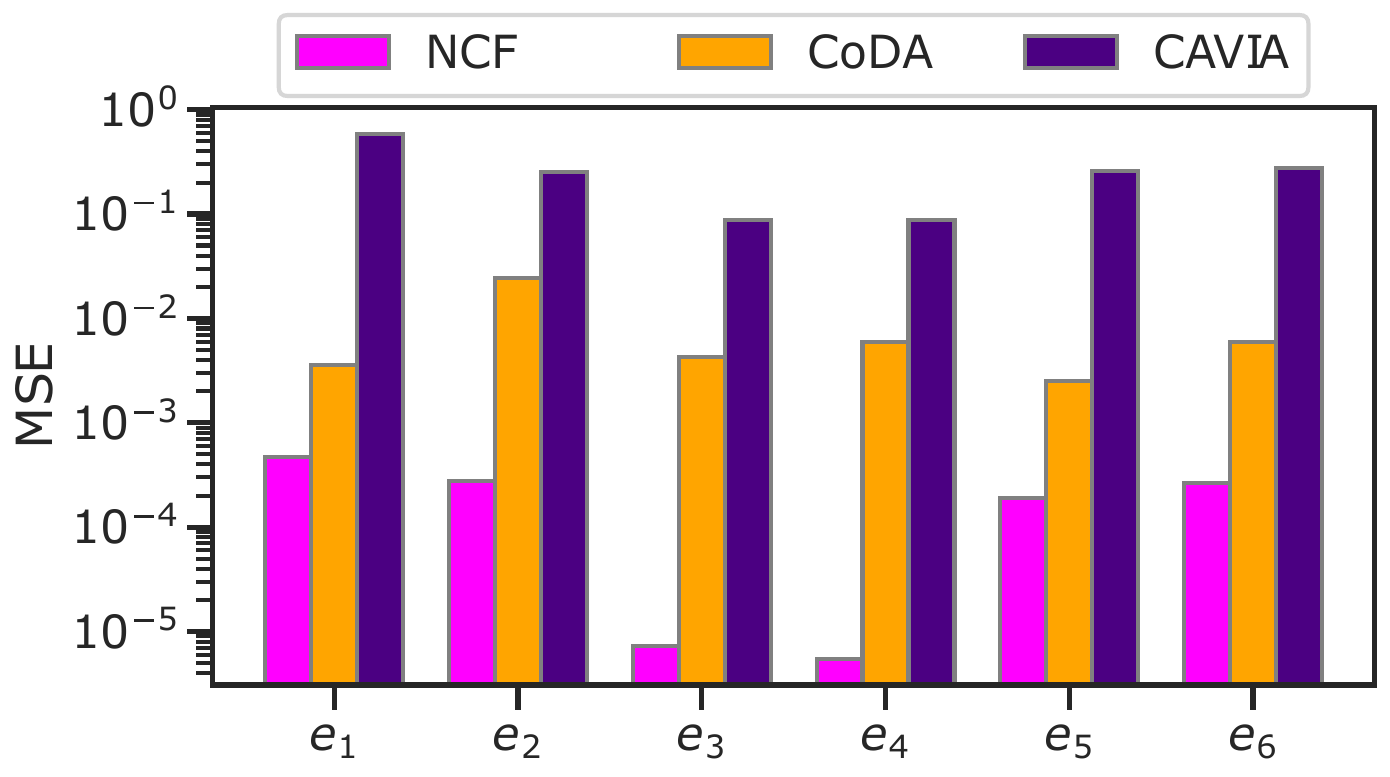}\label{fig:erropergroup}}
\caption{(a) Sample trajectories from the SM problem illustrating the Hopf bifurcation; (b) Losses during meta-training; (c) Subsequent adaptation MSE per method per environment.}
\label{fig:selkov}
\vspace*{-0.2cm}
\end{figure}

\section{Discussion}
\label{discussion}

\subsection{Benefits of NCFs}
\label{subsec:benefits}

Neural Context Flows provide a powerful and flexible framework for learning differential equations. They can handle irregularly-sampled sequences like time series, and can easily be extended to general regression tasks \cite{finn2017model}. Some of their other desirable properties are highlighted below.


\paragraph{Massively parallelizable.} \cref{eq:bigloss} indicates that NCFs are massively parallelizable along 3 directions. Indeed, evaluations of $\mathcal{L}$ can be vectorized across $m \times p$ environments, and across all $S$ trajectories. This leads to better use of computational resources for meta-training. We provide details on such vectorized NCF implementation in \cref{app:code}, along with a thorough discussion on its scalability in \cref{app:scalability}.

\begin{wrapfigure}[10]{R}{0.45\textwidth}
\vspace*{-0.65cm}
\begin{center}
\centerline{\includegraphics[width=\linewidth,  trim=0 0.15cm 0 0, clip]{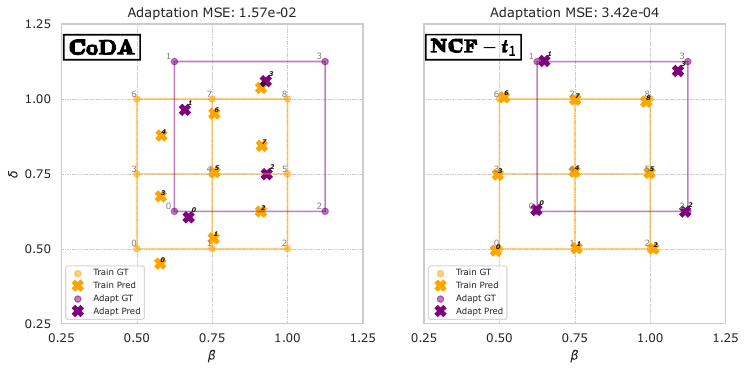}}
\caption{Interpretability with NCF. }
\label{fig:interpret_ncf_coda}
\end{center}
\end{wrapfigure}

\paragraph{Interpretable.} Understanding how a model adapts to new physical settings is invaluable in many scientific scenarios. Moving information from one environment to another via context-affine transformations provides a powerful framework for explaining Multi-Task Learning \cite{blanke2024interpretable}. With contextual self-modulation via Taylor expansion, we generalize this framework while maintaining computational efficiency. In \cref{fig:interpret_ncf_coda} for instance, we showcase system identification with \tone, where the underlying physical parameters $\beta$ and $\delta$ of the LV problem are recovered up to a linear transform. The corresponding experiment is detailed in \cref{subsec:interpret} \rebut{which further illustrates NCF's robustness to noise during system identification.}

\begin{proposition}[Identifiability of affine systems]
\label{prop:linearsystems}
Assume $d_{\xi} \geq d_c$, that $P$ is full-rank, and that $f_{\textup{true}}$ is differentiable in its second argument. In the limit of zero training loss in \cref{eq:bigloss}, $f_{\theta}$ trained with first-order Taylor expansion and the Random-All pool-filling strategy\footnote{A discussion on which neighboring contexts should go into the context pool $\mathrm{P}$ is provided in \cref{subsec:whatsinapool}.} is affine on an open region of $\mathbb{R}^{d_{\xi}}$. Furthermore, there exists $Q \in \mathbb{R}^{d_c \times d_{\xi}}$ and $q \in \mathbb{R}^{d_c}$ such that for any meta-trained $ \xi \in \{\xi^e\}_{e=1}^{m}$ and its corresponding underlying parameter $c \in \{c^e\}_{e=1}^{m}$, we have $c = Q\xi + q$.
\end{proposition}

To theoretically demonstrate the capacity to identify physical systems, we consider a learnable vector field $f_{\text{true}} : \mathbb{R}^d \times \mathbb{R}^{d_{c}} \rightarrow \mathbb{R}^d $ is that is affine i.e. $\exists P \in \mathbb{R}^{d\times d_{c}} \text { and } p \in \mathbb{R}^{d} $ such that $f_{\text{true}}(\cdot, c) = Pc + p$. Under these assumptions, the predictor $f_{\theta} : \mathbb{R}^d \times \mathbb{R}^{d_{\xi}} \rightarrow \mathbb{R}^d$ satisfies \cref{prop:linearsystems} above, inspired by Proposition 1 of \cite{blanke2024interpretable}. A detailed proof is provided in \cref{subsec:interpret}.

\begin{wrapfigure}[11]{L}{0.23\textwidth}
\vspace*{-0.5cm}
\begin{center}
\centerline{\includegraphics[trim=0 300 0 0, clip, width=\linewidth]{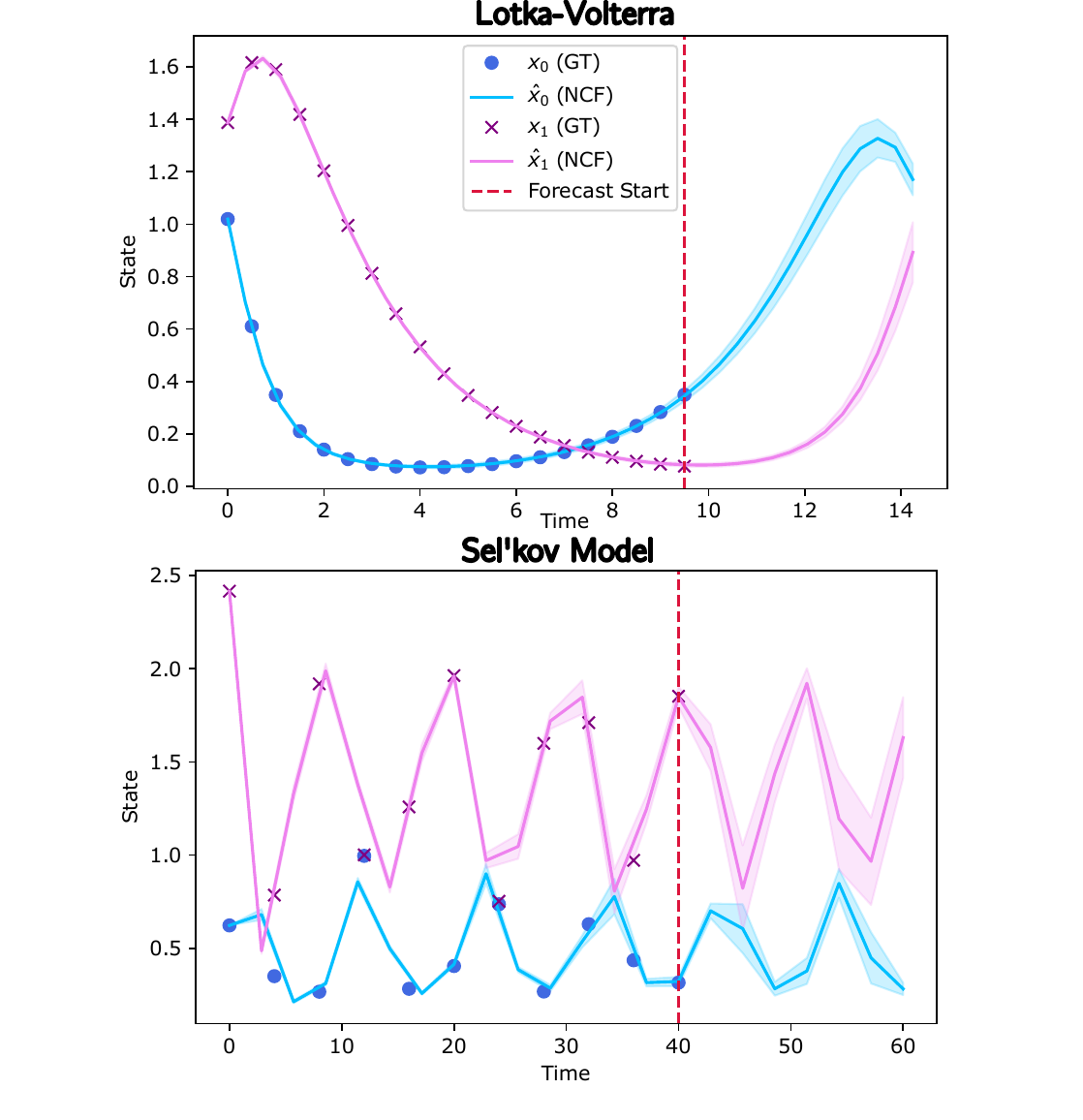}}
\caption{Uncertainty estimation with NCF.}
\label{fig:uq}
\end{center}
\end{wrapfigure}

\paragraph{Extendable.} The relative simplicity of our formulation makes NCFs easily adaptable to other works involving Neural ODE models. For instance, it is straightforward to augment the vector field as in APHYNITY \cite{yin2021augmenting}, augment the state space as in ANODE \cite{dupont2019augmented}, or control the vector field as in \cite{massaroli2020dissecting}.

\paragraph{Provides uncertainty.} Uncertainty requirements are increasingly critical in Meta-Learned dynamical systems \cite{liustochastic}. NCFs provide, by their very definition, a measure of the uncertainty in the model. During testing, all candidate trajectories stemming from available contexts (meta-trained, adapted, or both) can be used to ascertain when or where the model is most uncertain in its predictions. In \cref{fig:uq} for instance, we visualize the means and standard deviations (scaled 10 folds for visual exposition) across 9 and 21  candidate forecasts for the Lotka-Volterra and Sel'kov Model problems, respectively. \rebut{Additional results with quantitative metrics are provided in \cref{subsec:uncertainty}}.

\subsection{Limitations}

\paragraph{Further theory.} While massively parallelizable, extendable, interpretable, and applicable to countless areas well outside the learning of dynamics models, Neural Context Flows still need additional theoretical analysis to explain their effectiveness and uncover failure modes. While \rebut{the provided  \cref{prop:t2,prop:linearsystems}} partly compensates for this shortage, we believe that this creates an interesting avenue, along with other limitations outlined below, for future work.

\paragraph{Regularity assumptions.} Another limitation faced by Neural Context Flows lies within the assumptions they make. While differentiability of the vector field $f$ wrt $\xi$ is encountered with the majority of dynamical systems in science and engineering, some are bound to be fundamentally discontinuous. NCFs break down in such scenarios and would benefit from the vast body of research in numerical continuation \cite{allgower2012numerical}. 

\paragraph{Additional hyperparameters.} Finally, NCF introduces several new hyperparameters such as the context size $d_{\xi}$, the context pool size $p$, \rebut{the pool-filling strategy}, the proximity coefficient $\beta$ in \cref{alg:ncf_proximal}, and many more. Although we offer insights into their roles in the \cref{app:ablation}, we acknowledge that tuning them complicates the training process.

\subsection{Conclusion and Future Work}

This paper introduces Neural Context Flows (NCFs), an innovative framework that enhances model generalization across diverse environments by exploiting the differentiability of the predictor in its latent context vector. The novel application of Taylor expansion in NCFs facilitates vector field modulation for improved adaptation, enhances interpretability, and provides valuable uncertainty estimates for deeper model understanding. Our comprehensive experiments demonstrate the robustness and scalability of the NCF approach, particularly with respect to its most demanding hyperparameters. Future research will explore the limits of NCFs and their adaptation to even more complex scenarios. This work represents a promising step toward developing foundational models that generalize across scientific domains, offering a fresh and versatile approach to conditioning machine learning models.


\subsubsection*{Ethics Statement}

While the benefits of NCFs are evidenced in \cref{subsec:benefits}, their negative impacts should not be neglected. For instance, malicious deployment of such adaptable models in scenarios they were not designed for could lead to serious adverse outcomes. With that in mind, our code, data, and models are openly available at \url{https://github.com/ddrous/ncflow}.

\subsubsection*{Acknowledgments}
This work was supported by UK Research and Innovation grant EP/S022937/1: Interactive Artificial Intelligence, EPSRC program grant EP/R006768/1: Digital twins for improved dynamic design, EPSRC grant EP/X039137/1: The GW4 Isambard Tier-2 service for advanced computer architectures, and  EPSRC grant EP/T022205/1: JADE: Joint Academic Data Science Endeavour-2.
We extend our sincere gratitude to the anonymous reviewers whose insightful feedback and rigorous commentary substantially enhanced the quality and clarity of this manuscript. We also acknowledge Amarpal Sahota for valuable discussions on uncertainty quantification methodologies.

\bibliography{references}

\begin{thebibliography}{71}
\providecommand{\natexlab}[1]{#1}
\providecommand{\url}[1]{\texttt{#1}}
\expandafter\ifx\csname urlstyle\endcsname\relax
  \providecommand{\doi}[1]{doi: #1}\else
  \providecommand{\doi}{doi: \begingroup \urlstyle{rm}\Url}\fi

\bibitem[Allgower \& Georg(2012)Allgower and Georg]{allgower2012numerical}
Eugene~L Allgower and Kurt Georg.
\newblock \emph{Numerical continuation methods: an introduction}, volume~13.
\newblock Springer Science \& Business Media, 2012.

\bibitem[Attouch et~al.(2010)Attouch, Bolte, Redont, and Soubeyran]{attouch2010proximal}
H{\'e}dy Attouch, J{\'e}r{\^o}me Bolte, Patrick Redont, and Antoine Soubeyran.
\newblock Proximal alternating minimization and projection methods for nonconvex problems: An approach based on the kurdyka-{\l}ojasiewicz inequality.
\newblock \emph{Mathematics of operations research}, 35\penalty0 (2):\penalty0 438--457, 2010.

\bibitem[Bengio et~al.(2015)Bengio, Vinyals, Jaitly, and Shazeer]{bengio2015scheduled}
Samy Bengio, Oriol Vinyals, Navdeep Jaitly, and Noam Shazeer.
\newblock Scheduled sampling for sequence prediction with recurrent neural networks.
\newblock \emph{Advances in neural information processing systems}, 28, 2015.

\bibitem[Bettencourt et~al.(2019)Bettencourt, Johnson, and Duvenaud]{bettencourt2019taylor}
Jesse Bettencourt, Matthew~J Johnson, and David Duvenaud.
\newblock Taylor-mode automatic differentiation for higher-order derivatives in jax.
\newblock In \emph{Program Transformations for ML Workshop at NeurIPS 2019}, 2019.

\bibitem[Blanke \& Lelarge(2024)Blanke and Lelarge]{blanke2024interpretable}
Matthieu Blanke and Marc Lelarge.
\newblock Interpretable meta-learning of physical systems.
\newblock In \emph{ICLR 2024-The Twelfth International Conference on Learning Representations}, 2024.

\bibitem[Bradbury et~al.(2018)Bradbury, Frostig, Hawkins, Johnson, Leary, Maclaurin, Necula, Paszke, Vander{P}las, Wanderman-{M}ilne, and Zhang]{jax2018github}
James Bradbury, Roy Frostig, Peter Hawkins, Matthew~James Johnson, Chris Leary, Dougal Maclaurin, George Necula, Adam Paszke, Jake Vander{P}las, Skye Wanderman-{M}ilne, and Qiao Zhang.
\newblock {JAX}: composable transformations of {P}ython+{N}um{P}y programs, 2018.
\newblock URL \url{http://github.com/google/jax}.

\bibitem[Caruana(1997)]{caruana1997multitask}
Rich Caruana.
\newblock Multitask learning.
\newblock \emph{Machine learning}, 28:\penalty0 41--75, 1997.

\bibitem[Chauhan et~al.(2023)Chauhan, Zhou, Lu, Molaei, and Clifton]{chauhan2023brief}
Vinod~Kumar Chauhan, Jiandong Zhou, Ping Lu, Soheila Molaei, and David~A Clifton.
\newblock A brief review of hypernetworks in deep learning.
\newblock \emph{arXiv preprint arXiv:2306.06955}, 2023.

\bibitem[Chen(2018)]{torchdiffeq}
Ricky T.~Q. Chen.
\newblock torchdiffeq, 2018.
\newblock URL \url{https://github.com/rtqichen/torchdiffeq}.

\bibitem[Chen et~al.(2018)Chen, Rubanova, Bettencourt, and Duvenaud]{chen2018neural}
Ricky~TQ Chen, Yulia Rubanova, Jesse Bettencourt, and David~K Duvenaud.
\newblock Neural ordinary differential equations.
\newblock \emph{Advances in neural information processing systems}, 31, 2018.

\bibitem[Cuomo et~al.(2022)Cuomo, Di~Cola, Giampaolo, Rozza, Raissi, and Piccialli]{cuomo2022scientific}
Salvatore Cuomo, Vincenzo~Schiano Di~Cola, Fabio Giampaolo, Gianluigi Rozza, Maziar Raissi, and Francesco Piccialli.
\newblock Scientific machine learning through physics--informed neural networks: Where we are and what’s next.
\newblock \emph{Journal of Scientific Computing}, 92\penalty0 (3):\penalty0 88, 2022.

\bibitem[Daniels \& Nemenman(2015)Daniels and Nemenman]{daniels2015efficient}
Bryan~C Daniels and Ilya Nemenman.
\newblock Efficient inference of parsimonious phenomenological models of cellular dynamics using s-systems and alternating regression.
\newblock \emph{PloS one}, 10\penalty0 (3):\penalty0 e0119821, 2015.

\bibitem[DeepMind et~al.(2020)DeepMind, Babuschkin, Baumli, Bell, Bhupatiraju, Bruce, Buchlovsky, Budden, Cai, Clark, Danihelka, Dedieu, Fantacci, Godwin, Jones, Hemsley, Hennigan, Hessel, Hou, Kapturowski, Keck, Kemaev, King, Kunesch, Martens, Merzic, Mikulik, Norman, Papamakarios, Quan, Ring, Ruiz, Sanchez, Sartran, Schneider, Sezener, Spencer, Srinivasan, Stanojevi\'{c}, Stokowiec, Wang, Zhou, and Viola]{deepmind2020jax}
DeepMind, Igor Babuschkin, Kate Baumli, Alison Bell, Surya Bhupatiraju, Jake Bruce, Peter Buchlovsky, David Budden, Trevor Cai, Aidan Clark, Ivo Danihelka, Antoine Dedieu, Claudio Fantacci, Jonathan Godwin, Chris Jones, Ross Hemsley, Tom Hennigan, Matteo Hessel, Shaobo Hou, Steven Kapturowski, Thomas Keck, Iurii Kemaev, Michael King, Markus Kunesch, Lena Martens, Hamza Merzic, Vladimir Mikulik, Tamara Norman, George Papamakarios, John Quan, Roman Ring, Francisco Ruiz, Alvaro Sanchez, Laurent Sartran, Rosalia Schneider, Eren Sezener, Stephen Spencer, Srivatsan Srinivasan, Milo\v{s} Stanojevi\'{c}, Wojciech Stokowiec, Luyu Wang, Guangyao Zhou, and Fabio Viola.
\newblock The {D}eep{M}ind {JAX} {E}cosystem, 2020.
\newblock URL \url{http://github.com/google-deepmind}.

\bibitem[Diederik(2014)]{diederik2014adam}
P~Kingma Diederik.
\newblock Adam: A method for stochastic optimization.
\newblock \emph{(No Title)}, 2014.

\bibitem[Duan et~al.(2020)Duan, Rubin, and Swigon]{duan2020identification}
Xiaoyu Duan, JE~Rubin, and David Swigon.
\newblock Identification of affine dynamical systems from a single trajectory.
\newblock \emph{Inverse Problems}, 36\penalty0 (8):\penalty0 085004, 2020.

\bibitem[Dupont et~al.(2019)Dupont, Doucet, and Teh]{dupont2019augmented}
Emilien Dupont, Arnaud Doucet, and Yee~Whye Teh.
\newblock Augmented neural odes.
\newblock \emph{Advances in neural information processing systems}, 32, 2019.

\bibitem[Finn et~al.(2017)Finn, Abbeel, and Levine]{finn2017model}
Chelsea Finn, Pieter Abbeel, and Sergey Levine.
\newblock Model-agnostic meta-learning for fast adaptation of deep networks.
\newblock In \emph{International conference on machine learning}, pp.\  1126--1135. PMLR, 2017.

\bibitem[Garnelo et~al.(2018)Garnelo, Rosenbaum, Maddison, Ramalho, Saxton, Shanahan, Teh, Rezende, and Eslami]{garnelo2018conditional}
Marta Garnelo, Dan Rosenbaum, Christopher Maddison, Tiago Ramalho, David Saxton, Murray Shanahan, Yee~Whye Teh, Danilo Rezende, and SM~Ali Eslami.
\newblock Conditional neural processes.
\newblock In \emph{International conference on machine learning}, pp.\  1704--1713. PMLR, 2018.

\bibitem[Haber \& Ruthotto(2017)Haber and Ruthotto]{haber2017stable}
Eldad Haber and Lars Ruthotto.
\newblock Stable architectures for deep neural networks.
\newblock \emph{Inverse problems}, 34\penalty0 (1):\penalty0 014004, 2017.

\bibitem[Herde et~al.(2024)Herde, Raoni{\'c}, Rohner, K{\"a}ppeli, Molinaro, de~B{\'e}zenac, and Mishra]{herde2024poseidon}
Maximilian Herde, Bogdan Raoni{\'c}, Tobias Rohner, Roger K{\"a}ppeli, Roberto Molinaro, Emmanuel de~B{\'e}zenac, and Siddhartha Mishra.
\newblock Poseidon: Efficient foundation models for pdes.
\newblock \emph{arXiv preprint arXiv:2405.19101}, 2024.

\bibitem[Hewamalage et~al.(2023)Hewamalage, Ackermann, and Bergmeir]{hewamalage2023forecast}
Hansika Hewamalage, Klaus Ackermann, and Christoph Bergmeir.
\newblock Forecast evaluation for data scientists: common pitfalls and best practices.
\newblock \emph{Data Mining and Knowledge Discovery}, 37\penalty0 (2):\penalty0 788--832, 2023.

\bibitem[Hey et~al.(2020)Hey, Butler, Jackson, and Thiyagalingam]{hey2020machine}
Tony Hey, Keith Butler, Sam Jackson, and Jeyarajan Thiyagalingam.
\newblock Machine learning and big scientific data.
\newblock \emph{Philosophical Transactions of the Royal Society A}, 378\penalty0 (2166):\penalty0 20190054, 2020.

\bibitem[Kidger(2022)]{kidger2022neural}
Patrick Kidger.
\newblock On neural differential equations.
\newblock \emph{arXiv preprint arXiv:2202.02435}, 2022.

\bibitem[Kidger \& Garcia(2021)Kidger and Garcia]{kidger2021equinox}
Patrick Kidger and Cristian Garcia.
\newblock {E}quinox: neural networks in {JAX} via callable {P}y{T}rees and filtered transformations.
\newblock \emph{Differentiable Programming workshop at Neural Information Processing Systems 2021}, 2021.

\bibitem[Kidger et~al.(2020)Kidger, Morrill, Foster, and Lyons]{kidger2020neural}
Patrick Kidger, James Morrill, James Foster, and Terry Lyons.
\newblock Neural controlled differential equations for irregular time series.
\newblock \emph{Advances in Neural Information Processing Systems}, 33:\penalty0 6696--6707, 2020.

\bibitem[Kirchmeyer et~al.(2022)Kirchmeyer, Yin, Don{\`a}, Baskiotis, Rakotomamonjy, and Gallinari]{kirchmeyer2022generalizing}
Matthieu Kirchmeyer, Yuan Yin, J{\'e}r{\'e}mie Don{\`a}, Nicolas Baskiotis, Alain Rakotomamonjy, and Patrick Gallinari.
\newblock Generalizing to new physical systems via context-informed dynamics model.
\newblock In \emph{International Conference on Machine Learning}, pp.\  11283--11301. PMLR, 2022.

\bibitem[Kochkov et~al.(2024)Kochkov, Yuval, Langmore, Norgaard, Smith, Mooers, Kl{\"o}wer, Lottes, Rasp, D{\"u}ben, et~al.]{kochkov2024neural}
Dmitrii Kochkov, Janni Yuval, Ian Langmore, Peter Norgaard, Jamie Smith, Griffin Mooers, Milan Kl{\"o}wer, James Lottes, Stephan Rasp, Peter D{\"u}ben, et~al.
\newblock Neural general circulation models for weather and climate.
\newblock \emph{Nature}, 632\penalty0 (8027):\penalty0 1060--1066, 2024.

\bibitem[Koupa{\"\i} et~al.(2024)Koupa{\"\i}, Benet, Yin, Vittaut, and Gallinari]{koupai2024geps}
Armand~Kassa{\"\i} Koupa{\"\i}, Jorge~Misfut Benet, Yuan Yin, Jean-No{\"e}l Vittaut, and Patrick Gallinari.
\newblock Geps: Boosting generalization in parametric pde neural solvers through adaptive conditioning.
\newblock \emph{arXiv preprint arXiv:2410.23889}, 2024.

\bibitem[Lee \& Parish(2021)Lee and Parish]{lee2021parameterized}
Kookjin Lee and Eric~J Parish.
\newblock Parameterized neural ordinary differential equations: Applications to computational physics problems.
\newblock \emph{Proceedings of the Royal Society A}, 477\penalty0 (2253):\penalty0 20210162, 2021.

\bibitem[Li et~al.(2023)Li, Geng, and Evje]{li2023identification}
Qing Li, Jiahui Geng, and Steinar Evje.
\newblock Identification of the flux function of nonlinear conservation laws with variable parameters.
\newblock \emph{Physica D: Nonlinear Phenomena}, 451:\penalty0 133773, 2023.

\bibitem[Li et~al.(2019)Li, Zhu, and Tang]{li2019alternating}
Qiuwei Li, Zhihui Zhu, and Gongguo Tang.
\newblock Alternating minimizations converge to second-order optimal solutions.
\newblock In \emph{International Conference on Machine Learning}, pp.\  3935--3943. PMLR, 2019.

\bibitem[Li et~al.(2020)Li, Kovachki, Azizzadenesheli, Liu, Bhattacharya, Stuart, and Anandkumar]{li2020fourier}
Zongyi Li, Nikola Kovachki, Kamyar Azizzadenesheli, Burigede Liu, Kaushik Bhattacharya, Andrew Stuart, and Anima Anandkumar.
\newblock Fourier neural operator for parametric partial differential equations.
\newblock \emph{arXiv preprint arXiv:2010.08895}, 2020.

\bibitem[Liu et~al.()Liu, Cui, Yang, and Yang]{liustochastic}
Jiaqi Liu, Jiaxu Cui, Jiayi Yang, and Bo~Yang.
\newblock Stochastic neural simulator for generalizing dynamical systems across environments.

\bibitem[Ma et~al.(2021)Ma, Dixit, Innes, Guo, and Rackauckas]{ma2021comparison}
Yingbo Ma, Vaibhav Dixit, Michael~J Innes, Xingjian Guo, and Chris Rackauckas.
\newblock A comparison of automatic differentiation and continuous sensitivity analysis for derivatives of differential equation solutions.
\newblock In \emph{2021 IEEE High Performance Extreme Computing Conference (HPEC)}, pp.\  1--9. IEEE, 2021.

\bibitem[Massaroli et~al.(2020)Massaroli, Poli, Park, Yamashita, and Asama]{massaroli2020dissecting}
Stefano Massaroli, Michael Poli, Jinkyoo Park, Atsushi Yamashita, and Hajime Asama.
\newblock Dissecting neural odes.
\newblock \emph{Advances in Neural Information Processing Systems}, 33:\penalty0 3952--3963, 2020.

\bibitem[McCabe et~al.(2023)McCabe, Blancard, Parker, Ohana, Cranmer, Bietti, Eickenberg, Golkar, Krawezik, Lanusse, et~al.]{mccabe2023multiple}
Michael McCabe, Bruno R{\'e}galdo-Saint Blancard, Liam~Holden Parker, Ruben Ohana, Miles Cranmer, Alberto Bietti, Michael Eickenberg, Siavash Golkar, Geraud Krawezik, Francois Lanusse, et~al.
\newblock Multiple physics pretraining for physical surrogate models.
\newblock \emph{arXiv preprint arXiv:2310.02994}, 2023.

\bibitem[Mishra et~al.(2017)Mishra, Rohaninejad, Chen, and Abbeel]{mishra2017simple}
Nikhil Mishra, Mostafa Rohaninejad, Xi~Chen, and Pieter Abbeel.
\newblock A simple neural attentive meta-learner.
\newblock \emph{arXiv preprint arXiv:1707.03141}, 2017.

\bibitem[Mouli et~al.(2024)Mouli, Alam, and Ribeiro]{mouli2023metaphysica}
S~Chandra Mouli, Muhammad Alam, and Bruno Ribeiro.
\newblock Metaphysica: Improving ood robustness in physics-informed machine learning.
\newblock In \emph{The Twelfth International Conference on Learning Representations}, 2024.

\bibitem[Nzoyem et~al.(2023)Nzoyem, Barton, and Deakin]{nzoyem2023comparison}
Roussel~Desmond Nzoyem, David~AW Barton, and Tom Deakin.
\newblock A comparison of mesh-free differentiable programming and data-driven strategies for optimal control under pde constraints.
\newblock In \emph{Proceedings of the SC'23 Workshops of The International Conference on High Performance Computing, Network, Storage, and Analysis}, pp.\  21--28, 2023.

\bibitem[Owoyele \& Pal(2022)Owoyele and Pal]{owoyele2022chemnode}
Opeoluwa Owoyele and Pinaki Pal.
\newblock Chemnode: A neural ordinary differential equations framework for efficient chemical kinetic solvers.
\newblock \emph{Energy and AI}, 7:\penalty0 100118, 2022.

\bibitem[Parikh et~al.(2014)Parikh, Boyd, et~al.]{parikh2014proximal}
Neal Parikh, Stephen Boyd, et~al.
\newblock Proximal algorithms.
\newblock \emph{Foundations and trends{\textregistered} in Optimization}, 1\penalty0 (3):\penalty0 127--239, 2014.

\bibitem[Park et~al.(2023)Park, Berto, Jamgochian, Kochenderfer, and Park]{park2022first}
Junyoung Park, Federico Berto, Arec Jamgochian, Mykel Kochenderfer, and Jinkyoo Park.
\newblock First-order context-based adaptation for generalizing to new dynamical systems.
\newblock 2023.

\bibitem[Pearson(1993)]{pearson1993complex}
John~E Pearson.
\newblock Complex patterns in a simple system.
\newblock \emph{Science}, 261\penalty0 (5118):\penalty0 189--192, 1993.

\bibitem[Poli et~al.()Poli, Massaroli, Yamashita, Asama, Park, and Ermon]{politorchdyn}
Michael Poli, Stefano Massaroli, Atsushi Yamashita, Hajime Asama, Jinkyoo Park, and Stefano Ermon.
\newblock Torchdyn: Implicit models and neural numerical methods in pytorch.

\bibitem[Prigogine \& Lefever(1968)Prigogine and Lefever]{prigogine1968symmetry}
Ilya Prigogine and Ren{\'e} Lefever.
\newblock Symmetry breaking instabilities in dissipative systems. ii.
\newblock \emph{The Journal of Chemical Physics}, 48\penalty0 (4):\penalty0 1695--1700, 1968.

\bibitem[Qin et~al.()Qin, Yan, and Li]{qingeneralizing}
Tiexin Qin, Hong Yan, and Haoliang Li.
\newblock Generalizing to new dynamical systems via frequency domain adaptation.

\bibitem[Rackauckas et~al.(2020)Rackauckas, Ma, Martensen, Warner, Zubov, Supekar, Skinner, Ramadhan, and Edelman]{rackauckas2020universal}
Christopher Rackauckas, Yingbo Ma, Julius Martensen, Collin Warner, Kirill Zubov, Rohit Supekar, Dominic Skinner, Ali Ramadhan, and Alan Edelman.
\newblock Universal differential equations for scientific machine learning.
\newblock \emph{arXiv preprint arXiv:2001.04385}, 2020.

\bibitem[Raghu et~al.(2019)Raghu, Raghu, Bengio, and Vinyals]{raghu2019rapid}
Aniruddh Raghu, Maithra Raghu, Samy Bengio, and Oriol Vinyals.
\newblock Rapid learning or feature reuse? towards understanding the effectiveness of maml.
\newblock \emph{arXiv preprint arXiv:1909.09157}, 2019.

\bibitem[Ramachandran et~al.(2017)Ramachandran, Zoph, and Le]{ramachandran2017swish}
Prajit Ramachandran, Barret Zoph, and Quoc~V. Le.
\newblock Swish: a self-gated activation function.
\newblock \emph{arXiv: Neural and Evolutionary Computing}, 2017.
\newblock URL \url{https://api.semanticscholar.org/CorpusID:196158220}.

\bibitem[Rebuffi et~al.(2017)Rebuffi, Bilen, and Vedaldi]{rebuffi2017learning}
Sylvestre-Alvise Rebuffi, Hakan Bilen, and Andrea Vedaldi.
\newblock Learning multiple visual domains with residual adapters.
\newblock \emph{Advances in neural information processing systems}, 30, 2017.

\bibitem[Rebuffi et~al.(2018)Rebuffi, Bilen, and Vedaldi]{rebuffi2018efficient}
Sylvestre-Alvise Rebuffi, Hakan Bilen, and Andrea Vedaldi.
\newblock Efficient parametrization of multi-domain deep neural networks.
\newblock In \emph{Proceedings of the IEEE Conference on Computer Vision and Pattern Recognition}, pp.\  8119--8127, 2018.

\bibitem[Sel'Kov(1968)]{sel1968self}
EE~Sel'Kov.
\newblock Self-oscillations in glycolysis 1. a simple kinetic model.
\newblock \emph{European Journal of Biochemistry}, 4\penalty0 (1):\penalty0 79--86, 1968.

\bibitem[Serrano et~al.(2024)Serrano, Koupa{\"\i}, Wang, Erbacher, and Gallinari]{serrano2024zebra}
Louis Serrano, Armand~Kassa{\"\i} Koupa{\"\i}, Thomas~X Wang, Pierre Erbacher, and Patrick Gallinari.
\newblock Zebra: In-context and generative pretraining for solving parametric pdes.
\newblock \emph{arXiv preprint arXiv:2410.03437}, 2024.

\bibitem[Shen et~al.(2023)Shen, Appling, Gentine, Bandai, Gupta, Tartakovsky, Baity-Jesi, Fenicia, Kifer, Li, et~al.]{shen2023differentiable}
Chaopeng Shen, Alison~P Appling, Pierre Gentine, Toshiyuki Bandai, Hoshin Gupta, Alexandre Tartakovsky, Marco Baity-Jesi, Fabrizio Fenicia, Daniel Kifer, Li~Li, et~al.
\newblock Differentiable modelling to unify machine learning and physical models for geosciences.
\newblock \emph{Nature Reviews Earth \& Environment}, pp.\  1--16, 2023.

\bibitem[Spivak(1965)]{spivak2018calculus}
Michael Spivak.
\newblock \emph{Calculus on manifolds: a modern approach to classical theorems of advanced calculus}.
\newblock CRC press, 1965.

\bibitem[Stokes et~al.(1851)]{stokes1851effect}
George~Gabriel Stokes et~al.
\newblock On the effect of the internal friction of fluids on the motion of pendulums.
\newblock 1851.

\bibitem[Strogatz(2018)]{strogatz2018nonlinear}
Steven~H Strogatz.
\newblock \emph{Nonlinear dynamics and chaos: with applications to physics, biology, chemistry, and engineering}.
\newblock CRC press, 2018.

\bibitem[Subramanian et~al.(2024)Subramanian, Harrington, Keutzer, Bhimji, Morozov, Mahoney, and Gholami]{subramanian2024towards}
Shashank Subramanian, Peter Harrington, Kurt Keutzer, Wahid Bhimji, Dmitriy Morozov, Michael~W Mahoney, and Amir Gholami.
\newblock Towards foundation models for scientific machine learning: Characterizing scaling and transfer behavior.
\newblock \emph{Advances in Neural Information Processing Systems}, 36, 2024.

\bibitem[Takamoto et~al.(2023)Takamoto, Alesiani, and Niepert]{takamoto2023learning}
Makoto Takamoto, Francesco Alesiani, and Mathias Niepert.
\newblock Learning neural pde solvers with parameter-guided channel attention.
\newblock In \emph{International Conference on Machine Learning}, pp.\  33448--33467. PMLR, 2023.

\bibitem[Venturino(1994)]{venturino1994influence}
Ezio Venturino.
\newblock The influence of diseases on lotka-volterra systems.
\newblock \emph{The Rocky Mountain Journal of Mathematics}, pp.\  381--402, 1994.

\bibitem[Virtanen et~al.(2020)Virtanen, Gommers, Oliphant, Haberland, Reddy, Cournapeau, Burovski, Peterson, Weckesser, Bright, {van der Walt}, Brett, Wilson, Millman, Mayorov, Nelson, Jones, Kern, Larson, Carey, Polat, Feng, Moore, {VanderPlas}, Laxalde, Perktold, Cimrman, Henriksen, Quintero, Harris, Archibald, Ribeiro, Pedregosa, {van Mulbregt}, and {SciPy 1.0 Contributors}]{2020SciPy-NMeth}
Pauli Virtanen, Ralf Gommers, Travis~E. Oliphant, Matt Haberland, Tyler Reddy, David Cournapeau, Evgeni Burovski, Pearu Peterson, Warren Weckesser, Jonathan Bright, St{\'e}fan~J. {van der Walt}, Matthew Brett, Joshua Wilson, K.~Jarrod Millman, Nikolay Mayorov, Andrew R.~J. Nelson, Eric Jones, Robert Kern, Eric Larson, C~J Carey, {\.I}lhan Polat, Yu~Feng, Eric~W. Moore, Jake {VanderPlas}, Denis Laxalde, Josef Perktold, Robert Cimrman, Ian Henriksen, E.~A. Quintero, Charles~R. Harris, Anne~M. Archibald, Ant{\^o}nio~H. Ribeiro, Fabian Pedregosa, Paul {van Mulbregt}, and {SciPy 1.0 Contributors}.
\newblock {{SciPy} 1.0: Fundamental Algorithms for Scientific Computing in Python}.
\newblock \emph{Nature Methods}, 17:\penalty0 261--272, 2020.
\newblock \doi{10.1038/s41592-019-0686-2}.

\bibitem[Voynov \& Babenko(2020)Voynov and Babenko]{voynov2020unsupervised}
Andrey Voynov and Artem Babenko.
\newblock Unsupervised discovery of interpretable directions in the gan latent space.
\newblock In \emph{International conference on machine learning}, pp.\  9786--9796. PMLR, 2020.

\bibitem[Wang et~al.(2021)Wang, Zhao, and Li]{wang2021bridging}
Haoxiang Wang, Han Zhao, and Bo~Li.
\newblock Bridging multi-task learning and meta-learning: Towards efficient training and effective adaptation.
\newblock In \emph{International conference on machine learning}, pp.\  10991--11002. PMLR, 2021.

\bibitem[Wang et~al.(2022)Wang, Walters, and Yu]{wang2022meta}
Rui Wang, Robin Walters, and Rose Yu.
\newblock Meta-learning dynamics forecasting using task inference.
\newblock \emph{Advances in Neural Information Processing Systems}, 35:\penalty0 21640--21653, 2022.

\bibitem[Wanner \& Hairer(1996)Wanner and Hairer]{wanner1996solving}
Gerhard Wanner and Ernst Hairer.
\newblock \emph{Solving ordinary differential equations II}, volume 375.
\newblock Springer Berlin Heidelberg New York, 1996.

\bibitem[Weinan(2017)]{weinan2017proposal}
Ee~Weinan.
\newblock A proposal on machine learning via dynamical systems.
\newblock \emph{Communications in Mathematics and Statistics}, 1\penalty0 (5):\penalty0 1--11, 2017.

\bibitem[Wu et~al.(2012)Wu, Liu, and Wang]{wu2012grey}
Lifeng Wu, Sifeng Liu, and Yinao Wang.
\newblock Grey lotka--volterra model and its application.
\newblock \emph{Technological Forecasting and Social Change}, 79\penalty0 (9):\penalty0 1720--1730, 2012.

\bibitem[Yin et~al.(2021{\natexlab{a}})Yin, Ayed, de~B{\'e}zenac, Baskiotis, and Gallinari]{yin2021leads}
Yuan Yin, Ibrahim Ayed, Emmanuel de~B{\'e}zenac, Nicolas Baskiotis, and Patrick Gallinari.
\newblock Leads: Learning dynamical systems that generalize across environments.
\newblock \emph{Advances in Neural Information Processing Systems}, 34:\penalty0 7561--7573, 2021{\natexlab{a}}.

\bibitem[Yin et~al.(2021{\natexlab{b}})Yin, Le~Guen, Dona, de~B{\'e}zenac, Ayed, Thome, and Gallinari]{yin2021augmenting}
Yuan Yin, Vincent Le~Guen, J{\'e}r{\'e}mie Dona, Emmanuel de~B{\'e}zenac, Ibrahim Ayed, Nicolas Thome, and Patrick Gallinari.
\newblock Augmenting physical models with deep networks for complex dynamics forecasting.
\newblock \emph{Journal of Statistical Mechanics: Theory and Experiment}, 2021\penalty0 (12):\penalty0 124012, 2021{\natexlab{b}}.

\bibitem[Zhuang et~al.(2020)Zhuang, Tang, Ding, Tatikonda, Dvornek, Papademetris, and Duncan]{zhuang2020adabelief}
Juntang Zhuang, Tommy Tang, Yifan Ding, Sekhar~C Tatikonda, Nicha Dvornek, Xenophon Papademetris, and James Duncan.
\newblock Adabelief optimizer: Adapting stepsizes by the belief in observed gradients.
\newblock \emph{Advances in neural information processing systems}, 33:\penalty0 18795--18806, 2020.

\bibitem[Zintgraf et~al.(2019)Zintgraf, Shiarli, Kurin, Hofmann, and Whiteson]{zintgraf2019fast}
Luisa Zintgraf, Kyriacos Shiarli, Vitaly Kurin, Katja Hofmann, and Shimon Whiteson.
\newblock Fast context adaptation via meta-learning.
\newblock In \emph{International Conference on Machine Learning}, pp.\  7693--7702. PMLR, 2019.

\end{thebibliography}
\bibliographystyle{references}

\newpage

\appendix


\DoToC

\newpage

\section{Algorithms \& Proofs}
\label{app:proofs}

\subsection{Second-order Taylor expansion with JVPs}

For ease of demonstration, we propose an equivalent formulation of \cref{prop:t2} that disregards the first input of $f$ not necessary for its proof. Although we will consistently write terms like $f(x)$, we emphasize that $x$ is meant to stand for the context, not the state variable.

\begingroup
\def\theproposition{\ref{prop:t2}}
\begin{proposition}[Second-order Taylor expansion with JVPs] 
    Assume $f: \mathbb{R}^{d_{\xi}} \rightarrow \mathbb{R}^{d}$ is $\mathcal{C}^2$. Let $x \in \mathbb{R}^{d_\xi}$, and define $g: y \mapsto \nabla f(y)(x-y) $. The second-order Taylor expansion of $f$ around any $ x_0 \in \mathbb{R}^{d_{\xi}}$ is then expressed as
    \begin{align*}
        f(x) &= f(x_0) + \frac{3}{2} g(x_0) 
        + \frac{1}{2} \nabla g(x_0) (x - x_0) 
        + o(\Vert x - x_0 \Vert^2).
    \end{align*}
\end{proposition}
\addtocounter{proposition}{-1}
\endgroup

\begin{proof}

Let $ x, x_0 \in \mathbb{R}^{d_{\xi}}$. The second-order Taylor expansion of $f$ includes its Hessian that we view as a 3-dimensional tensor, and we contract along its last axis such that
\begin{align} \label{eq:1}
    f(x) = f(x_0) + \nabla f (x_0) (x-x_0) + \frac{1}{2} \left[ \nabla^2 f(x_0) (x-x_0) \right] (x-x_0) + o(\Vert x - x_0 \Vert^2). 
\end{align}

Next we define $g: y \mapsto \nabla f(y)(x-y) $, and we consider a small perturbation $h \in \mathbb{R}^{d_{\xi}}$ to write
\begin{align*}
    g(y+h) &= \nabla f(y+h) (x-y-h) \\
    &=  \left[ \nabla f (y) + \nabla^2 f(y)(h) + o(\Vert h \Vert) \right] (x-y-h)     \qquad \text{\ding{43} } \text{Taylor expansion of $\nabla f$}    \\
    &= \nabla f(y)(x-y) - \nabla f(y) h + \left[ \nabla^2 f(y) h \right] (x-y) + o(\Vert h \Vert) + \cancel {o(\Vert h \Vert^2)}
\end{align*}

The Hessian is by definition symmetric along its last two axes, which allows us to rewrite the third term as $\left[ \nabla^2 f(y) (x-y) \right] h $. We then have    
\begin{align*}
    g(y+h) = g(y) + \left[ \nabla^2 f(y) (x-y) - \nabla f (y) \right] h + o(\Vert h \Vert),
\end{align*}
which indicates that $\nabla g (y) := \nabla^2 f(y) (x-y) - \nabla f (y) $, from which we derive
$$ \forall y \in \mathbb{R}^{d_{\xi}}, \quad \nabla^2 f(y) (x-y) = \nabla g (y) + \nabla f (y).$$ In particular,
\begin{align} \label{eq:2}
    \nabla^2 f(x_0) (x-x_0) = \nabla g (x_0) + \nabla f (x_0).
\end{align}

Plugging this into \cref{eq:1}, we have 
\begin{align*}
    f(x) &= f(x_0) + \nabla f (x_0) (x-x_0) + \frac{1}{2} \left[ \nabla g(x_0) + \nabla f(x_0) \right] (x-x_0) + o(\Vert x - x_0 \Vert^2) \\
    &= f(x_0) + \frac{3}{2} \underbrace{\nabla f (x_0) (x-x_0)}_{g(x_0)} + \frac{1}{2} \nabla g(x_0) (x-x_0) + o(\Vert x - x_0 \Vert^2).
\end{align*}

This concludes the proof.
\end{proof}
While this expression of the second-order Taylor expansion of a vector-valued function makes its implementation memory-efficient via Automatic Differentiation (AD), it still relies on nested derivatives, which scale exponentially with the order of the Taylor expansion due to avoidable recomputations. For even higher-order Taylor expansions, this scaling is frightfully inefficient. Taylor-Mode AD is a promising avenue to address this issue \cite{bettencourt2019taylor}.




\subsection{Identifiability of physical parameters}
\label{subsec:interpret}

Identifying the underlying physical parameters of a system is of enduring interest to scientific machine learning practitioners. Our work critically builds on CAMEL \cite{blanke2024interpretable} which extensively studies a model similar to \tone as a linearly parametrized system. Its Proposition 1 states, informally, that in the limit of vanishing training loss $\mathcal{L}(\cdot, \cdot, \cdot) = 0$, the relationship between the learned context vectors and the system parameters is linear and can be estimated using ordinary least squares. This forces the model to learn a meaningful representation of the system instead of overfitting the examples from the training tasks. 

{\color{rebutcolor}

We now reformulate \citet{blanke2024interpretable}'s identifiability result for linear systems trained with first-order Taylor expansion and the Random-All pool-filling strategy (see \cref{subsec:whatsinapool}), then we provide an alternative proof suited to our setting. Building on notations from \cref{eq:ode,eq:neuralode}, we restate that our goal is to approximate the true vector field based on
\begin{align*}
    \frac{\md x}{\md t}(t) = f_{\text{true}}(x(t), c), \qquad \text{and } \qquad \frac{\md x}{\md t}(t) = f_{\theta}(x(t), \xi), \qquad \forall t \in \left[ 0,T \right].
\end{align*}
We drop the dependence on $x \in \mathbb{R}^{d}$ to ease notations in the sequel (the $\nabla$ will henceforth indicate gradients wrt $\xi$). As such, the predictor $f_{\theta} :\mathbb{R}^{d_{\xi}} \rightarrow \mathbb{R}^d$ is parametrized as a neural network and learned, while $f_{\text{true}} : \mathbb{R}^{d_{c}} \rightarrow \mathbb{R}^d $ is known and affine, i.e. $\exists P \in \mathbb{R}^{d\times d_{c}} \text { and } p \in \mathbb{R}^{d} $ such that $f_{\text{true}}(c) = Pc + p$.

\begingroup
\def\theproposition{\ref{prop:linearsystems}}
\begin{proposition}[Identifiability of affine systems]
Assume $d_{\xi} \geq d_c$, that $P$ is full-rank, and that $f_{\textup{true}}$ is differentiable. In the limit of zero training loss in \cref{eq:bigloss}, $f_{\theta}$ is affine on an open region of $\mathbb{R}^{d_{\xi}}$. Furthermore, there exists $Q \in \mathbb{R}^{d_c \times d_{\xi}}$ and $q \in \mathbb{R}^{d_c}$ such that for any meta-trained $ \xi \in \{\xi^e\}_{e=1}^{m}$ and its corresponding underlying parameter $c \in \{c^e\}_{e=1}^{m}$, we have $c = Q\xi + q$. 
\end{proposition}
\addtocounter{proposition}{-1}
\endgroup

\begin{proof}
    Let $e \in \lbint 1,m \rbint$. In the limit of zero training loss, $f_{\theta}$ coincides with its first-order Taylor expansion in a neighborhood $U(\xi^e)$ which contains all other $\{\xi^j \}_{j=1}^m$. We can write
    \begin{align} \label{eq:eqA}
        f_{\theta}(\xi) = f_{\theta}(\xi^e) + A (\xi - \xi^e),  \qquad \forall \, \xi \in U(\xi^e)
    \end{align}
    where $A = \nabla f_{\theta}(\xi^e)$ is constant.

    Similarly, for $j \in \lbint 1,m \rbint$, there exists an open set $U(\xi^j)$ which contains all other $\{\xi^e \}_{e=1}^m$, such that
    \begin{align} \label{eq:eqB}
        f_{\theta}(\xi) = f_{\theta}(\xi^j) + B (\xi - \xi^j),       \qquad \forall \, \xi \in U(\xi^j)
    \end{align}
     where $B = \nabla f_{\theta}(\xi^j)$ is constant.

    To show that necessarily $A=B$, let's consider without loss of generality $\xi \in U(\xi^e) \cap U(\xi^j)$ (which is non-empty since both $\xi^e$ and $\xi^j$ are included in both sets) and proceed by contradiction, assuming $A\neq B$. Let $v \in \mathbb{R}^{d_{\xi}}$ sufficiently small, we can set $\xi = \xi^e + tv$ in \cref{eq:eqA}  and write the directional derivative 
    \begin{align}
        \lim_{t\to 0} \frac{f_{\theta}(\xi^e + tv) - f_{\theta}(\xi^e)}{t} &= \lim_{t\to 0} \frac{f_{\theta}(\xi^e) + A (\xi^e + tv - \xi^e) - f_{\theta}(\xi^e)}{t} \notag \\
        &= \lim_{t\to0} \frac{tAv}{t} \notag \\
        &= Av.
    \end{align}
    We also write, using \cref{eq:eqB}, the same directional derivative as
    \begin{align}
        \lim_{t\to 0} \frac{f_{\theta}(\xi^e + tv) - f_{\theta}(\xi^e)}{t} &= \lim_{t\to 0} \frac{f_{\theta}(\xi^j) + B (\xi^e + tv - \xi^j) - f_{\theta}(\xi^e)}{t} \notag \\
        &= \lim_{t\to 0} \cancel{\frac{f_{\theta}(\xi^j) - f_{\theta}(\xi^e) + B (\xi^e - \xi^j)}{t}} + \lim_{t\to0} \frac{tBv}{t} \notag\\
        &= Bv.
    \end{align}
    For $v \notin \ker{(B-A)}$, we have $Av \neq Bv$ which contradicts with the uniqueness of directional derivatives for differentiable functions \cite{spivak2018calculus}. This shows that $f_{\theta}$ is affine on the open set $U = \bigcap\limits_{e=1}^{m} U(\xi^e)$.

    Furthermore, using \cref{eq:eqA}, we have 
    \begin{align} \label{eq:rhs_proof}
        f_{\theta}(\xi) - A \xi =  f_{\theta}(\xi^e) - A\xi^e, \qquad \forall \, \xi \in U
    \end{align}
    which is valid for all $e \in \lbint 1, m \rbint$. This indicates that the right hand side of \cref{eq:rhs_proof} is constant, i.e. $\exists \, \tilde q \in \mathbb{R}^{d_{\xi}}  $ such that $\forall \, e \in \lbint 1, m \rbint$, $f_{\theta}(\xi^e) - A\xi^e = \tilde q$ . We then use the fact that in the limit of zero training loss, the predicted and true vector fields coincide for a context $\xi \in U$ and its corresponding underlying parameters $c$ : 
    \begin{align*}
        f_{\text{true}}(c) = f_{\theta}(\xi) \Rightarrow Pc + p &= A(\xi - \xi^e) + f_{\theta}(\xi^e) \qquad \text{for} \, e \in \lbint 1,m \rbint \\ &= A\xi + \tilde q.
    \end{align*}
    Since $d_{\xi} \geq d_c$ and $P$ is full-rank, its rows are linearly independent, guaranteeing the existence of a pseudo-inverse. We can thus write
    $c = Q\xi + q$ with 
    \begin{align} \label{eq:closedform}
        Q = (P^T P)^{-1} P^T A, \qquad \text{ and } \qquad q = (P^TP)^{-1} P^T [\tilde q - p].
    \end{align}    
\end{proof}

The closed form expression \cref{eq:closedform} can be challenging to derive, especially since $P$ and $p$ might not be fully known when collecting data. So similar to \cite{blanke2024interpretable}, one can perform post-training, ordinary least squares regression on observed $\{ c^e \}_{e=1}^{m'}$ (with $m' \leq m$) to estimate the optimal $Q^*$ and $q^*$ :
\begin{align}
    Q^*, q^* \in \underset{Q, q}{\text{argmin}} \frac{1}{2} \sum_{e=1}^{m'} \Vert Q\xi^e +q - c^e \Vert^2_2. 
\end{align}
}

\paragraph{Experimental validation.} We validate \cref{prop:linearsystems} by modifying the Lotka-Volterra (LV) experiment. For CoDA \cite{kirchmeyer2022generalizing} and \tone, we use the exact same network architecture: a 4-layer MLP with 224 hidden units. This means no context nor state network is used in \tone: the context vector of size $d_{\xi}=2$ is directly concatenated to the state vector as done by \citet{zintgraf2019fast}. We note that the results obtained using this configuration further indicate the superiority of NCF, even when model comparison centers on their main/root networks (see also \cite{park2022first} for a similar model comparison based on parameter count)\footnote{As reported in \cref{app:detailes}, the primary model comparison approach in this work counts all learnable parameters, including hypernetworks if involved. Only the context vectors are exempt from this count.}. 

After meta-training and meta-testing, we set out to recover the underlying parameters of the Lotka-Volterra systems via a linear transformation of the learned context vectors. We fit a linear regression model to the 9 meta-training context vectors, using the true physical parameters as supervision signal. We test on the 4 adaptation contexts. The results, displayed in \cref{fig:interpret_ncf_coda}, adequately illustrate interpretability as stated in \cite[Proposition 1]{blanke2024interpretable} and our \cref{prop:linearsystems}. They show that our trained meta-parameters recover the underlying system parameters up to a linear transform, and thus enable \textbf{zero-shot} (physical parameter-induced) adaptation via inverse regression.

{\color{rebutcolor}
\paragraph{Robustness to noise.} Additionally, system identification with NCFs is robust to noise in the trajectory. We show this empirically by corrupting the single trajectory in each adaptation environment with a Gaussian noise scaled by a factor of $\eta$. Upon addition of this noise, sequential adaptation is performed to recover new $\xi$ which are then transformed into $c$ and plotted. The weight $Q$ and bias $q$ of the affine transform are fitted on the training environments and their corresponding underlying parameters, which are unchanged across all noise levels. \cref{fig:noiselevels} shows that the reconstruction\footnote{Sample context vectors pre-reconstruction can be observed in \cref{fig:poolingcontexts} as the structure is preserved with changing seeds.} MSE remains low despite the noise (compared to CoDA's MSE of $1.57\times 10^{-2}$ when $\eta=0$ shown in \cref{fig:interpret_ncf_coda}), and the physical system remains visually identifiable, especially in the convex hull of training environments. Outside the convex hull, the identifiability is notably worse with $\eta\geq 0.1$ indicating that this noise level is excessively high given the range of the LV state values.
}

\begin{figure}[h]
\vspace*{-0.1cm}
\centering
\includegraphics[width=0.3\textwidth]{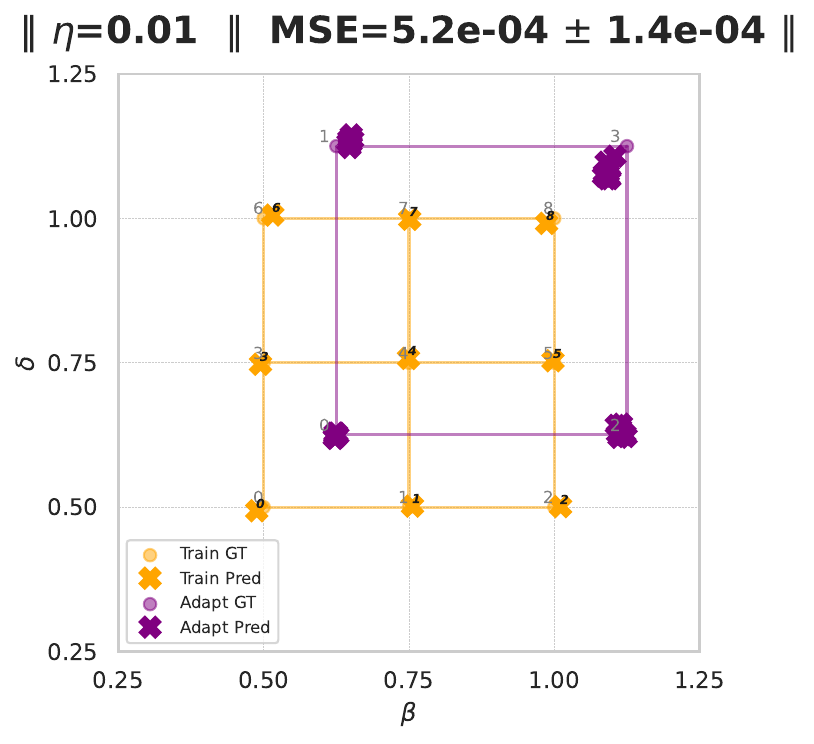} \hfill
\includegraphics[width=0.3\textwidth]{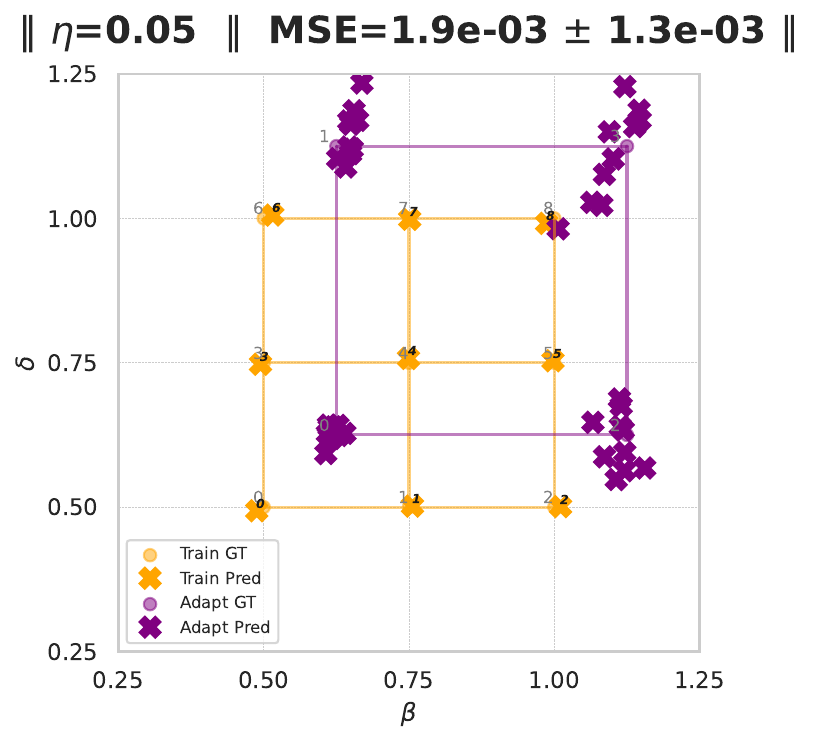} \hfill
\includegraphics[width=0.3\textwidth]{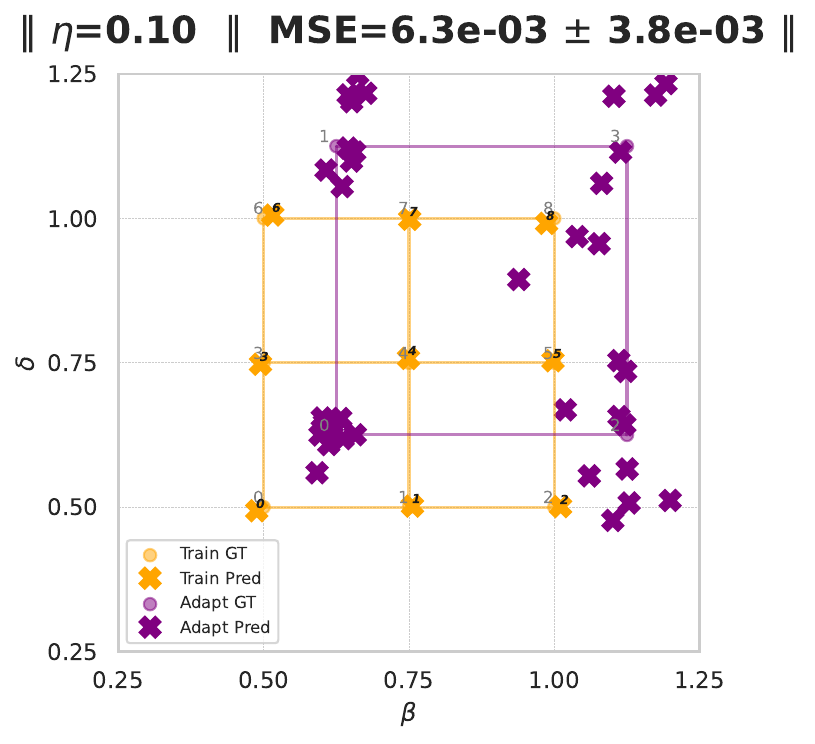}
\caption{\rebut{Robustness of \tone when identifying physical parameters with varying levels of noise $\eta$ injected into the adaptation trajectory. The MSE is reported with the standard deviation across 10 runs with different seeds for the Gaussian noise. (Left) $\eta=0.01$, (Middle) $\eta=0.05$, (Right) $\eta=0.1$.}}
\label{fig:noiselevels}
\vspace*{-0.2cm}
\end{figure}

\subsection{Convergence of proximal alternating minimization}
\label{app:convproximal}

For clarity of exposition, \cref{the:convergence} expressing the convergence of \cref{alg:ncf_proximal} to second-order stationary points is repeated below.

\begingroup
\def\thetheorem{\ref{the:convergence}}
\begin{theorem}[Convergence to second-order stationary points]
Assume that $\mathcal{L}(\cdot, \cdot, \trainset)$ satisfies the Kurdyka-Lojasiewicz (KL) property, is $L$ bi-smooth, and $\nabla \mathcal L(\cdot, \cdot, \trainset)$ is Lipschitz continuous on any bounded subset of domain $\mathbb{R}^{d_\theta} \times \mathbb{R}^{d_{\xi} \times m}$. Under those assumptions, let $(\theta_0, \xi^{1:m}_0)$ be a random initialization and $(\theta_q, \xi^{1:m}_q)$ be the sequence generated by \cref{alg:ncf_proximal}. If the sequence $(\theta_q, \xi^{1:m}_q)$ is bounded, then it converges to a second-order stationary point of $\mathcal{L}(\cdot, \cdot, \trainset)$ almost surely.
\end{theorem}
\addtocounter{theorem}{-1}
\endgroup

We use this section to briefly emphasize that the assumptions of KL property \cite{attouch2010proximal} and Lipschitz continuity are mild and easily achievable with neural networks. Interestingly, boundedness of the sequence $(\theta_k, \xi^{1:m}_k)$ is guaranteed a priory if $\mathcal L$ is coercive \cite{li2019alternating}, a property we encouraged by regularizing $\ell$ wrt weights and contexts as in \cref{eq:loss}.

\subsection{Ordinary alternating minimization}

Here we provide an additional procedure (\cref{alg:ncf}) for training Neural Context Flows which we reserved for the \tone variant. While stronger assumptions \cite{li2019alternating} are needed to establish convergence guarantees like its proximal extension, is it relatively easier to implement, and exposes fewer hyperparameters to tune.

\begin{minipage}[t]{0.48\textwidth}

\begin{algorithm}[H]
   \caption{Ordinary Alternating Minimization}
   \label{alg:ncf}
\begin{algorithmic}[1]
   \State {\bfseries Input:} 
   $ \trainset := \{ \mathcal{D}^e_{\text{tr}} \}_{e\in \lbint 1,m \rbint}$
   \State $\,\, \theta \in \mathbb{R}^{d_{\theta}}$ randomly initialized 
   \State $\,\, \xi^{1:m} := \bigcup\limits_{e=1}^m \xi^e$, where $ \xi^e = \mathbf{0} \in \mathbb{R}^{d_{\xi}} $
   \State $\,\, \eta_{\theta}, \eta_{\xi} > 0$
   \Repeat
   \State \quad $\theta \leftarrow \theta - \eta_{\theta} \nabla_{\theta} \mathcal{L}(\theta, \xi^{1:m}, \trainset)$
   \State \quad $\xi^{1:m} \leftarrow \xi^{1:m} - \eta_{\xi} \nabla_{\xi} \mathcal{L}(\theta, \xi^{1:m}, \trainset)$
   \Until{$\left( \theta, \xi^{1:m} \right) $ converges}
\end{algorithmic}
\end{algorithm}

\end{minipage}
\hfill
\begin{minipage}[t]{0.48\textwidth}

\begin{algorithm}[H]
   \caption{Bulk Adaptation of NCF}
   \label{alg:ncf_adp_bulk}
\begin{algorithmic}[1]
   \State {\bfseries Input:} 
   $ \adaptset := \{ \mathcal{D}^{e'} \}_{e'\in \lbint a,b \rbint}$
   \State $\,\, \theta \in \mathbb{R}^{d_{\theta}}$ learned 
   \State $\,\, \xi^{a:b} = \bigcup\limits_{e'=a}^b \xi^{e'}$, where $ \xi^{e'} := \mathbf{0} \in \mathbb{R}^{d_{\xi}} $
   \State $\,\, \eta > 0$
   \Repeat
   \State \quad $\xi^{a:b} \leftarrow \xi^{a:b} - \eta \nabla_{\xi} \mathcal{L}(\theta, \xi^{a:b}, \adaptset)$
    \Until{$\xi^{a:b} $ converges}
\end{algorithmic}
\end{algorithm}
\end{minipage}

\subsection{Bulk adaptation algorithm}
\label{app:bulk}


Neural Context Flows are designed to be fast at adaptation time. However, one might want to adapt to hundreds or thousands of environments, perhaps to identify where the performance degrades on downstream tasks. In such cases, \cref{alg:ncf_adp_seq} could be slow due to its sequential nature. We provide \cref{alg:ncf_adp_bulk} that leverages the same parallelism exploited during training (while restricting information flow from one adaptation environment to the next). 

Although highly parallelizable, we realize in practice that \cref{alg:ncf_adp_bulk} is susceptible to two pitfalls:
\begin{enumerate}
\item{\textbf{Memory scarcity}:} the operating system needs to allocate enough resources to store the data $\adaptset$, the combined context vectors $\xi^{a:b} \in \mathbb{R}^{d_{\xi}\times (b-a+1)}$ and backpropagate their gradients in order to adapt all environments at once, which might be impossible if the context pool size $p$ is set too high. To avoid this issue, we recommend using $p \ll m$, or ideally $p=1$ if contextual self-modulation is disabled.
\item{\textbf{Slower convergence}:} the bulk algorithm could take longer to converge to poorer context vectors when the jointly adapted environments are all very far apart. This is because the contextual self-modulation process would be rendered useless by such task discrepancy, and sampling $j=e$ in the context pool $\mathrm{P}$ to return to a standard Neural ODE in \cref{method} would be harder. We find in practice that manually setting $j=e$ by adjusting the vector field to \emph{forego the Taylor expansion} works well (see \cref{fig:huge_adapt}). A more direct way of achieving the same result is to disregard $\mathrm{P}$ and retain $f_{\theta}$ (rather than its Taylor expansion) in \cref{eq:method}.
\end{enumerate}

All adaptation results in this paper use the sequential adaptation procedure in \cref{alg:ncf_adp_seq}, except for the grid-wise adaptation in \cref{experiments}.

\subsection{What's in a context pool ?}
\label{subsec:whatsinapool}

The content of the context pool $\mathrm{P}$ not only defines the candidate trajectories we get from \cref{eq:candidates}, but also the speed and memory cost of the meta-training process. \rebut{Based on intuitive understanding of the role of $\mathrm{P}$,} we outline 3 \rebut{tunable} pool-filling strategies for selecting $p$ neighboring contexts: 
\begin{itemize}
    \item \textbf{Random-All (RA)}: all $p$ distinct contexts can be selected by randomly drawing their indices from $\lbint 1,m \rbint$. By repeatedly doing so, we maximize long-range interactions to provide the broadest form of self-modulation -- since information can always (in the stochastic limit) flow from any environment into $e$.
    \item \textbf{Nearest-First (NF)}: only the $p$ closest contexts to $\xi^e$ are selected, thus encouraging environments to form clusters. (Note, however, that if $p=1$, then $\mathrm{P} = \{ e \}$ itself, and no self-modulation occurs.) \rebut{In one ablation study, we observe that this strategy is the most balanced with regard to training time and performance (see \cref{app:pooling}).} 
    \item \textbf{Smallest-First (SF)}: the smallest contexts in $L^1$ norm are selected first. Since an environment with context close to \textbf{0} can be interpreted as an environment-agnostic feature like in \cite{kirchmeyer2022generalizing,blanke2024interpretable}, this strategy prioritizes the flow of information from that base or canonical environment to the one of interest $e$. 
\end{itemize}

\subsection{Scalability of the NCF algorithms}
\label{app:scalability}

The Neural Context Flow (NCF) framework demonstrates excellent scalability with respect to various hyperparameters, owing to its innovative design and implementation. This scalability is evident in three key aspects: distributed training capabilities, efficient handling of large context vectors, and utilization of first-order optimization techniques.

Primarily, NCF's training process can be distributed and parallelized across environments and trajectories (see \cref{eq:bigloss}, \cref{fig:trainigtimes}, and \cref{alg:ncf_adp_bulk}), a feature that distinguishes it from baseline methods which are limited to parallelization across trajectories. Furthermore, the framework employs an efficient approach, as outlined in Proposition 3.1, to avoid materializing Jacobians or Hessians wrt potentially large context vectors, thereby significantly enhancing scalability. Additionally, NCF utilizes only first-order gradients wrt model weights $\theta$, in contrast to methods like CAVIA that require second-order information in their bi-level optimization loop.

Scalability can also be evaluated in terms of component size, particularly in meta-learning adaptation rules that incorporate additional contextual parameters beyond shared model weights. These components may include encoders, hypernetworks, and other mechanisms for generating or refining context vectors necessary for task adaptation. The complexity and memory requirements of contextual meta-learning rules, which are directly related to the size of these components, can be quantified by their parameter count among other metrics. In this regard, NCF maintains a constant memory cost $O(1)$, while baseline methods such as CAVIA and CoDA require additional memory to produce better contexts (see \cite[Table 1]{park2022first}).

Despite these advantages, the NCF framework may face challenges related to memory and computational efficiency due to the requirement of solving $p$ Neural ODEs in \eqref{method}, as opposed to a single one. In scenarios where all training environments are utilized (i.e., $p=m$), this results in a quadratic cost $O(m^2)$ for \cref{alg:ncf_proximal,alg:ncf}. However, our ablation studies in \cref{app:contextpoolsize} demonstrate that competitive performance can be achieved on most problems using as few as $p=2$ neighboring environments. Additional studies in \cref{app:contextsize,app:3networks} establish the necessity of expressive context vectors \cite{voynov2020unsupervised} and validate the efficacy of the 3-networks architecture, respectively.

It is worth noting that limiting these quantities could directly contribute to improved parameter counts and more interpretable models. Moreover, restricting the total number of environments contributing to the loss \eqref{eq:bigloss} at each iteration may further enhance efficiency.


\newpage
\section{Datasets \& Additional Results}
\label{app:additionalresults}
\label{app:datasets}

\subsection{Gen-Dynamics}
\label{subsec:gen-dyn}

Given the lack of benchmark consistency in Scientific Machine Learning \cite{massaroli2020dissecting}, we launched \texttt{Gen-Dynamics}: \url{https://github.com/ddrous/gen-dynamics}. This is a call for fellow authors to upload their metrics and datasets, synthetic or otherwise, while following a consistent interface. In the context of OoD generalization for instance, we suggest the dataset be split in 4 parts: \textbullet{} (1) \texttt{train}: For In-Domain meta-training; \textbullet{} (2) \texttt{test}: For In-Domain evaluation; \textbullet{} (3) \texttt{ood\_train}: For Out-of-Distribution adaptation to new environments (meta-testing); \textbullet{} (4) \texttt{ood\_test}: For OoD evaluation.

Each split should contain trajectories \texttt{X} and the time points \texttt{t} at which the states were recorded. The time \texttt{t} is a 1-dimensional array, and we recommend a 4-dimensional \texttt{X} tensor with dimensions described as follows: \textbullet{} (1) \texttt{nb\_envs}: Number of distinct environments; \textbullet{} (2) \texttt{nb\_trajs\_per\_env}: Number of trajectories per environment; \textbullet{} (3) \texttt{nb\_steps\_per\_traj}: Number of time steps per trajectory (matching the size of \texttt{t}); \textbullet{} (4) \texttt{state\_size}: Size of the state space.

While the suggestions apply mostly to dynamical systems' Meta-Learning, we believe they are generalizable to other problems, and are open to suggestions from the community. All problems described below are now represented in \texttt{Gen-Dynamics}.

Finally, we define the MSE and the mean absolute percentage error (MAPE) criteria as they apply to all trajectory data found in \texttt{Gen-Dynamics}. Unless stated otherwise, the following are the metrics used throughout this work, including in \cref{tab:sota_results}.
\begin{align}
    \text{MSE}(x, \hat x) &= \frac{1}{N \times d} \sum_{n=1}^{N} \Vert x(t_n) - \hat{x}(t_n) \Vert_2^2, \\ \text{MAPE}(x, \hat x) &= \frac{1}{N \times d} \sum_{n=1}^{N} \left| \frac{x(t_n) - \hat{x}(t_n)}{x(t_n)} \right| \times 100.
\end{align}

{\color{rebutcolor}
\subsection{Uncertainty Estimation}
\label{subsec:uncertainty}

Neural Context Flows can provide uncertainty about their predictions. To show this, we calculate  (\textit{i}) the relative mean squared error (MSE), (\textit{ii}) the mean absolute percentage error (MAPE), and (\textit{iii}) the 3-$\sigma$ Coverage or Confidence Level (CL) \cite{serrano2024zebra}, with the following formulae applied In-Domain:
\begin{align} \label{eq:uqmse}
    \text{Rel. MSE} = \frac{100}{m \times S \times N \times d} \sum_{e=1}^m \sum_{i=1}^S \sum_{n=1}^{N}  \frac{\Vert x^{e}_i(t_n) - \hat \mu^{e}_i(t_n)\Vert_2^2}{\Vert x^{e}_i(t_n)\Vert_2^2},
\end{align}
\begin{align} \label{eq:uqmape}
    \text{MAPE} = \frac{100}{m \times S \times N \times d} \sum_{e=1}^m \sum_{i=1}^S \sum_{n=1}^{N} \sum_{d'=1}^{d} \left| \frac{x^{e}_{i,d'}(t_n) - \hat \mu^{e}_{i,d'}(t_n)}{x^{e}_{i,d'}(t_n)} \right|,
\end{align}
\begin{align} \label{eq:conflevel}
    \text{CL} = \frac{100}{m \times S \times N \times d} \sum_{e=1}^m \sum_{i=1}^S \sum_{n=1}^{N} \sum_{d'=1}^{d} \mathbbm{1}_{x^e_{i,d'}(t_n) \in \text{CI}(e,i, n, d')},
\end{align}
with the mean and standard deviation across candidate trajectory predictions defined as
\begin{align} \label{eq:meanstd}
    \hat \mu^e_i(t_n)= \frac{1}{p} \sum_{j=1}^p \hat x^{e,j}_i(t_n), \quad \text{and } \, \hat \sigma^e_i(t_n)= \sqrt{\frac{1}{N} \sum_{j=1}^p ( \hat x^{e,j}_i(t_n) - \hat \mu^e_i(t_n) )^2},
\end{align}
and the pointwise empirical confidence interval defined as 
\begin{align}
    \text{CI}(e,i,\cdot, \cdot) = \left[ \hat \mu^e_i-3\hat \sigma^e_i, \hat \mu^e_i+3\hat \sigma^e_i \right],
\end{align}
where $S$ indicates the number of trajectories used per environment, $N$ the length of each trajectory, and $d$ the dimensionality of the problem. The number of InD environments is denoted by $m$, to be replaced with $b-a+1$ for OoD cases.\footnote{Consequently, the first summation symbol's bounds and its corresponding factor in the denominators of \cref{eq:uqmse,eq:uqmape,eq:conflevel} should be adjusted accordingly for OoD formulae.} We set $p=m$ for InD uncertainty metrics calculation, and we are guaranteed the existence of those same $m$ environments for OoD cases. However, if our model performs well across all adaptation environments it encounters --as is the case with \ttwo as observed in \cref{tab:sota_results}-- then \emph{all} training and adaptation environments can be used, resulting in $p=m+b-a+1$ instead of only $p=m$ environments at our disposal for \cref{eq:meanstd}. This produces more sample predictions, allowing for more reliable population statistics.  

The results in \cref{tab:uq_metrics} show low relative MSE on ODE problems, and remarkably low MAPE scores on all problems. These indicate that the empirical mean of the predictions is indeed close to the ground truth. We also notice that the confidence levels vary from very low values on NS to very high on BT. Based on \cref{eq:conflevel}, we hypothesize that this is primarily due to standard deviations across predictions.\footnote{We remark that the empirical confidence level CL is a metric that favors models that are uncertain in giving the right mean prediction.}

To test our hypothesis, we plot the standard deviations as they evolve with time in \cref{fig:std_devs}. When compared to per-problem InD and OoD CLs from \cref{tab:uq_metrics}, \cref{fig:std_devs} reveals that higher confidence levels align with higher standard deviations, which is particularly noticeable on forecasts beyond training time horizons. Additionally, in \cref{fig:std_devs_vs_errors}, we plot the pointwise standard deviations as they relate to the corresponding absolute errors. Naturally, we observe a non-negligible correlation between the two, especially on the GO problem. Focusing on OoD behavior, our model successfully avoids undesirable regions of low uncertainty but higher-than-InD errors (along the top left corners). Instead, some OoD predictions for ODE problems fall in a region of higher-than-InD uncertainty, but still low error (along the bottom right corners), which stresses the well-suitedness of our approach for OoD generalization. Despite not knowing the underlying source of uncertainty we wish to model, these results suggest that our framework is capable of providing meaningful uncertainty estimates. This said, we emphasize that the calculation and interpretation of aforementioned uncertainty metrics (e.g., the width of CI) should be grounded on knowledge and goals of the problem at hand.
}

\begin{table*}[h]
\caption{\rebut{Uncertainty \rebut{estimation} metrics with \ttwo, all expressed in percentage points (\%). The star $*$ indicates cases where the close to zero denominators had to be filtered out to retain state values greater than $10^{-3}$. The Relative MSE is the most sensitive to these instabilities due to squaring at the denominator.} }
\vspace*{-0.3cm}
\label{tab:uq_metrics}
\begin{center}
\footnotesize
\begin{sc}
\begin{tabularx}{0.71\textwidth}{lcccccc}
\toprule
& \multicolumn{3}{c}{\textbf{LV}} & \multicolumn{3}{c}{\textbf{GO}} \\
\cmidrule(lr){2-4} \cmidrule(lr){5-7}
 & Rel. MSE  & MAPE  & CL &  Rel. MSE  & MAPE  & CL   \\
\midrule
InD   &  0.032  & 0.80 & 56.22 & 5.119 & 10.36 & 82.06 \\
OoD   & 0.183 & 1.86 &  78.24 & 1.653 & 7.17  & 70.47 \\
\bottomrule
\end{tabularx}
\begin{tabularx}{0.71\textwidth}{lcccccc}
& \multicolumn{3}{c}{\textbf{SM}} & \multicolumn{3}{c}{\textbf{BT}} \\
\cmidrule(lr){2-4} \cmidrule(lr){5-7}
 & Rel. MSE  & MAPE  & CL &  Rel. MSE  & MAPE  & CL   \\
\midrule
InD   &  2.005*  & 4.07* & 65.64 & 42.199 & 18.95 & 92.25 \\
OoD   & 0.158* & 1.99* &  94.70 & 46.028 & 22.28  & 90.63 \\
\bottomrule
\end{tabularx}
\begin{tabularx}{0.71\textwidth}{lcccccc}
& \multicolumn{3}{c}{\textbf{GS}} & \multicolumn{3}{c}{\textbf{NS}} \\
\cmidrule(lr){2-4} \cmidrule(lr){5-7}
 & Rel. MSE  & MAPE  & CL &  Rel. MSE  & MAPE  & CL   \\
\midrule
InD   &  2118.51*  & 58.33* & 14.58 & 166.696* & 17.30* & 11.19 \\
OoD   & 281.986* & 33.43* &  11.89 & 152.78* & 16.72*  & 11.09 \\
\bottomrule
\end{tabularx}
\end{sc}
\end{center}
\end{table*}

\begin{figure}[h]
\begin{center}
\centerline{\includegraphics[width=\columnwidth]{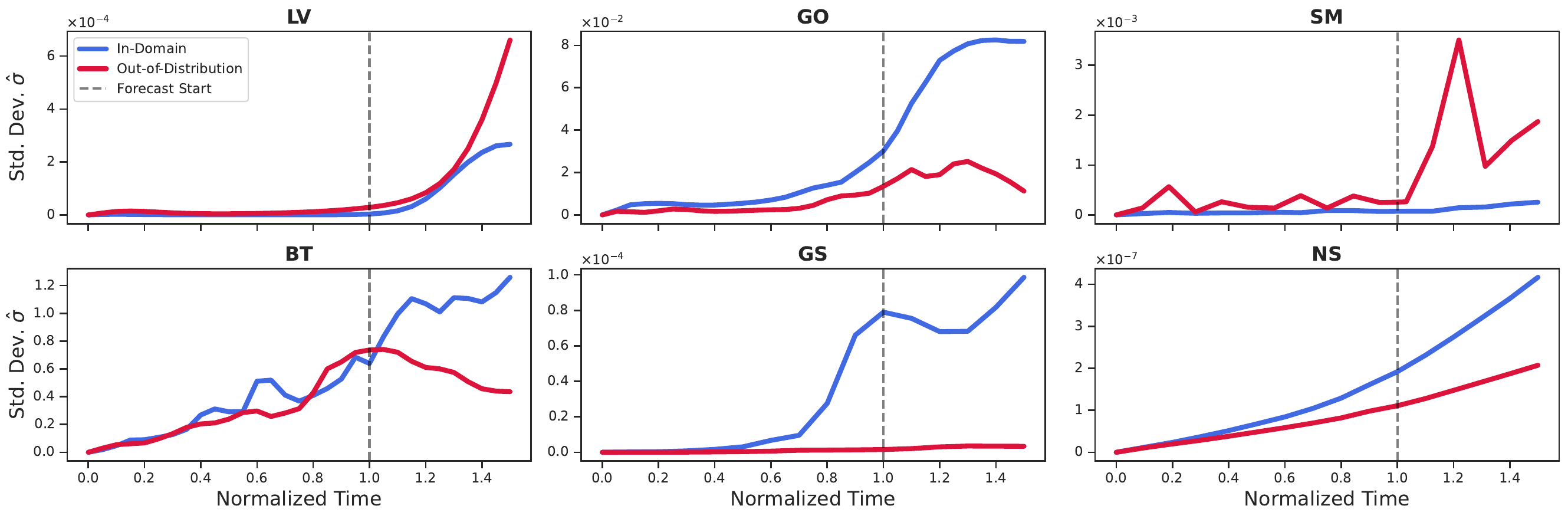}}
\caption{\rebut{Average standard deviations indicating how uncertainty grows with time, including when the model forecasts in time domains not seen during training. Higher standard deviations correlate with higher confidence level metrics observed in \cref{tab:uq_metrics}.}}
\label{fig:std_devs}
\end{center}
\end{figure}

\begin{figure}[h]
\begin{center}
\centerline{\includegraphics[width=\columnwidth]{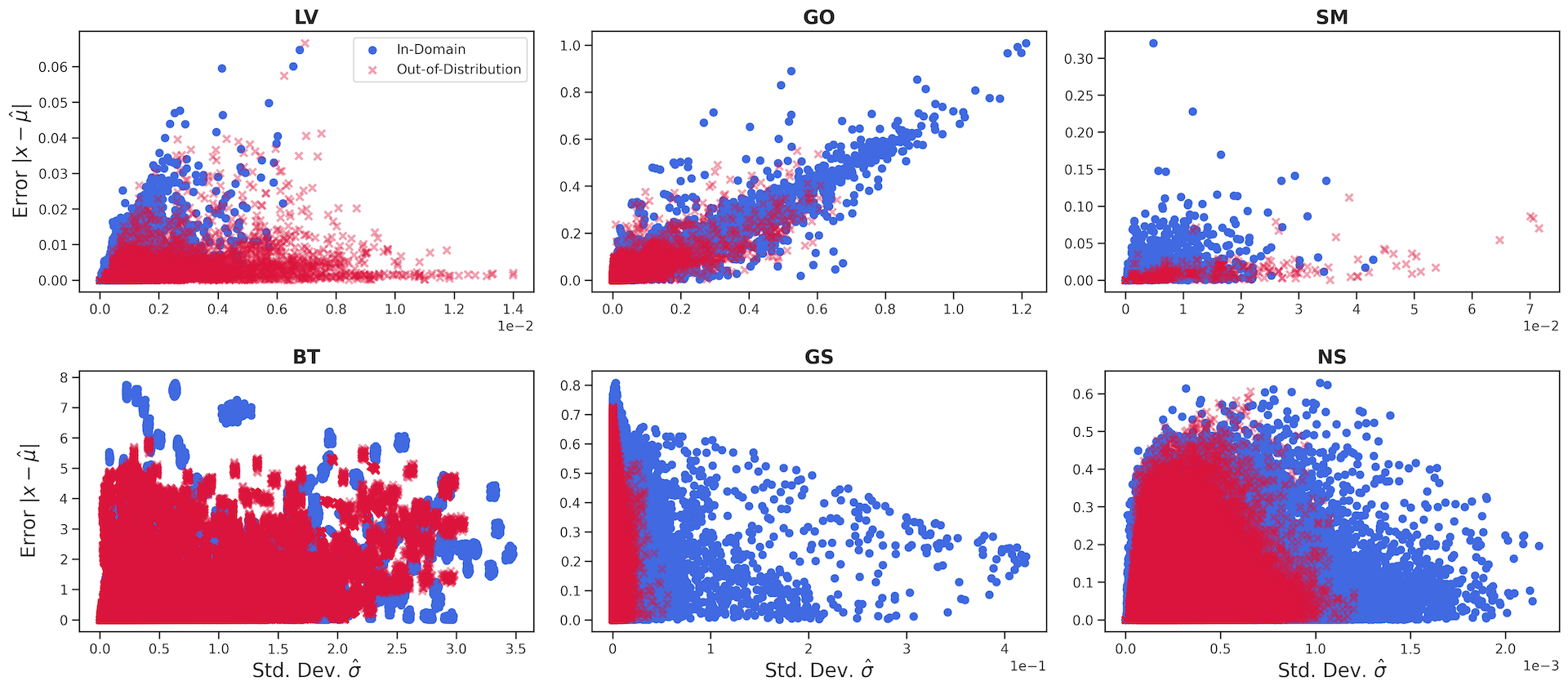}}
\caption{\rebut{Pointwise absolute errors against standard deviations for all problems.}}
\label{fig:std_devs_vs_errors}
\end{center}
\end{figure}

\subsection{Simple Pendulum (SP)}
\label{subsec:sp}

The autonomous dynamical system at play here corresponds to a frictionless pendulum suspended from a stationary point by a string of fixed length $L$. The state space $x = (\alpha, \omega)$, comprises the angle the pendulum makes with the vertical, and its angular velocity, respectively
\begin{equation}
\left\{
\begin{aligned}
\frac{{\md \alpha}}{{\md t}} &= \omega, \\
\frac{{\md \omega}}{{\md t}} &= -\frac{\textcolor{red}{g}}{L} \sin(\alpha).
\end{aligned}
\right.
\label{eq:simple_pendulum}
\end{equation}
For this problem, each environment corresponds to a different gravity.\footnote{To give an intuition behind the term ``environment'', one might consider the surface of a celestial body in the solar system (see \cref{introduction}).} With $L=1$ set, the goal is to learn a dynamical system that easily generalizes across the unobserved $g$. Trajectories are generated with a Runge-Kutta 4th order solver, and a fixed step size of $\Delta t = 0.25$. Only the \tone variant is used in all experiments involving this problem.

{\color{rebutcolor}

\paragraph{OFA vs. OPE vs. NCF.} The training and adaptation MSE metrics were reported in \cref{tab:pendulum} during our favorable comparison to two baselines: One-For-All (OFA) -- one context-free model trained for all environments; and One-Per-Env (OPE) -- one model trained from scratch for each environment we encounter.\footnote{The OFA paradigm offers no mechanism to adapt to new environments, unless we fine-tune the vector field's weights, thus returning to an OPE-like setting.} Here, we expand on said comparison with our loss values during training reported in \cref{fig:losses_ofavsncf}. Additionally, we probe the training time of the different methods. Given that OFA and OPE do not require contextual information, we increase the capacity of their main networks to match the NCF total parameter count. Each method is then trained until the loss stagnates, and we report the amortized training times in \cref{tab:pendulum_times}.


While exhibiting much larger adaptation times, OPE overfits to its few-shot trajectories as evidenced in \cref{fig:losses_ofavsncf}. The same figure shows that the OFA training loss quickly stagnates to a relatively high value during training. Unsurprisingly, learning one context-agnostic vector field for all environments (OFA) is suboptimal given the vast differences in gravity from one environment to the next. These observations align with the more complete study on OFA and OPE by \citet{yin2021leads}. Leveraging Meta-Learning, our method, trained on all environments at once like OFA, effectively learns to discriminate between them and produces low validation MSE metrics and accurate trajectories, one of which is presented in \cref{fig:trajscompare}b. Taken together, \cref{tab:pendulum,tab:pendulum_times,fig:losses_ofavsncf} show that Meta-Learning delivers on training time, adaptation time, and most importantly, testing accuracy.

\begin{table}[h]
\caption{\rebut{Meta-training and adaptation times ($\downarrow$) with OFA, OPE, and \tone on the SP problem. We report the amortized times (in \textbf{minutes}) corresponding to fitting one single environment (of which we count 25 when meta-training, and 2 for adaptation).}}
\label{tab:pendulum_times}
\begin{center}
\begin{small}
\begin{sc}
\vspace*{-0.4cm}
\begin{tabular}{lccc}
\toprule
& \#Params & Train & Adapt   \\
\midrule
OFA    & 49776 & 0.34 & 0 \\
OPE    & 49776 & 4.63 & 5.72 \\
NCF  & 50000 & 2.96 & 0.51 \\
\bottomrule
\end{tabular}
\end{sc}
\end{small}
\end{center}
\end{table}

\begin{figure}[h]
\begin{center}
\centerline{\includegraphics[width=0.6\columnwidth]{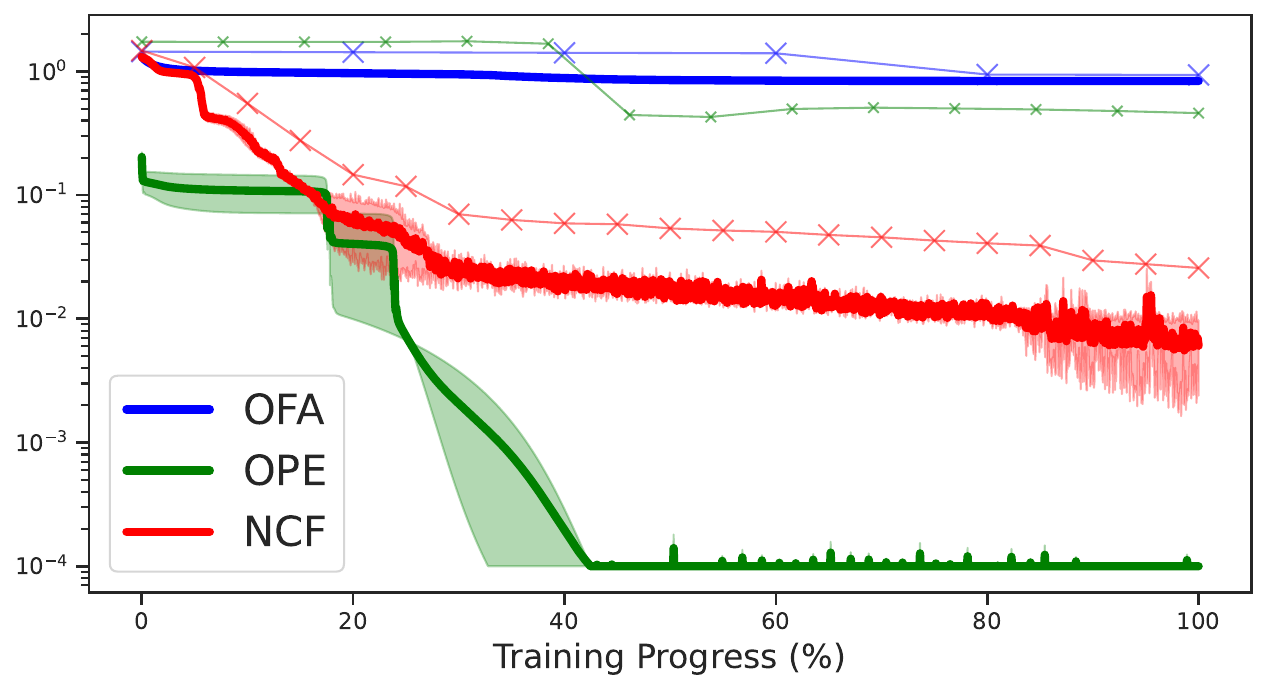}}
\caption{\rebut{MSE loss values when training \tone on the SP problem, compared with the baseline OFA and OPE formulations. The crosses $\times$ indicate mean validation curves across 3 runs, color-coded to match the training curves. OFA fails catastrophically since the diversity of environments in the training dataset prevents the approximation of any meaningful vector field, while OPE overfits to its 4 training trajectories.}}
\label{fig:losses_ofavsncf}
\end{center}
\end{figure}

\paragraph{Sample efficiency with the number of trajectories.} We compare our model to the two OFA and OPE baselines as the number of trajectories $S$ in each training environment is increased from 1 to 12. The results reported in \cref{fig:sampleeffsp} indicate that NCF is indeed the best option when data is limited, while the improvements in OFA MSEs are barely noticeable. As $S$ increases, we observe that OPE is able to overcome overfitting to ultimately achieve the best results. These results demonstrate that NCF efficiently uses its few-shot trajectories. However, if neither data nor training time is a concern (cf. \cref{tab:pendulum_times}), then the traditional One-Per-Env should be prioritized. 

\begin{figure}[h]
\begin{center}
\centerline{\includegraphics[width=0.5\columnwidth]{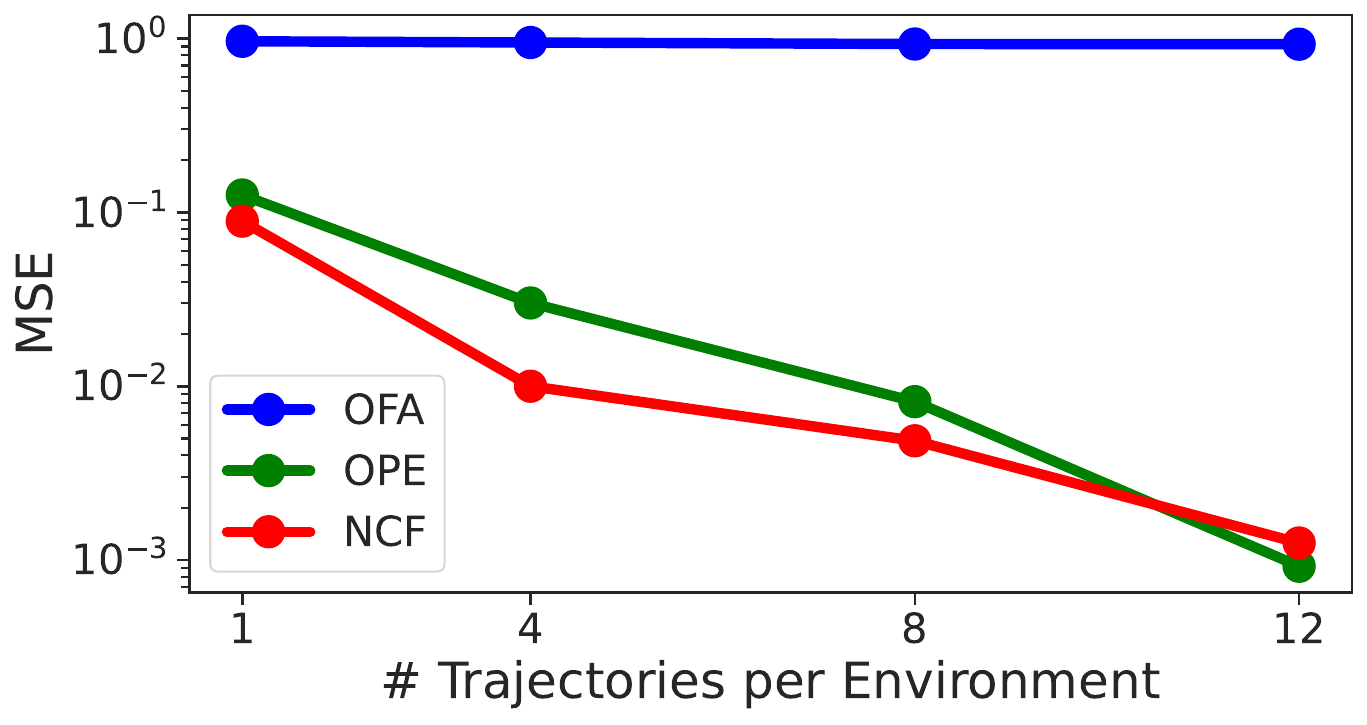}}
\caption{\rebut{In-Domain MSEs on the SP problem comparing the sample efficiency of \tone against OPE and OFA. NCF is effective in low-data regimes, and OPE overcomes the gap as data increases.}}
\label{fig:sampleeffsp}
\end{center}
\end{figure}

}

{\color{rebutcolor}
\paragraph{Scaling with the training environments.} The computational speed of any method based on Neural ODEs depends on the numerical integrator it uses. To provide a consistent number of function evaluations (NFEs), we switch the adaptive time-stepper \texttt{Dopri5} for the fixed time-stepper \texttt{RK4}, then we measure the duration of epochs as the training progresses. These times as used to produce \cref{fig:trainigtimes}, which indicates training times per epoch (in seconds) as the number of training environments is increased while keeping the range of gravities unchanged between 2 and 24. We observe excellent scaling, with the training time only increasing by roughly 23\% (from 0.38 to 0.47 seconds) when the number of environments is scaled by 10 (from 5 to 50).

\begin{figure}[h]
\begin{center}
\centerline{\includegraphics[width=0.5\columnwidth]{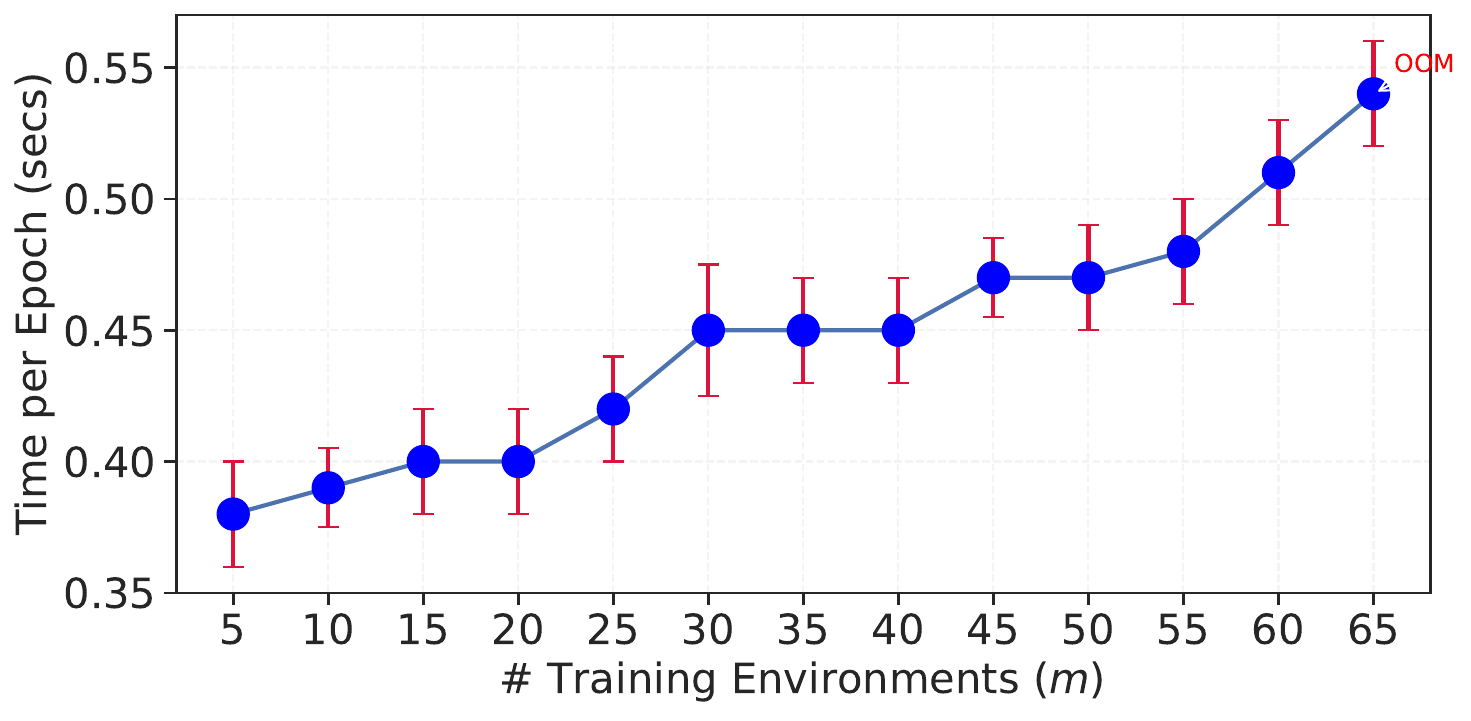}}
\caption{\rebut{Training time as the number of training environments $m$ is increased from 5 to 65. The vertical bars are proportional to the standard deviation across epochs. OOM indicates that our workstation ran out of memory.}}
\label{fig:trainigtimes}
\end{center}
\end{figure}
}

\paragraph{Probing the context vectors.} Beyond serving as a control signal for the vector fields, the contexts encode useful representations. In \cref{fig:repr}, we visualize the first two dimensions of the various $\{\xi^e\}_{1\leq e \leq 25}$ after training. We observe that environments close in indices\footnote{The indices of the training environments correspond to their ordering in increasing values of gravity.} are equally close in the two context dimensions. Similarly, distant environments are noticeably far apart in this view of the context space. The same observation is made during adaptation, where, for instance, $e'=a_2$ (corresponding to $g=14.75$) gets a context close to $e=15$ (corresponding to $g=14.83$). This observation indicates that the latent context vector is encoding features related to gravity, which may be used for further downstream representation learning tasks.

\begin{figure}[h]
\begin{center}
\centerline{\includegraphics[width=0.6\columnwidth]{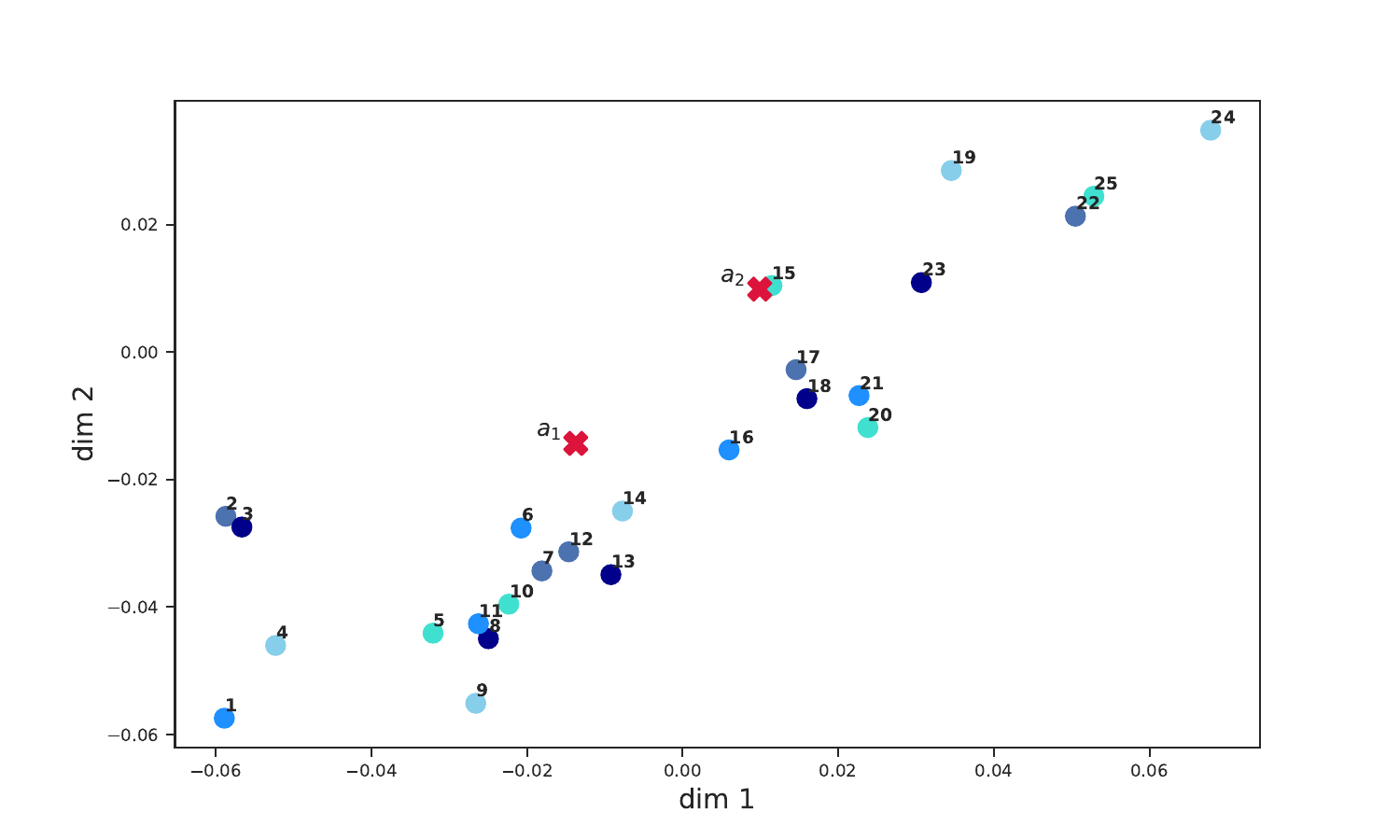}}
\caption{Representation of the first and second dimensions of the learned contexts for the SP problem. The labels 1 to 25 identify the training environments (in shades of \textcolor{blue}{blue}), while $a_1$ and $a_2$ (in \textcolor{red}{red}), indicate the adaptation environments. We observe that environments close in indices (and thus in gravity) share similar contexts along these dimensions.}
\label{fig:repr}
\end{center}
\end{figure}

\subsection{Lotka-Volterra (LV)}
\label{subsec:lv}
With dynamics that are of continued interest to many fields (epidemiology \cite{venturino1994influence}, economy \cite{wu2012grey}, etc.), the Lotka-Volterra (LV) ODE models the evolution of the concentration of prey $x$ and predators $y$ in a closed ecosystem. The behavior of the system is controlled by the prey's natural growth rate $\alpha$, the predation rate $\beta$, the predator's increase rate upon consuming prey $\delta$, and the predator's natural death rate $\gamma$ 
\begin{align}
\begin{dcases}
\frac{{\md x}}{{\md t}} &= \alpha x - \textcolor{red}{\beta} xy, \\
\frac{{\md y}}{{\md t}} &= \textcolor{red}{\delta} xy - \gamma y.
\end{dcases}
\label{eq:lotka_volterra}
\end{align}
We repeat the experiment as designed in \cite{kirchmeyer2022generalizing}. All synthetic ground truth data is generated with the two initial states both following the $\mathcal{U}(1,3)$ distribution. Once sampled, we note that the same initial condition is used to generate trajectories for all environments. The parameters that vary across training environments are $\beta \in \{0.5, 0.75, 1\}$ and $\delta \in \{0.5, 0.75, 1\}$. In each training environment, we generate 4 trajectories with a Runge-Kutta time-adaptive 4th-order scheme, while we generate 32 for In-Domain evaluation. For one-shot adaptation, we extrapolate to $\beta \in \{0.625, 1.125\}$ and $\delta \in \{0.625, 1.125\}$, with only 1 trajectory per environment, and 32 for OoD testing. The observed parameters $\alpha$ and $\gamma$ are always fixed at $0.5$.



\begin{figure}[H]
\centering
\subfigure[MAPE]{\includegraphics[width=0.4\textwidth]{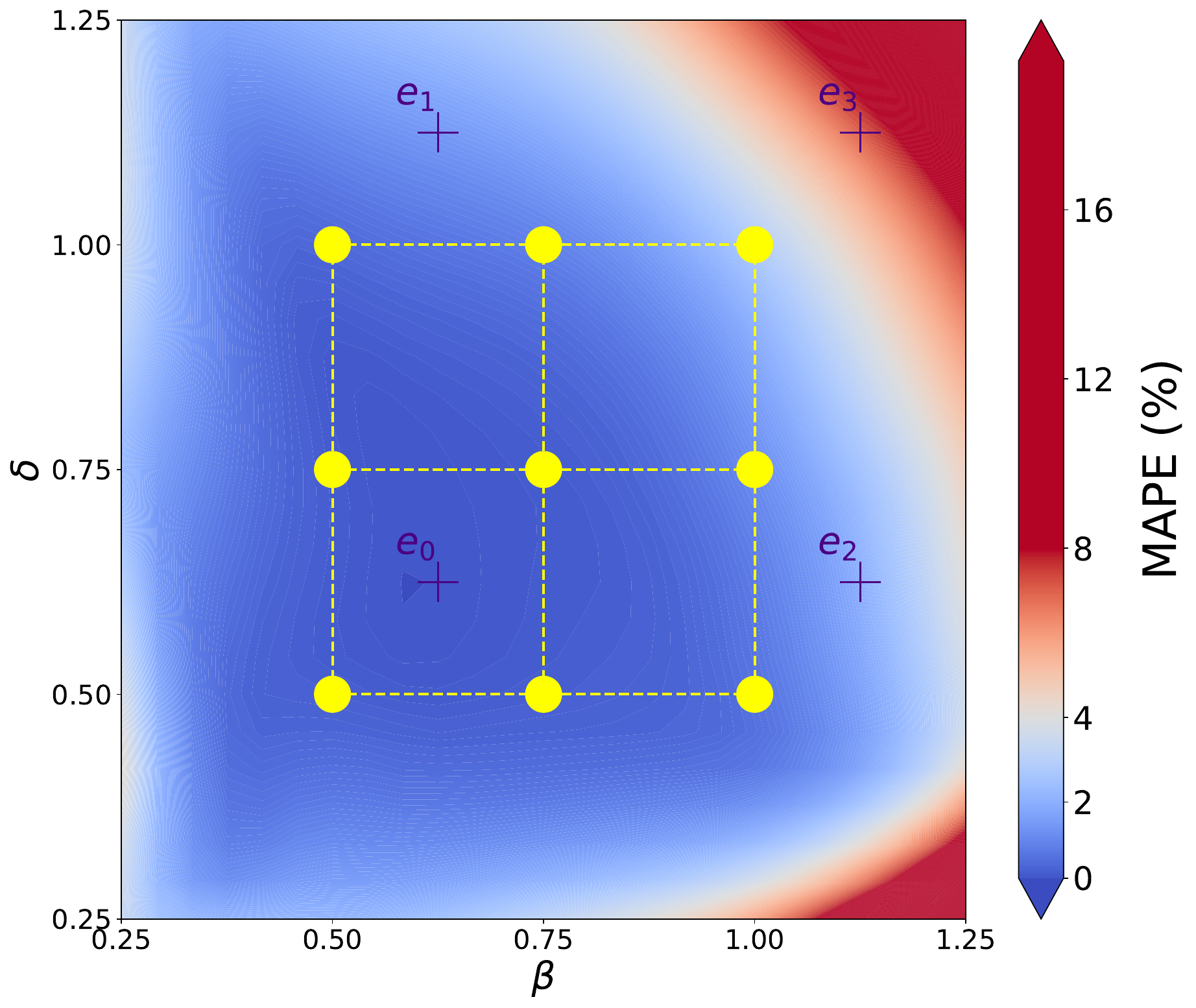}\label{fig:huge_adapt_1}}
\subfigure[MSE]{\includegraphics[width=0.417\textwidth]{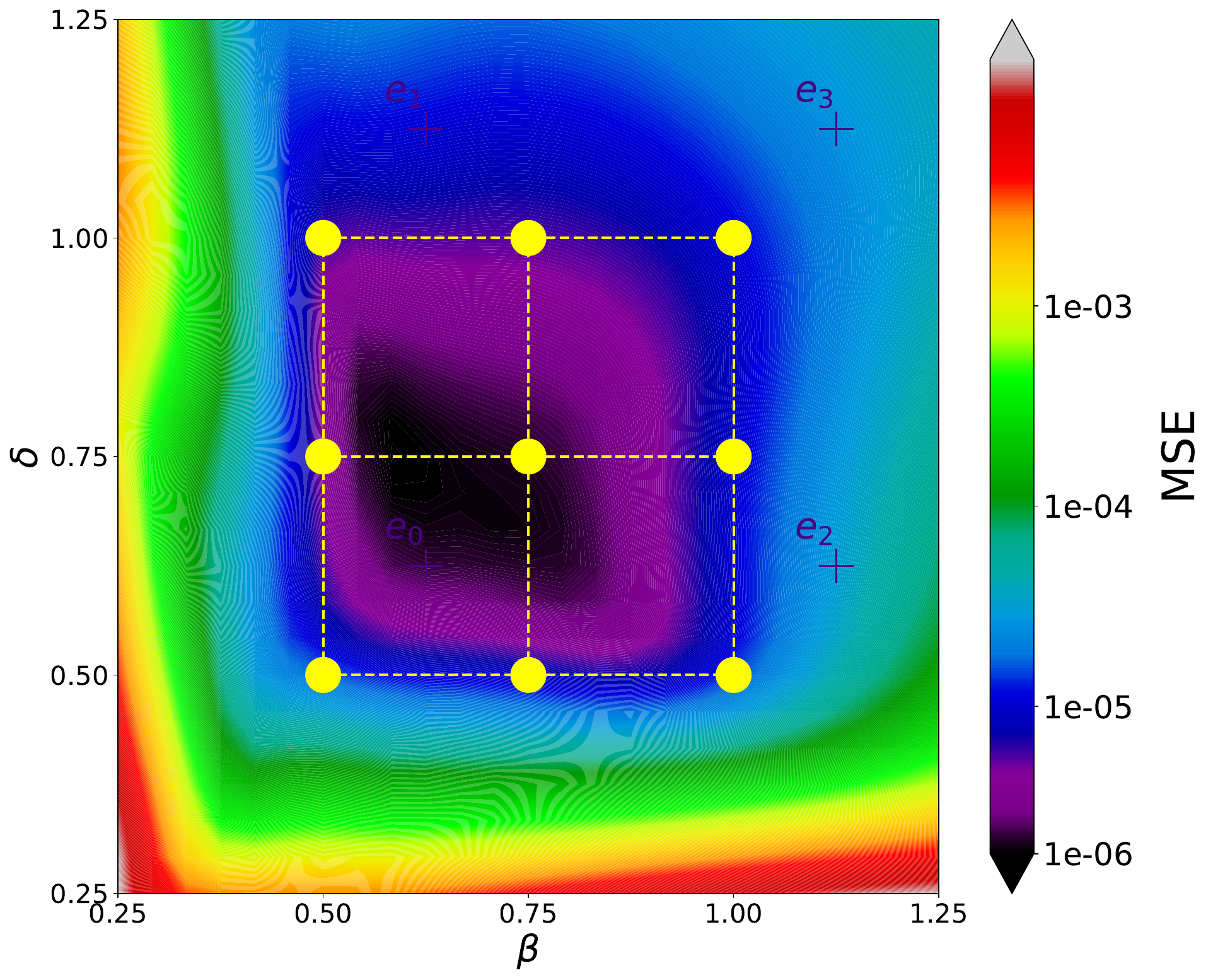}\label{fig:huge_adapt_2}}
\caption{Results for the Large adaptation of the LV problem to a grid, showcasing the MAPE and training MSE. Compared to \cref{fig:huge_adapt}, the colorbar range in (a) is shrunk to focus on the MAPEs between 0 and 8\%.}
\label{fig:huge_adapt_full}
\end{figure}

The large grid-wise adaptation experiment conducted on the LV problem in \cref{subsec:extrapolation} showed that NCFs were powerful at extrapolating. Here, through the MSE training losses in \cref{fig:huge_adapt_2}, we highlight the fact that the bottom and left edges of the grid display quite high errors. This illustrates the importance of observing several loss metrics when learning time series \cite{hewamalage2023forecast}.

\subsection{Glycolytic-Oscillator (GO)}

This equation defines the evolution of the concentration of 7 biochemical species $\{s_i\}_{i=1,2,\ldots,7}$ according to the ODE:
\begin{align*}
\frac{ds_1}{dt} &= J_0 - \textcolor{red}{k_1} \frac{s_1 s_6}{1 + (s_6/\textcolor{red}{K_1})^q} \\
\frac{ds_2}{dt} &= 2\textcolor{red}{k_1} \frac{s_1 s_6}{1 + (s_6/\textcolor{red}{K_1})^q} - k_2 s_2 (N - s_5) - k_6 s_2 s_5 \\
\frac{ds_3}{dt} &= k_2 s_2 (N - s_5) - k_3 s_3 (A - s_6) \\
\frac{ds_4}{dt} &= k_3 s_3 (A - s_6) - k_4 s_4 s_5 - \kappa (s_4 - s_7) \\
\frac{ds_5}{dt} &= k_2 s_2 (N - s_5) - k_4 s_4 s_5 - k_6 s_2 s_5 \\
\frac{ds_6}{dt} &= -2\textcolor{red}{k_1} \frac{s_1 s_6}{1 + (s_6/\textcolor{red}{K_1})^q} + 2k_3 s_3 (A - s_6) - k_5 s_6 \\
\frac{ds_7}{dt} &= \psi \kappa (s_4 - s_7) - k s_7
\end{align*}
where the parameters either vary or are set fixed as per \cref{tab:parameters}.

\begin{table}[h]
    \caption{Physical parameters and their values for the GO problem}
    \small
    \centering
    \begin{tabular}{l|cccccccccccccc}
        \toprule
        \tabhead{Parameter} & $J_0$ & $k_1$ & $k_2$ & $k_3$ & $k_4$ & $k_5$ & $k_6$ & $K_1$ & $q$ & $N$ & $A$ & $\kappa$ & $\psi$ & $k$ \\
        \midrule
        \tabhead{Value} & 2.5 & (varies) & 6 & 16 & 100 & 1.28 & 12 & (varies) & 4 & 1 & 4 & 13 & 0.1 & 1.8 \\
        \bottomrule
    \end{tabular}
    \label{tab:parameters}
\end{table}

Again, we follow the same procedure as in \cite{kirchmeyer2022generalizing} to generate trajectories for each environment. Namely, we sample initial conditions from a specific distribution \cite[Table 2]{daniels2015efficient}. We vary $k_1 \in \{100,90,80\}$ and $K_1 \in \{ 1,0.75,0.5\}$ to create 9 training environments with 32 trajectories each, both for InD training and testing. We use 1 trajectory for OoD adaptation to 4 environments defined by  $k_1 \in \{85,95\}$ and $K_1 \in \{ 0.625,0.875\}$. We use 32 trajectories for OoD evaluation. 

Intuitive post-processing of the contexts offers many insights into the meta-training process, particularly because it clusters the 9 environments into groups of 3, as illustrated in \cref{fig:gs_clusters}.   
\begin{figure}[ht]
\begin{center}
\includegraphics[width=.3\columnwidth]{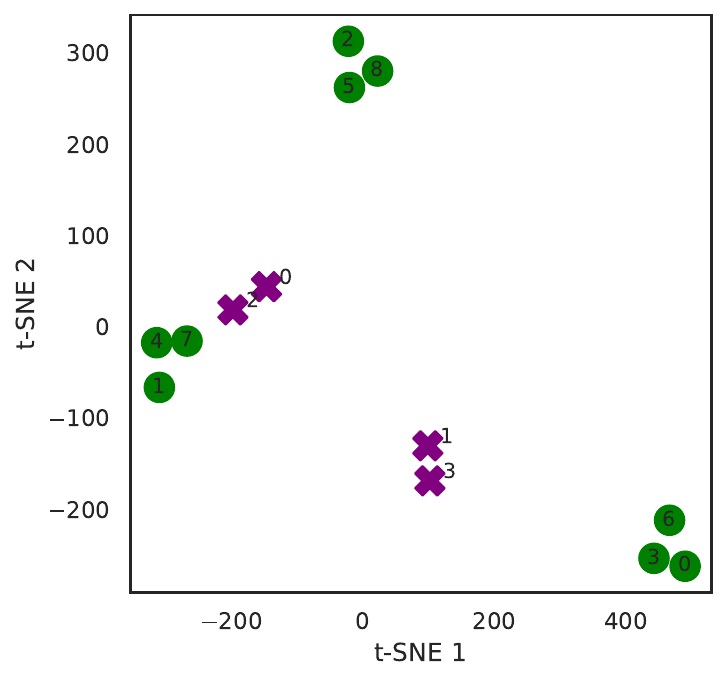}
\caption{Illustration of 256-dimensional context vectors consistently clustered with t-SNE embeddings with perplexity of 2 on the Glycolytic Oscillator (GO). The training environments are in green and labeled as their IDs, while the adaptation environments are in purple. This clustering mirrors the distribution of the system parameters.}
\label{fig:gs_clusters}
\end{center}
\vskip -0.2in
\end{figure}

\subsection{Sel'kov Model (SM)}

We introduce the Sel'kov model in dimensionless form \cite{strogatz2018nonlinear}, a highly non-linear ODE mainly studied for its application to yeast glycolysis : 
\begin{align*}
\frac{dx}{dt} &= -x + ay + x^2y \\
\frac{dy}{dt} &= \textcolor{red}{b} - ay - x^2y
\end{align*}
where $a=0.1$ and $b$ is the parameter that changes. The trajectories are generated with a time horizon of 40, with 11 regularly spaced time steps. We sample each initial condition state from the distribution $\mathcal{U}\{0,3\}$, and we observe the appearance of a limit cycle (L1), then an equilibrium point (E), and another limit cycle as $b$ changes (see \cref{fig:attractors}). 

Specifically, our 21 training environments are a union of 7 environments evenly distributed in $b \in [-1,-0.25 ] $ , then 7 evenly distributed in $b \in [-0.1,-0.1]$, and finally 7 others with $b \in [0.25,1]$. We generate 4 trajectories for training and 4 for InD testing. As for adaptation, we choose 6 environments, with $b \in \{ -1.25, -0.65, -0.05, 0.02, 0.6, 1.2 \} $ (see \cref{fig:selkov_train_adapt_vis}). We set aside 1 trajectory for adaptation and 4 for OoD testing.

In \cref{experiments}, we showed that \ttwo outperformed other adaptation rules. We note, however, that training such systems is very difficult, even with NCFs. We observed in practice that convergence is extremely sensitive to weight initialization, as the shaded regions on the loss curves of \cref{fig:selkov_train_losses} attest. The curves are complemented with \cref{fig:selkov_train_adapt_mse}, emphasizing that the model performs best in environments near the equilibrium.

\begin{figure}[ht]
\begin{center}
\includegraphics[width=.99\columnwidth]{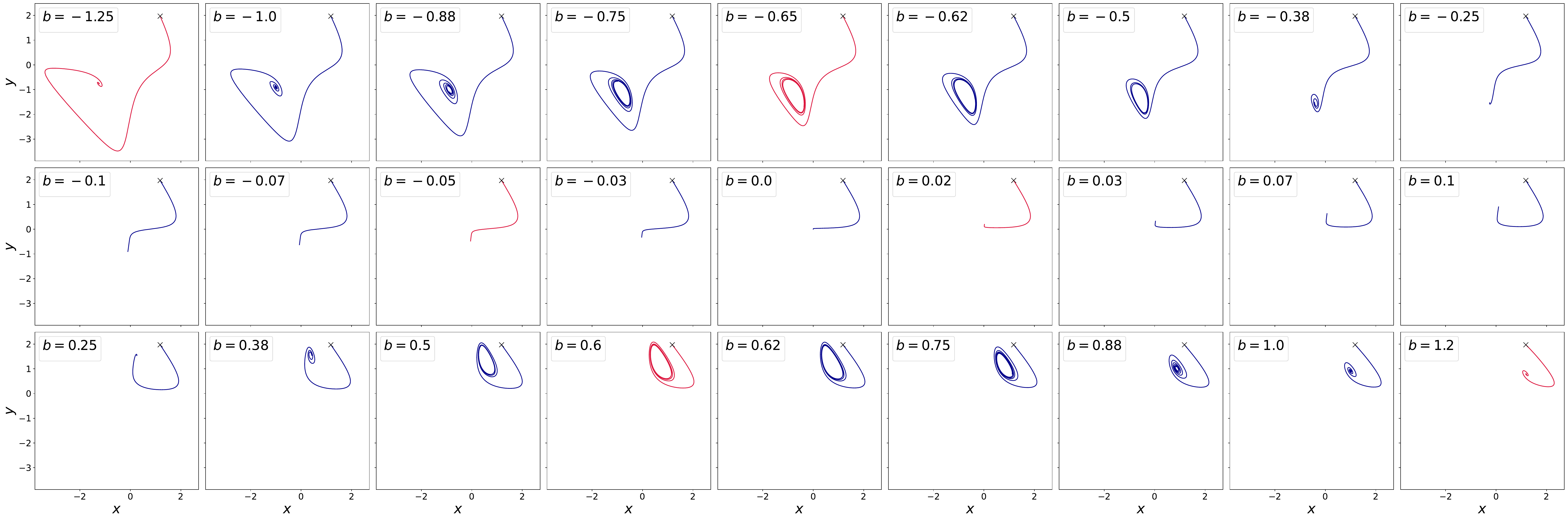}
\caption{Visualization of a trajectory with the same initial condition in various environments of the Sel'kov Model's dataset. In blue are the meta-training environments, in red the meta-testing ones. \rebut{We observe three attractors: the limit cycle (L1) along the top row, the fixed equilibrium (E) along the second row, and another limit cycle (L2) along the third row.}}
\label{fig:selkov_train_adapt_vis}
\end{center}
\vskip -0.2in
\end{figure}

\begin{figure}[ht]
\begin{center}
\includegraphics[width=.8\columnwidth]{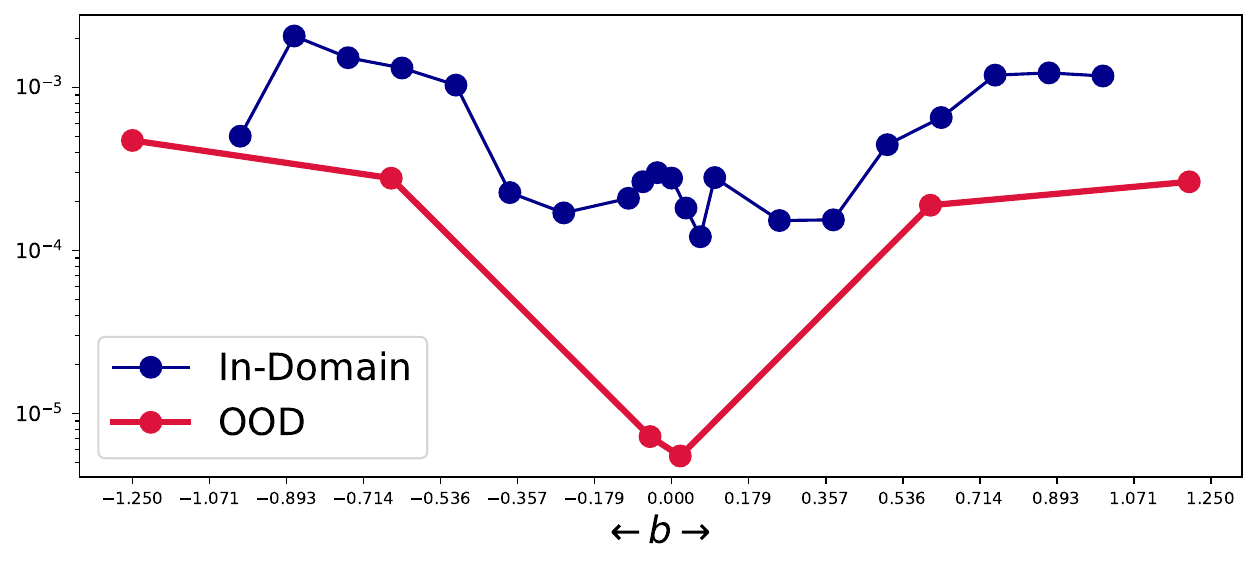}
\caption{Per-environment In-Domain and adaptation MSEs for the SM problem. In blue are the meta-training environments, in red the meta-testing ones. \rebut{Central environments in the attractor (E) are better resolved than the others.}}
\label{fig:selkov_train_adapt_mse}
\end{center}
\vskip -0.2in
\end{figure}

\subsection{Brusselator (BT)}

The Brusselator model \cite{prigogine1968symmetry} describes the reaction-diffusion dynamics of two chemical species, \(U\) and \(V\), and is given by the following system of partial differential equations (PDEs) defined on a $8 \times 8$ grid:

\begin{align*}
\frac{\partial U}{\partial t} &= D_u \Delta U + \textcolor{red}{A} - (\textcolor{red}{B} + 1)U + U^2V, \\
\frac{\partial V}{\partial t} &= D_v \Delta V + \textcolor{red}{B}U - U^2V,
\end{align*}

where:
\begin{itemize}
    \item \(U\) and \(V\) are the concentrations of the chemical reactants.
    \item \(D_u=1\) and \(D_v=0.1\) are the diffusion coefficients for \(U\) and \(V\), respectively.
    \item \(A\) and \(B\) are constants representing the rate parameters of the chemical reactions, both varying across environments.
\end{itemize}

Using an RK4 adaptive time-step solver, the system is simulated up to $T=10$ (excluded), with the step reported every time step $\Delta t =0.5$. The initial condition for each trajectory for all environments is sampled
 and broadcast as follows:
\begin{align*}
U_0 &= \bar A \\
V_0 &= \frac{\bar B}{\bar A} + 0.1 \eta_{ij}
\end{align*}
where \( \bar A \sim \mathcal{U}(0.5, 2.0) \), \( \bar B \sim \mathcal{U}(1.25, 5.0) \), and \( \eta_{ij} \sim \mathcal{N}(0, 1) \) for each grid position \((i, j)\).

We aim to keep all environments involved in this problem outside the Brusselator's oscillatory regime $B^2 > 1 + A^2$. For Meta-training, $A$ and $B$ are selected from $\{ 0.75, 1, 1.25 \} \times \{ 3.25, 3.5, 3.75 \}$ yielding 12 environments, each with 4 trajectories for training, and 32 for InD testing. For adaptation, we select 12 environments from $\{ 0.875, 1.125, 1.375 \} \times \{ 3.125, 3.375, 3.625, 3.875 \}$, with 1 trajectory for training, and 32 for OoD testing.

{\color{rebutcolor}

As observed with the metrics in \cref{tab:sota_results}, the BT dataset is one of the most challenging to learn on. Indeed, the non-Meta-Learning baselines OFA and OPE equally struggle with it. To show this, we vary the number of trajectories in each training environment between 1 and 8. Then, we plot in \cref{fig:sampleeffbt} the sample efficiency of our NCF approach against the non-Meta-Learning baselines. Like with the SP problem (cf. \cref{fig:sampleeffsp}), we observe good NCF performance when data is scarce, underlining the suitability of our approach for few-shot learning.

\begin{figure}[h]
\begin{center}
\centerline{\includegraphics[width=0.5\columnwidth]{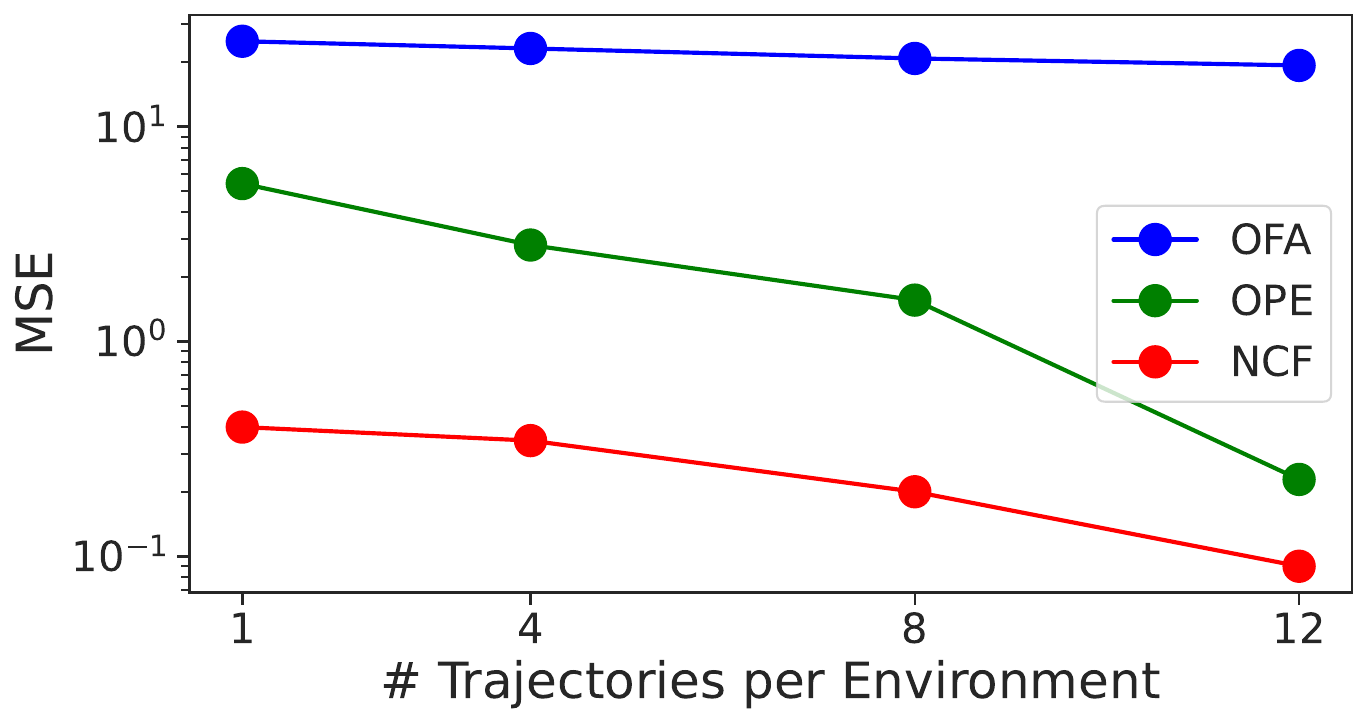}}
\caption{\rebut{In-Domain MSEs on the BT problem comparing the sample efficiency of \ttwo against OPE and OFA. NCF is effective in low-data regimes, and OPE closes the gap as data increases.}}
\label{fig:sampleeffbt}
\end{center}
\end{figure}

}

\subsection{Gray-Scott (GS)}

We aim to validate the empirical potential of our method on spatiotemporal systems beyond ODEs. We thus identify the Gray-Scott (GS) model for reaction-diffusion, a partial differential equation (PDE) defined over a spatial $32\times32$ grid:
\begin{align*}
\frac{\partial U}{\partial t} &= D_u\Delta U - UV^2 + \textcolor{red}{F}(1 - U)  \\
\frac{\partial V}{\partial t} &= D_v\Delta V + UV^2 - (\textcolor{red}{F} + \textcolor{red}{k})V
\end{align*}

Like in the Brusselator case above, $U$ and $V$ represent the concentrations of the two chemical components in the spatial domain with periodic boundary conditions. $D_u, D_v$ denote their diffusion coefficients respectively, while $F$ and $k$ are the reaction parameters. We generate trajectories on a temporal grid with $\Delta t = 40$ and temporal horizon $T = 400$.

The parameters we use to generate our environments are the reaction parameters $F$ and $k$. We consider 4 environments for meta-training: $F \in \{0.30, 0.39\}, k \in \{0.058, 0.062\}$; and 4 others for adaption: $F \in \{0.33, 0.36\}, k \in \{0.59, 0.61\}$. Other simulation parameters, as well as the initial condition-generating distribution, are inherited from \cite{kirchmeyer2022generalizing}, where we direct readers for additional details.



\subsection{Navier-Stokes (NS)}

Like the Gray-Scott case, the 2D incompressible Navier-Stokes case is inherited from \cite{kirchmeyer2022generalizing}. The PDE is defined on a $32 \times 32 $ spatial grid as:
\begin{align*}
\frac{\partial \omega}{\partial t}  &= - v \nabla \omega + \textcolor{red}{\nu} \Delta \omega + f \\
\nabla v &= 0
\end{align*}

where:
\begin{itemize}
    \item \(\omega =  \nabla \times v \) is the vorticity,
    \item $v$ is the velocity field,
    \item \(\nu\) is the kinematic viscosity.
\end{itemize}

The trajectory data from $t=0$ to $T=10$ with $\Delta t= 1$ is obtained through a custom Euler integration scheme. By varying the viscosity from $\nu \in \{ 8\cdot 10^{-4}, 9\cdot 10^{-4}, 1.0\cdot 10^{-3}, 1.1\cdot 10^{-3}, 1.2\cdot 10^{-3} \}$ we gather 5 meta-training environments, each with 16 trajectories for training and 32 for testing. Similarly, we collect 4 adaptation environments with $\nu \in \{ 8.5\cdot 10^{-4}, 9.5\cdot 10^{-4}, 1.05\cdot 10^{-3}, 1.15\cdot 10^{-3}$ each with 1 trajectory for training, and 32 for testing. Other parameters of the simulation,  as well as the initial condition generating distribution, are inherited from \cite{kirchmeyer2022generalizing}, where we encourage the readers to find more details.

\newpage
\section{Ablation Studies}
\label{app:ablation}

The ablation studies described in this section are designed to investigate how the context pool size, the context size, the pool-filling strategy, and the 3-networks architecture affect the performance of NCFs.

\subsection{Limiting the context pool size}
\label{app:contextpoolsize}

The context pool $\mathrm{P}$ from which environments $j$ are randomly sampled contributes significantly to the computational and memory complexity of the algorithm. As evidenced in \cref{alg:ncf,alg:ncf_proximal}, the computation of the loss and hence its gradient can be parallelized across the context vectors $\xi^j$. With the goal of assessing the associated computational burden, we vary the size of the pool size for the LV problem from 1 to 9, reporting the In-Domain and OoD metrics, and computational time in \cref{fig:contextpoolsize}.

\begin{figure}[H]
\centering
\subfigure[MSE Losses]{\includegraphics[width=0.395\textwidth]{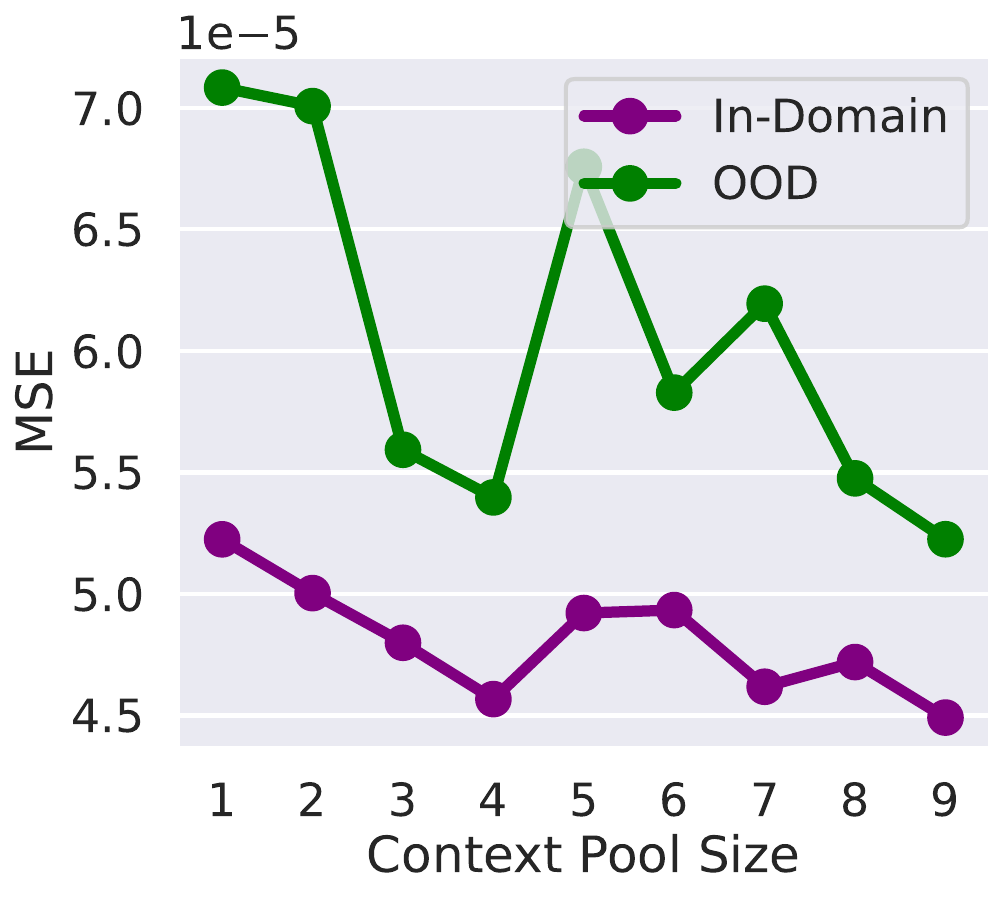}\label{fig:poolsizemsess}}
\subfigure[Training Times]{\includegraphics[width=0.4\textwidth]{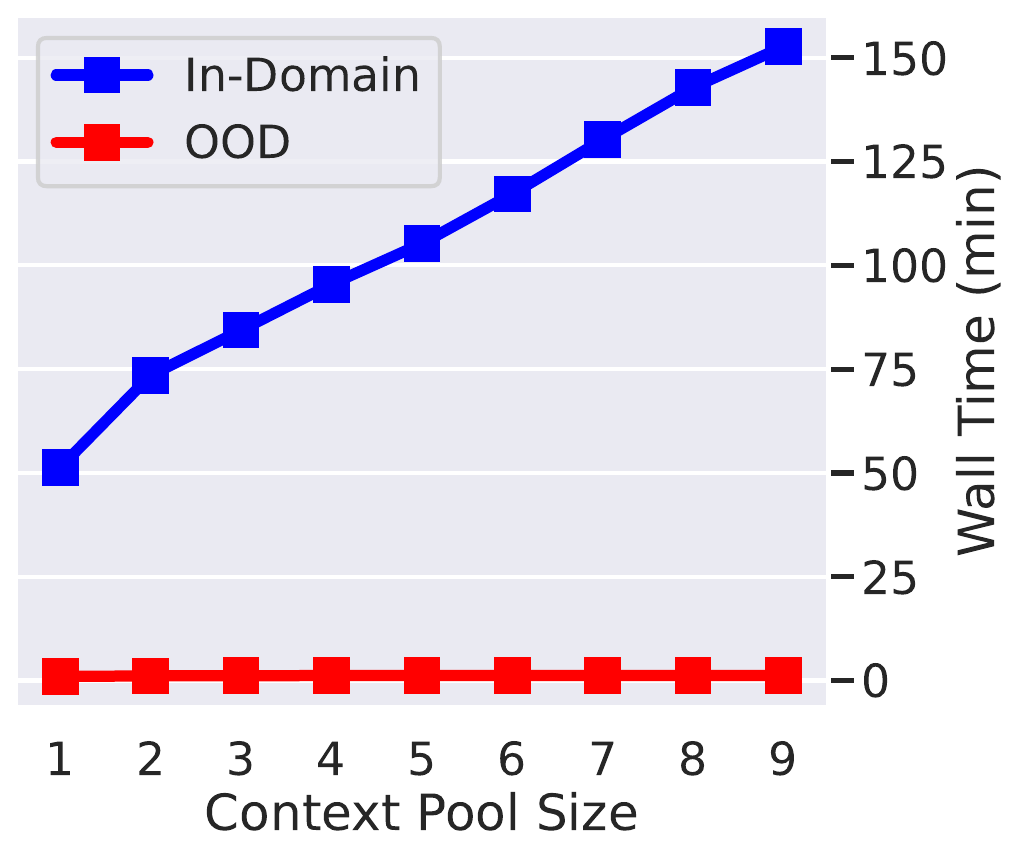}\label{fig:poolsizetimes}}
\caption{(a) In-domain and OoD MSEs, and (b) training and adaptation times when varying the context pool size on the LV problem with \ttwo. }
\label{fig:contextpoolsize}
\end{figure}

\begin{figure}[H]
\centering
\subfigure[In-Domain]{\includegraphics[width=0.48\textwidth]{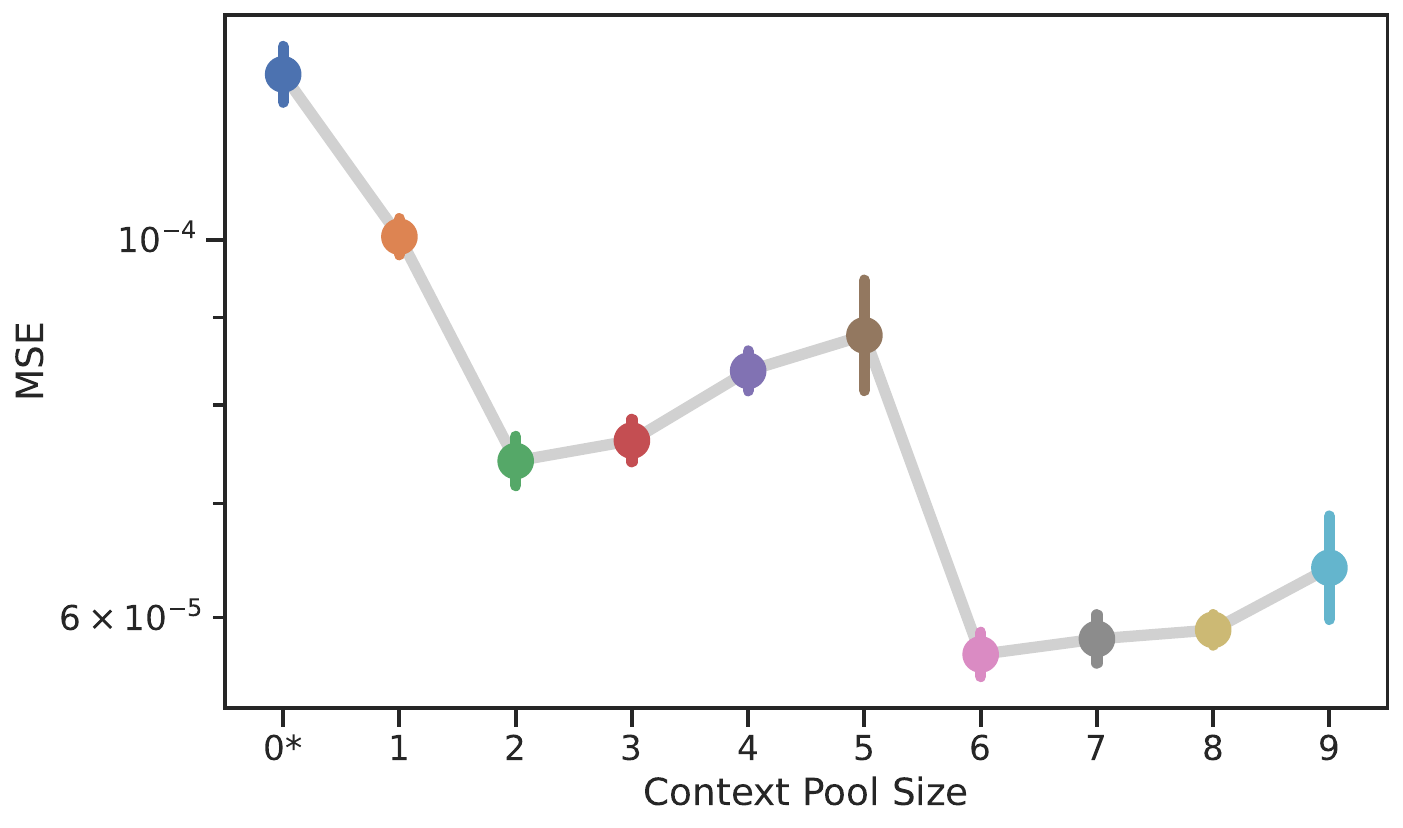}\label{fig:poolsizemsess_T1}}
\subfigure[OoD]{\includegraphics[width=0.46\textwidth]{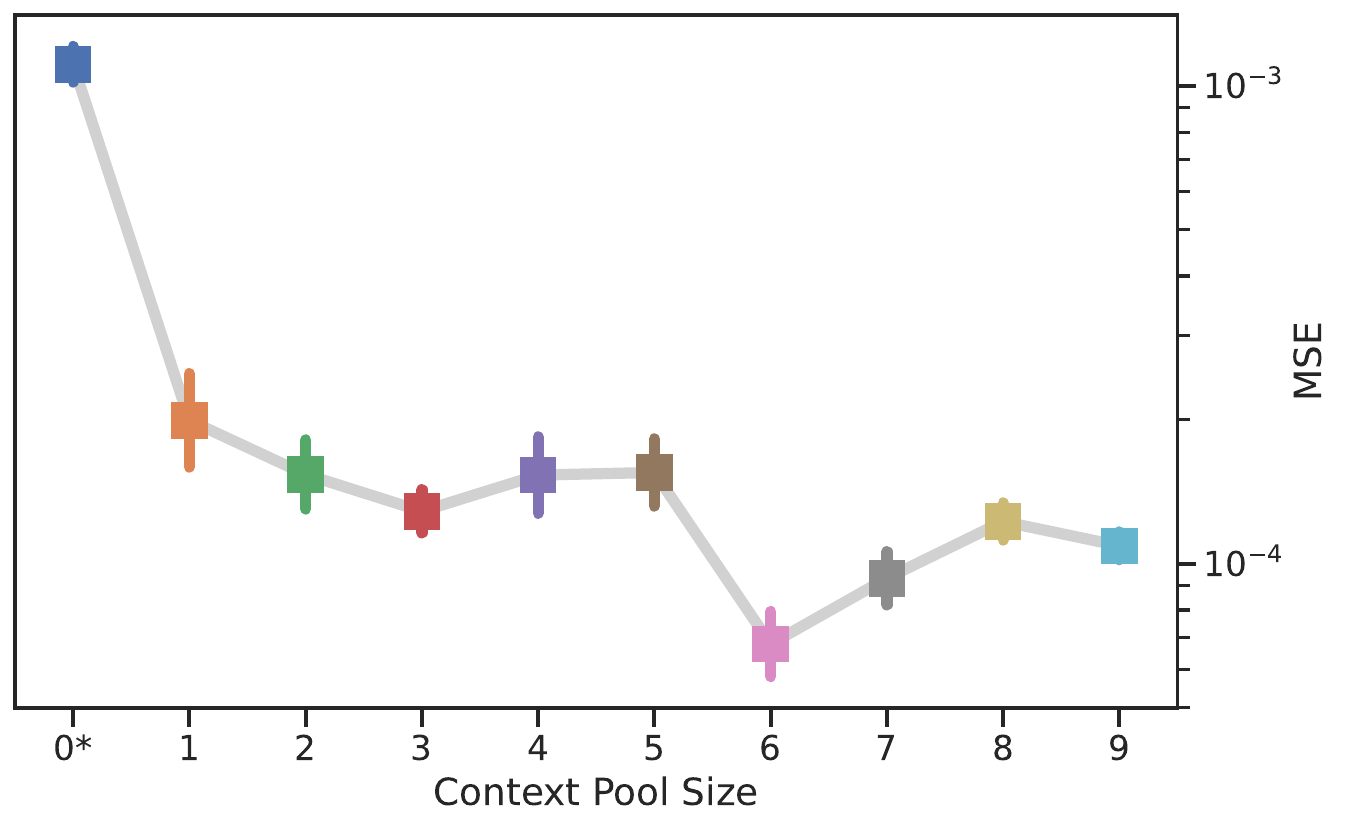}\label{fig:poolsizetimes_T1}}
\caption{Average MSEs for the ablation of the context pool size for \tone, with evaluation carried over many seeds, hence the vertical bars indicating the standard deviations.}
\label{fig:contextpoolsize_T1}
\end{figure}

We dig further into the influence of the context pool size. While its ideal value might be difficult to conclude based on \cref{fig:contextpoolsize}, \cref{fig:contextpoolsize_T1} investigates \tone to highlight how important any value bigger than 1 is (with $p=6$ proving to be exceptionally adequate). Indeed, 0* indicates no context pool was used, and the vector field is evaluated without Taylor expansion. This ablation results in a significant drop in MSE of about one order of magnitude. While computational meta-training time scales linearly with the pool size (see \cref{fig:poolsizetimes}), the various losses do not. Overall, these results suggest that the context pool size $p$ largely remains a hyperparameter that should be tuned for maximum balance of accuracy and training time.

Our motivations for the above conclusion lie in the fact that Taylor expansion is the key to information flowing from one environment to another, along with the concept of ``task relatedness'' or ``context proximity''. Indeed, depending on the pool-filling strategy used, some environments that are prohibitively far from one another may be required to interact in context space. As a result, the Taylor approximation incurs a non-reducible residual error term. To counter this effect, one could restrict the pool size to limit the impact of those far-apart environments since they will be less often sampled from the pool as the training progresses.

\subsection{Restriction of the context size}
\label{app:contextsize}

Rather than limiting the pool size, what if the context vectors themselves were limited? The latent context vectors are the building blocks of NCF, and having shown that they encode useful representations vital for downstream tasks in \cref{subsec:sp}, we now inquire as to how their size influences the overall learning performance. This further provides the opportunity to test the contextual self-modulation process. Indeed, over-parametrization should not degrade the performance of NCFs since context vectors must be automatically kept small and close to each other (in $L^1$ norm) for the Taylor approximations to be accurate.

Like in CoDA \cite{kirchmeyer2022generalizing}, the context size $d_\xi$ is directly related to the parameter count of the model; and limiting the parameter count bears practical importance for computational efficiency and interpretability. Thus, we perform the LV experiment as described in \cref{subsec:lv} with $d_\xi \in \{2, 4, 8, 16, 32, 64, 128, 256, 512, 1024\}$. The results observed in \cref{subfig:spctsize} align with our intuitive understanding of increased expressiveness with bigger latent vectors. Together, \cref{subfig:spctsize,subfig:loss_landscape} shed light on the relationship between $\xi$ and the underlying physical parameters pair $(\beta, \delta)$, suggesting that NCFs would benefit from the vast body of research in representation learning.



We appreciate that while selecting $d_{\xi}$ small is more interpretable (as can be observed by the clustering of the training environment \cref{subfig:loss_landscape}), it is important to choose $d_{\xi}$ sufficiently big. Indeed, using the JVP-based implementation we provide in \cref{app:code}, large context vectors come at a reduced extra cost in both speed and memory.

\begin{figure}[h]
\centering
\subfigure[MSE Loss]{\includegraphics[width=0.42\textwidth]{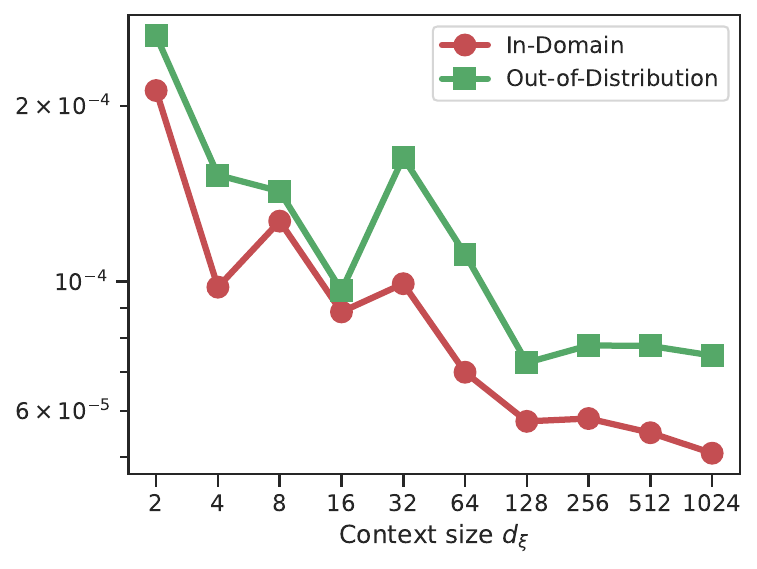}\label{subfig:spctsize}}
\subfigure[Loss landscape]{\includegraphics[width=0.43\textwidth]{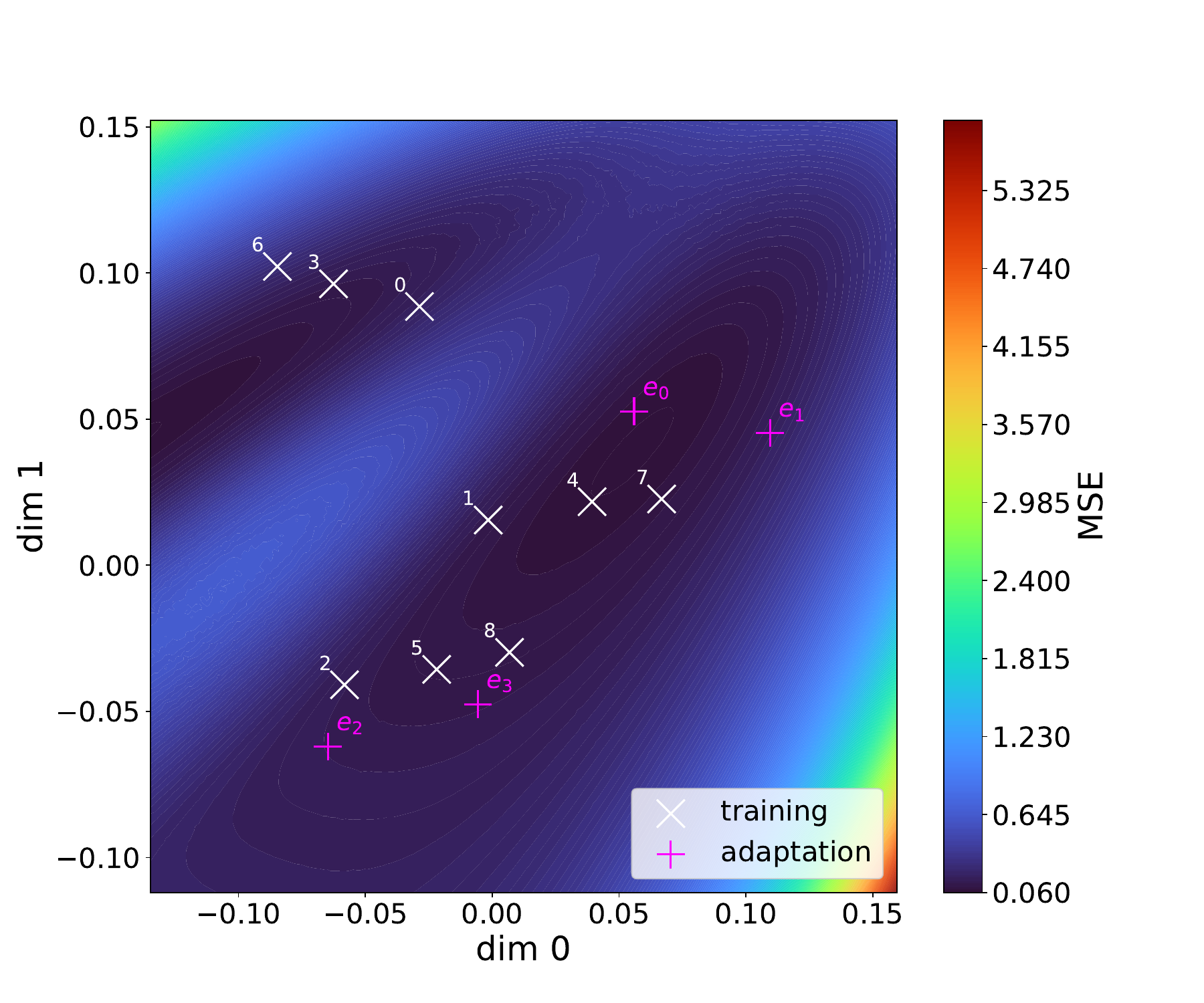}\label{subfig:loss_landscape}}
\caption{(a) InD and OoD evaluation MSEs for the LV problem as the context size is increased. (b) Loss landscape plotted in context vectors space for $d_{\xi} = 2$. The clustering of the training contexts mirrors the physical parameter pairs $(\beta, \delta)$ from \cref{fig:huge_adapt}.}
\label{fig:context_size_experiment}
\end{figure}

{\color{rebutcolor}
\subsection{Variation of the Pooling Strategy}
\label{app:pooling}

The context pool plays an important role in the NCF training process. In order to complement our comments in \cref{subsec:whatsinapool}, we test the effect of choosing various pooling strategies, namely Nearest-First (NF), Random-All (RA), and Smallest-First (SF). In this experiment, we focus on the LV case with $d_{\xi}=2$ due to its linearity and our knowledge of the interpretable behavior of its contexts (see \cref{subsec:interpret}). We train all three strategies with all hyperparameters identical (including the neural network vector field initialization). We monitor the validation loss to select the best model across all 10000 epochs. The experiment is repeated 3 times with different seeds.

\begin{figure}[h]
\centering
\includegraphics[width=.8\textwidth]{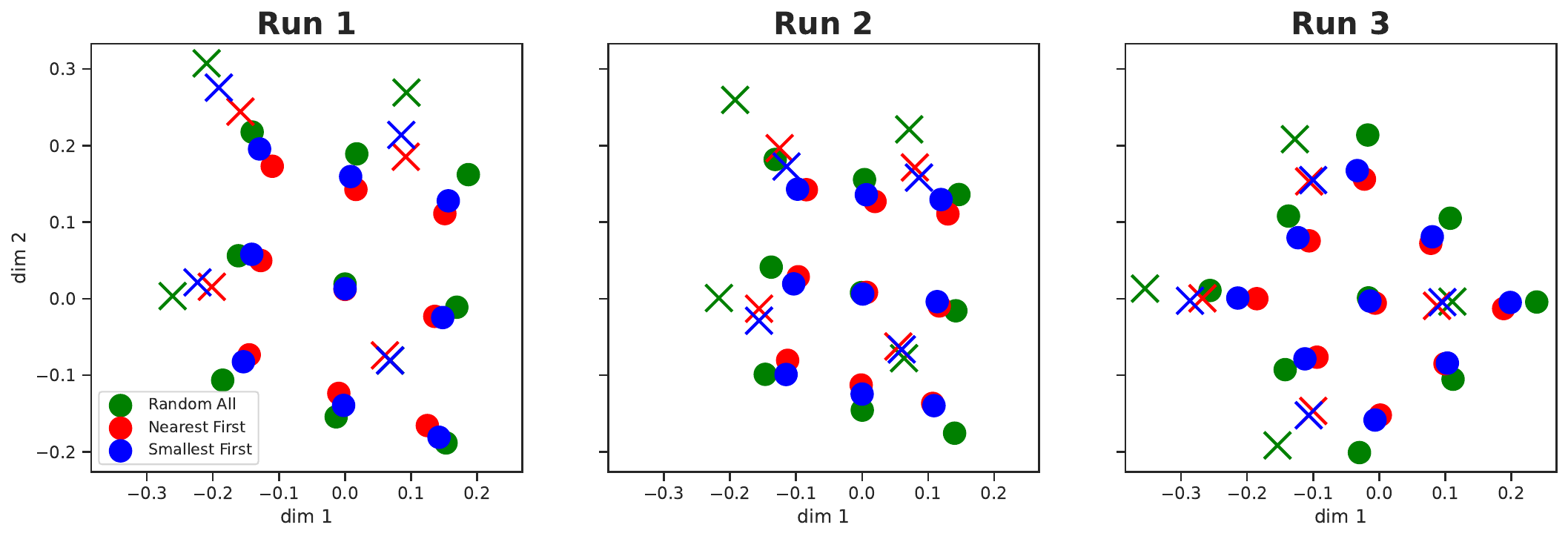}
\caption{\rebut{Both dimensions of the context vector post-training (indicated by dots), and post-adaptation (crosses) across several runs. Given the same neural network initialization, all pooling strategies are able to find the structure mimicking that of the underlying physical parameters observed in \cref{fig:interpret_ncf_coda}.}}
\label{fig:poolingcontexts}
\end{figure}

Although all strategies capture the underlying structure of the LV physical system (see \cref{fig:poolingcontexts}),
\cref{tab:poolingstrategy} indicates that RA and NF have the best performance, contrasting with high metrics and uncertainties displayed by SF. The training and validation dynamics in
\cref{fig:poolinglosses} provide clearer insight into this discrepancy, highlighting an important spread of validation values for SF. Interestingly, \cref{fig:poolinglosses} shows that RA takes much longer to converge during training, despite ultimately producing the best OoD performance. This slow convergence is because very early on, the model is still forced to find commonalities between \emph{all} pairs of environments, despite some pairs being less related than others. All in all, NF is the most balanced strategy on the LV problem since it converges fast, shows signs of increased stability, and produces excellent results InD and OoD.

\begin{table}[h]
\caption{\rebut{Comparison of pooling strategies on the LV problem. The reported mean and standard deviation of training and (sequential) adaptation times are expressed in minutes. All strategies are trained for the same number of epochs across 3 separate runs with identical hyperparameters.}}
\label{tab:poolingstrategy}
\begin{center}
\begin{small}
\begin{tabular}{lccc}
\toprule
& \textsc{\textbf{Random-All (RA)}} & \textsc{\textbf{Nearest-First (NF)}} & \textsc{\textbf{Smallest-First (NF)}}     \\
\midrule
\textsc{Train Time} (mins)  & 22.90 $\pm$ 0.04 & 22.82 $\pm$ 0.03 &  22.94 $\pm$ 0.04  \\
\textsc{Adapt Time} (mins)  & 2.80 $\pm$ 0.01 & 2.80 $\pm$ 0.01 &  2.80 $\pm$ 0.01  \\
\textsc{InD MSE} ($\times  10^{-4}$) & 2.65 $\pm$ 0.16 & 2.52 $\pm$ 0.66 & 5.03 $\pm$ 3.01   \\
\textsc{OoD MSE} ($\times  10^{-4}$) & 2.38 $\pm$ 0.43 & 2.69$\pm$0.35 & 3.34 $\pm$ 0.96   \\
\bottomrule
\end{tabular}
\end{small}
\end{center}
\end{table}

\begin{figure}[h]
\centering
\includegraphics[width=.8\textwidth]{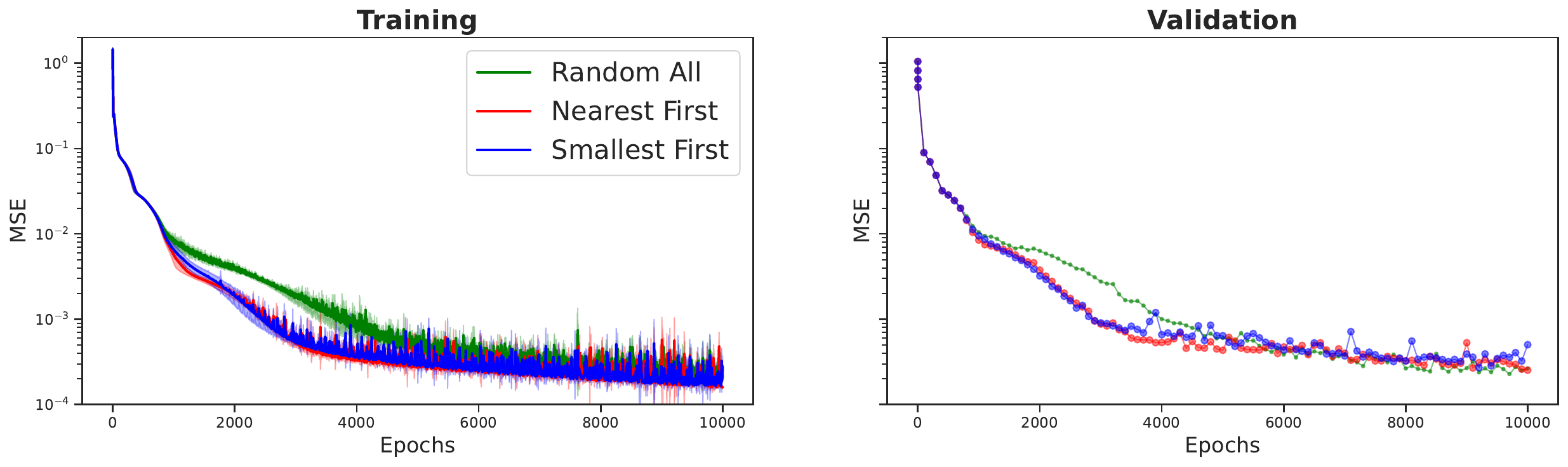}
\caption{\rebut{Loss curves with varying pooling strategies on the LV problem. (Left) Training; (Right) Mean validation losses across 3 runs, color-coded to match training curve labels.}}
\label{fig:poolinglosses}
\end{figure}

}

\subsection{Ablation of the 3-Networks Architecture}
\label{app:3networks}

Another key element of the NCF framework is the 3-networks architecture described in \cref{fig:method}b. In it, the vector field consists of state and context-specific networks that help bring the inputs into the same latent representational space before passing to another network to approximate the local derivative. Like the context size, its ablation directly contributes to a reduction in learnable parameter count.

For this experiment, we consider \tone. We remove the data- and context-specific neural networks in the vector field, and we directly concatenate $x^e_i(t)$, and $\xi$; the result of which is passed to a (single) neural network. In this regard, NCF* (which designates NCF when its 3-networks architecture is removed) is analogous to CAVIA \cite{zintgraf2019fast}. That said, NCF still benefits from the flow of contextual information, which allows for self-modulation.

\begin{figure}[H]
\begin{center}
\centerline{\includegraphics[width=0.6\columnwidth]{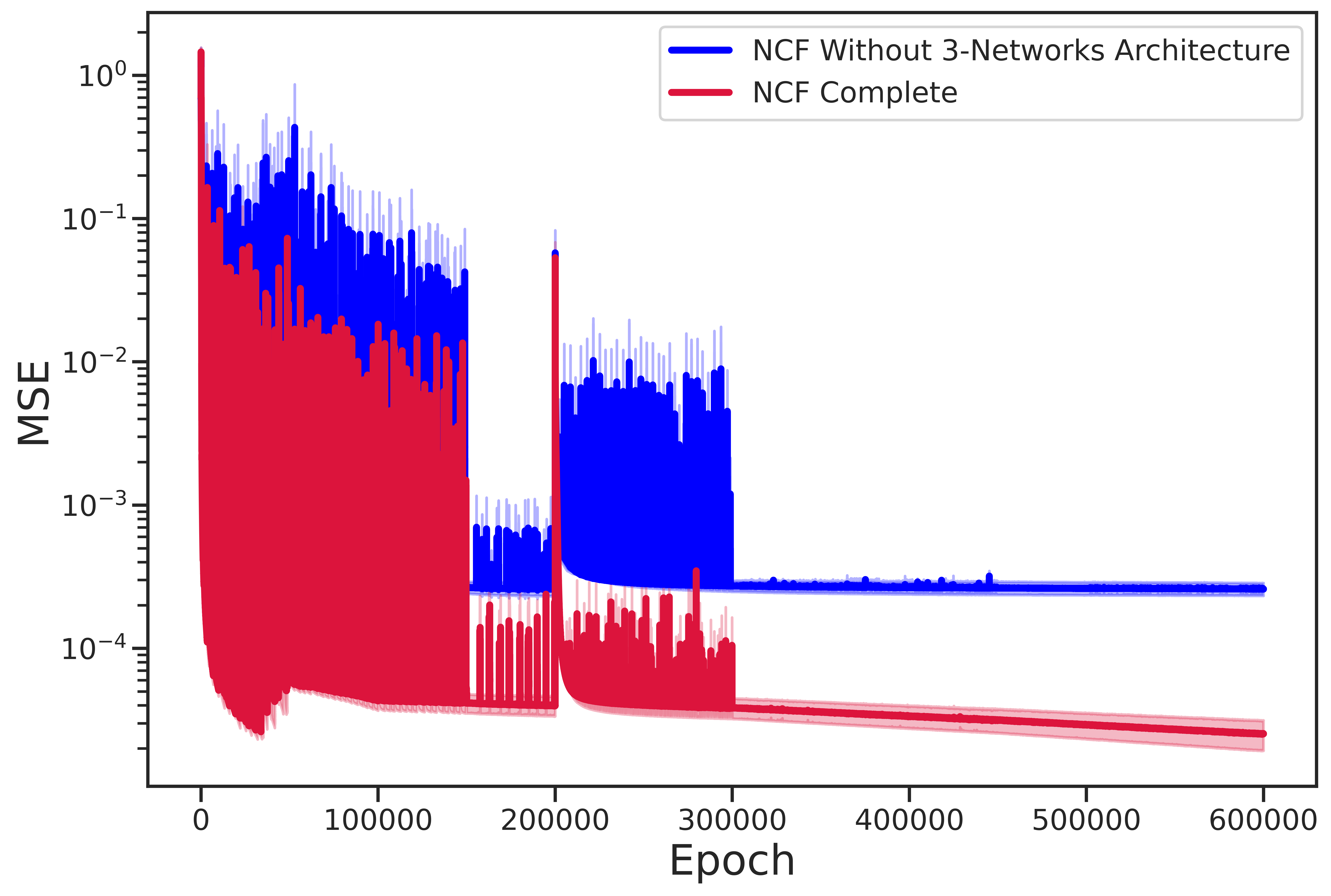}}
\caption{MSE loss curves when training the LV problem on a complete NCF, and on NCF*, the NCF variant deprived of the 3-networks architecture.}
\label{fig:losses_ablation}
\end{center}
\end{figure}


\begin{table*}
\caption{MSE upon ablation of the 3-networks architecture (NCF*) on both SP and LV problems, highlighting difficulties adapting to new environments.}
\label{tab:network_ablation}
\begin{center}
\begin{small}
\begin{sc}
\begin{tabular}{lcc|cc}
    \toprule
    & \multicolumn{2}{c}{SP ($\times 10^{-2}$)} & \multicolumn{2}{c}{LV ($\times 10^{-5}$)} \\
    & In-Domain & Adaptation & In-Domain & Adaptation \\
    \midrule
    NCF     & 11.0 $\pm$ 0.3 & 0.0003 $\pm$ 0.00001 & 6.73 $\pm$ 0.87  & 7.92  $\pm$ 1.04 \\
    NCF*    & 79.8 $\pm$ 0.8  & 1.6402  $\pm$ 0.7539 & 11.03 $\pm$ 1.1  & 69.98 $\pm$ 84.36 \\
    \bottomrule
\end{tabular}
\end{sc}
\end{small}
\end{center}
\end{table*}

This study is of particular significance since, if in addition to foregoing the 3-networks architecture, we eliminated the Taylor expansion step, NCF would turn into the data-controlled neural ODE \cite{massaroli2020dissecting}. So we run both dynamics forecasting problems without the 3-networks architecture, and we report the training MSE in \cref{fig:losses_ablation}. In-domain and adaptation MSEs for both LV and SP problems are reported in \cref{tab:network_ablation}.

While \cref{fig:losses_ablation} highlights a performance discrepancy of nearly 1 order of magnitude during training, the key insight is hidden in the adaptation columns of \cref{tab:network_ablation}. Indeed, the removal of the data- and context-specific networks in the vector field considerably restricts the model's ability to generalize to unseen environments, both for the SP and LV problems. \rebut{We observed in our experiments, similar performance drops when only one of the two state or context networks was removed.} This consolidates the 3-networks architecture as an essential piece of the NCF framework.











\newpage
\section{Experimental Details}
\label{app:detailes}

This section shares crucial details that went into the experiments conducted in \cref{experiments} and expanded upon in \cref{app:additionalresults,app:ablation}. Unless otherwise specified, the hyperparameters in the next 3 paragraphs were applied to produce all figures and tables in this work. 

Each method was tuned to provide an excellent balance of accuracy and efficiency. The main motivation behind our choices was to restrain model sizes both to allow efficient training of all baselines and fair comparison at roughly equal parameter counts. Irreconcilable differences with the baseline adaptation rules made it difficult to perform a systematic comparison, which is why we focused on parameter count as done in \cite{qingeneralizing}; all the while noting that all models involved had sufficient capacity to approximate the problems at hand. For NCF and CAVIA, we counted the number of learnable parameters in the main network being modulated. For CoDA, we counted the root network plus the hypernetwork's learnable parameters as outlined in \cref{eq:coda}.

Our main workstation for training and adaptation was fitted with an Nvidia GeForce RTX 4080 graphics card which was used for the SP, LV, SM, and NS problems. Additionally, we used an RTX 3090 GPU of the same generation for the GO problem, and an NVIDIA A100 Tensor Core GPU for the GS problem, as its CNN-based architecture was more memory-intensive. In addition to the training hardware and deep learning frameworks (JAX for NCFs, and PyTorch for CAVIA and CoDA) that made apple-to-apple comparison challenging, our network architectures varied slightly across methods, as the subsections below highlight. This is part of the reason we decided to focus on total parameter count as the great equalizer. 

To ensure fair comparison, we made sure each model was sufficiently large to represent the task at hand (by comparing to commonly used hyperparameters in the literature, e.g., in \cite{kirchmeyer2022generalizing}). Since an ``epoch'' meant something different for the frameworks under comparison, we simply made sure the models were trained for sufficiently long, i.e., the lowest-performing method for each task was trained for at least as long as the second-best for that task, and its checkpoint with the best validation error was restored at the end.

\subsection{NCFs}

\paragraph{Model architecture} The MLPs used as our context, state and main networks (as depicted in \cref{fig:method}b) typically had a uniform width. When that was not the case, we use the notation $[h_{\text{in}} \rightarrow h_1 \rightarrow ... \rightarrow h_{d-1} \rightarrow h_{\text{out}}]$ to summarize the number of units in each layer of an MLP of depth $d$ for brevity. We used a similar notation to summarize the number of channels in a Convolutional Neural Network (CNN). For instance, on the \textbf{SP}\xspace problem, the context network was the MLP $[256 \rightarrow 64 \rightarrow 64 \rightarrow 64 ]$ (which indicates a context size of $d_{\xi}=256$), the state network was $[2 \rightarrow 64 \rightarrow 64 \rightarrow 64 ]$, and the main network was $[128 \rightarrow 64 \rightarrow 64 \rightarrow 64 \rightarrow 2 ]$, all with Swish activations \cite{ramachandran2017swish}. For the OFA and OPE experiments that do not require contextual information (cf. \cref{tab:pendulum,tab:pendulum_times}), we delete the state and context networks then increase the hidden units of the main network to $156$ to match the NCF parameter count. Below, we follow the same convention to detail the networks used for other problems in this paper's main comparison (see \cref{tab:sota_results}).

\textbullet{} \lv (context) $[1024 \rightarrow 256 \rightarrow 256 \rightarrow 64 ]$; (state) $[2 \rightarrow 64 \rightarrow 64 \rightarrow 64 ]$; (main) $[128 \rightarrow 64 \rightarrow 64 \rightarrow 64 \rightarrow 2 ]$ \textbullet{} \go (context) $[256 \rightarrow 64 \rightarrow 64 \rightarrow 122 ]$; (state) $[7 \rightarrow 122 \rightarrow 122 \rightarrow 122 ]$; (main) $[244 \rightarrow 64 \rightarrow 64 \rightarrow 64 \rightarrow 7 ]$  \textbullet{} \sm (context, state, main) Exact same architectures as with SP above \textbullet{} \bt (context) linear layer with 256 input and 128 output features followed by an appropriate reshaping then a 2D convolution with 8 channels and circular padding to maintain image size, a kernel of size $3\times 3$, and Swish activation\footnote{Except for the number of channels, these were default convolution settings we used throughout. The circular padding enforced the periodic boundary condition observed in the trajectory data.}; (state) appropriate reshaping before 2D convolution with 8 channels; (main) CNN $[16 \rightarrow 64 \rightarrow 64 \rightarrow 64 \rightarrow 2 ]$  \textbullet{} \gs Same network architectures as for the previous BT problem, with the exception that the linear layer in the context network outputted 2048 features \textbullet{} \ns (context) A single linear layer with 202 input and 1024 output features reshaped and concatenated to the state and the grid coordinates before processing; (state) No state network was used for this problem; (main) 2-dimensional Fourier Neural Operator (FNO) \cite{li2020fourier} with 4 spectral convolution layers, each with 8 frequency modes and hidden layers of width 10. Its lifting operator was a $1\times 1$ convolution with 4 input and 10 output channels respectively -- functionally similar to a fully connected layer with weight sharing. The projection operator had two such layers, with 10 inputs, 1 output, and 16 hidden channels in-between.

\paragraph{Training hyperparameters}
As for the propagation of gradients from the loss function to the learnable parameters in the right-hand side of the neural ODEs (\ref{eq:method}), we opted to differentiate through our solvers, rather than numerically solving for adjoint states \cite{chen2018neural}. Since our main solvers were the less demanding RK4 and Dopri5 solvers, we didn't incur the heavy memory cost that Differentiable Programming methods are typically known for \cite{nzoyem2023comparison, kidger2020neural}.

Specifically, we used the following integration schemes: Dopri5 \cite{wanner1996solving} for GO, SM, GS, and BT, with relative and absolute tolerances of $10^{-3}$ and $10^{-6}$ respectively; RK4 for LV with $\Delta t = 0.1$; Explicit Euler for NS with $\Delta t = 1$. Across all NCF variants, we kept the context pool size under 4 to reduce computational workload ($p=2$ for LV and SM, and $p=4$ for LV and GO, and $p=3$ for all PDE problems). We used the Adam optimizers \cite{diederik2014adam} for both model weights and contexts on ODE problems, and Adabelief \cite{zhuang2020adabelief} for PDE problems. Their initial learning rates were as follows: $3\times10^{-4}$ for LV and NS, $10^{-3}$ for GO, BT, and GS, $10^{-4}$ for SM. That learning rate was kept constant throughout the various trainings, except for BT, GS, and NS where it was multiplied by a factor (0.1, 0.5, and 0.1, respectively) after a third of the total number of training steps, and again by the same factor at two-thirds completion. The same initial learning rates were used during adaptation, with the number of iterations typically set to 1500.

As per \cref{eq:bigloss}, the NCF implementation was batched across both environments and trajectories for distributed and faster training. This was in contrast to the two baselines, in which predictions were only batched across trajectories. Using full batches to accelerate training meant that with LV for instance, all 4, 32, or 1 trajectories were used at once depending on the data split (meta-training vs. meta-testing, support vs. query). For regularization of the loss function \cref{eq:bigloss}, we set $\lambda_1 = 10^{-3}$ for all problems; but $\lambda_2 = 0$ for ODE problems and $\lambda_2 = 10^{-3}$ for PDE problems. For \ttwo we always used a proximal coefficient $\beta=10$ except for the LV where $\beta=100$. 

As for the simple pendulum SP problem used in this work, it leveraged \tone with $p=4$ and $d_{\xi}=256$. The integrator was Dopri5 with default tolerances, a proximal coefficient $\beta=10^2$, a constant learning rate of $10^{-4}$, and 12000 epochs. The OFA and OPE comparisons used the same hyperparameters, but with 6000 epochs for OPE and 2000 for OFA. For the other problems highlighted in \cref{tab:sota_results}, the remaining crucial hyperparameters (not stated above) are defined in \cref{tab:hpsncf1,tab:hpsncf2}.

\begin{table}[h!]
\centering
\caption{Hyperparameters for \tone }
\label{tab:hpsncf1}
\begin{tabular}{>{\bfseries}l*{6}{>{\centering\arraybackslash}p{1.3cm}}}
\toprule
\textbf{Hyperparameters} & \textbf{LV} & \textbf{GO} & \textbf{SM} & \textbf{BT} & \textbf{GS} & \textbf{NS} \\
\midrule
Context size $d_\xi$       & 1024      &  256   &  256   &  256  &  256  &  202  \\
Pool-filling strategy              &  NF   &  RA   &  RA  &  NF  & RA  & NF \\
\# Epochs        &  10000  &  2400  &  24000  &  10000  &  10000 & 5000 \\
\bottomrule
\end{tabular}
\label{tab:hyperparameters_ncf1}
\end{table}

\begin{table}[h!]
\centering
\caption{Hyperparameters for \ttwo}
\label{tab:hpsncf2}
\begin{tabular}{>{\bfseries}l*{6}{>{\centering\arraybackslash}p{1.3cm}}}
\toprule
\textbf{Hyperparameters} & \textbf{LV} & \textbf{GO} & \textbf{SM} & \textbf{BT} & \textbf{GS} & \textbf{NS} \\
\midrule
Context size $d_\xi$        & 1024    &  256   &  256   &  256  &  256  &  202  \\
Pool-filling strategy             &  NF   &  NF   &  RA  &  NF  & RA  & NF \\
\# Inner iterations\footnotemark        &  25  &  10  &  10  &  20  &  20  & 25 \\
\# Outer iterations        &  250  &  1000  &  1500  &  1000  & 700 & 250 \\
\bottomrule
\end{tabular}
\label{tab:hyperparameters_ncf2}
\end{table}

\footnotetext{The inner iterations used in \ttwo are not to be confused with the inner gradient update steps in CAVIA, since NCF performs alternating rather than bi-level optimization.}

\subsection{CoDA}

The reference CoDA implementation from \cite{kirchmeyer2022generalizing} was readily usable for most problems. We adapted its data generation process to incorporate two new benchmarks introduced in this paper (the SM and BT problems). 

\paragraph{Model architecture} Neural networks without physics priors were employed throughout, all with the Swish activation function. \textbullet{} \lv: 4-layers MLP with 224 neurons per hidden layers (width) \textbullet{} \go: 4-layers MLP with width 146 \textbullet{} \sm:  4-layers MLP with width 90 \textbullet{} \bt: 4-layers ConvNet with 46 hidden convolutional filters and $3\times3$ kernels. Its output was rescaled by $10^{-4}$ to stabilize rollouts \textbullet{} \gs: Same ConvNet as BT, but with 106 filters. \textbullet{} \ns: 2-dimensional Fourier Neural Operator (FNO) \cite{li2020fourier} with 4 spectral convolution layers, each with 8 frequency modes and hidden layers of width 10. Its lifting operator was a single fully connected layer, while its projection operator had two such layers, with 16 hidden neurons.

\paragraph{Training hyperparameters} Using \texttt{TorchDiffEq} \cite{torchdiffeq}, CoDA backpropagates gradients through the numerical integrator. We used the same integration schemes as NCF above. We stabilized its training by applying exponential Scheduled Sampling \cite{bengio2015scheduled} with constant $k=0.99$ and initial gain $\epsilon_0=0.99$ updated every 10 epochs (except for the \ns problem updated every 15 epochs). The Adam optimizer with a constant learning rate of $10^{-4}$ was used. The batch size and the number of epochs are given in the \cref{tab:hpscoda}.

\begin{table}[h!]
\centering
\caption{Hyperparameters for CoDA}
\label{tab:hpscoda}
\begin{tabular}{>{\bfseries}l*{6}{>{\centering\arraybackslash}p{1.3cm}}}
\toprule
\textbf{Hyperparameters} & \textbf{LV} & \textbf{GO} & \textbf{SM} & \textbf{BT} & \textbf{GS} & \textbf{NS} \\
\midrule
Context size $d_{\xi}$              & 2    & 2    & 2   & 2  & 2   & 2 \\
Minibatch size              & 4    & 32    & 4   & 1  & 1   & 16 \\
\# Epochs        & 40000   & 40000   & 12000   & 10000 & 120000     & 30000\\
\bottomrule
\end{tabular}
\label{tab:hyperparameters_coda}
\end{table}

We note that in line with CoDA's fundamental low-rank assumption, we maintained a context size $d_\xi = 2$ throughout this work since no more than 2 parameters varied for any given problem.

\subsection{CAVIA}

The reference implementation of CAVIA-Concat \cite{zintgraf2019fast} required substantial modifications to fit dynamical systems. Importantly, we incorporated the \texttt{TorchDiffEq} \cite{torchdiffeq} open-source package and we adjusted other hyperparameters accordingly to match NCF and CoDA on parameter count and other relevant aspects for fair comparison. The performance of CAVIA-Concat was heavily dependent on the context size, which we adjusted depending on the problem.

\paragraph{Model architecture} Like CoDA, the Swish activation was used for all neural networks \textbullet{} \lv: 4-layers MLP with width 278 \textbullet{} \go: 4-layers MLP with width 168 \textbullet{} \sm: 4-layers MLP with width 84  \textbullet{} \bt: 7-layers ConvNet with 46 hidden convolutional filters and $3\times3$ kernels. Its outputs was rescaled by $0.1$ to stabilize rollouts \textbullet{} \gs: 4-layers ConvNet with 184 hidden convolutional filters and $3\times3$ kernels. Its outputs was rescaled by $0.1$ to stabilize rollouts. \textbullet{} \ns: Same FNO as with NCF above, including the single-layer context network with 1024 output neurons.

\paragraph{Training hyperparameters} We used full batches to accelerate training. We use the Adam optimizer with a learning rate of 0.001 for the meta-update step. The single inner update step had a learning rate of 0.1. The number of iterations is given in \cref{tab:hpscavia}.

\begin{table}[h!]
\centering
\caption{Hyperparameters for CAVIA}
\label{tab:hpscavia}
\begin{tabular}{>{\bfseries}l*{6}{>{\centering\arraybackslash}p{1.3cm}}}
\toprule
\textbf{Hyperparameters} & \textbf{LV} & \textbf{GO} & \textbf{SM} & \textbf{BT} & \textbf{GS} & \textbf{NS} \\
\midrule
Context size $d_\xi$              &   1024  &  256   &  256  & 64  &  1024  & 202 \\
\# Iterations        &  200  &  100  &  1200  & 500  &  50  & 1200 \\
\bottomrule
\end{tabular}
\label{tab:hyperparameters_cavia}
\end{table}

\newpage
\section{Example Implementation of NCFs}
\label{app:code}

A highly performant JAX implementation \cite{jax2018github} of our algorithms is available at \url{https://github.com/ddrous/ncflow} \footnote{A PyTorch codebase is equally made available at \url{https://github.com/ddrous/ncflow-torch}.}. We provide below a few central pieces of our codebase using the ever-growing JAX ecosystem, in particular Optax \cite{deepmind2020jax} for optimization, and Equinox \cite{kidger2021equinox} for neural network definition.

\paragraph{Vector field}

The vector field takes center stage when modeling using ODEs. It is critical to use JVPs, since we never want to materialize the Jacobian (the context size can be prohibitively large). The Jacobian-Vector Product primitive from Equinox \cite{kidger2021equinox} \texttt{filter\_jvp} as illustrated below, coupled with \texttt{jit}-compilation, enabled fast runtimes for all problems. 

{
\tiny

\begin{lstlisting}[language=Python, caption=Second-order Taylor expansion in the NCF vector field]
class ContextFlowVectorField(eqx.Module):
    physics: eqx.Module
    augmentation: eqx.Module

    def __init__(self, augmentation, physics=None):
        self.augmentation = augmentation
        self.physics = physics

    def __call__(self, t, x, ctxs):
        ctx, ctx_ = ctxs
        # ctx = \xi^e, while ctx_ = \xi^j

        if self.physics is None:
            vf = lambda xi: self.augmentation(t, x, xi)
        else:
            vf = lambda xi: self.physics(t, x, xi) + self.augmentation(t, x, xi)

        gradvf = lambda xi_: eqx.filter_jvp(vf, (xi_,), (ctx-xi_,))[1]
        scd_order_term = eqx.filter_jvp(gradvf, (ctx_,), (ctx-ctx_,))[1]

        return vf(ctx_) + 1.5*gradvf(ctx_) + 0.5*scd_order_term
\end{lstlisting}
}

\paragraph{Loss function} \cref{eq:bigloss} provides a structured summation layout particularly suited for a function transformation like JAX's \texttt{vmap}. During implementation, the loss functions can be defined in two stages, the innermost summation term along with \cref{eq:loss} making up the first stage; while the two outermost summations of \eqref{eq:bigloss} make up the second stage. Once the first stage is complete, the second is relatively easy to implement. We provide a vectorized implementation of the first stage below.

\begin{lstlisting}[language=Python, caption=Inner NCF loss function with vectorization support]
def loss_fn_ctx(model, trajs, t_eval, ctx, all_ctx_s, key):
    """ Inner loss function Eq. 9 """

    ind = jax.random.permutation(key, all_ctx_s.shape[0])[:context_pool_size]
    ctx_s = all_ctx_s[ind, :]                   # construction of the context pool P

    batched_model = jax.vmap(model, in_axes=(None, None, None, 0))

    trajs_hat, nb_steps = batched_model(trajs[:, 0, :], t_eval, ctx, ctx_s)
    new_trajs = jnp.broadcast_to(trajs, trajs_hat.shape)

    term1 = jnp.mean((new_trajs-trajs_hat)**2)  # reconstruction error

    term2 = jnp.mean(jnp.abs(ctx))              # context regularisation

    term3 = params_norm_squared(model)          # weights regularisation

    loss_val = term1 + 1e-3*term2 + 1e-3*term3

    return loss_val
\end{lstlisting}



\newpage
\section{Trajectories Visualization}

\begin{figure}[H]
\begin{center}
\centerline{\includegraphics[width=.35\columnwidth,trim=0mm 3mm 0mm 0mm]{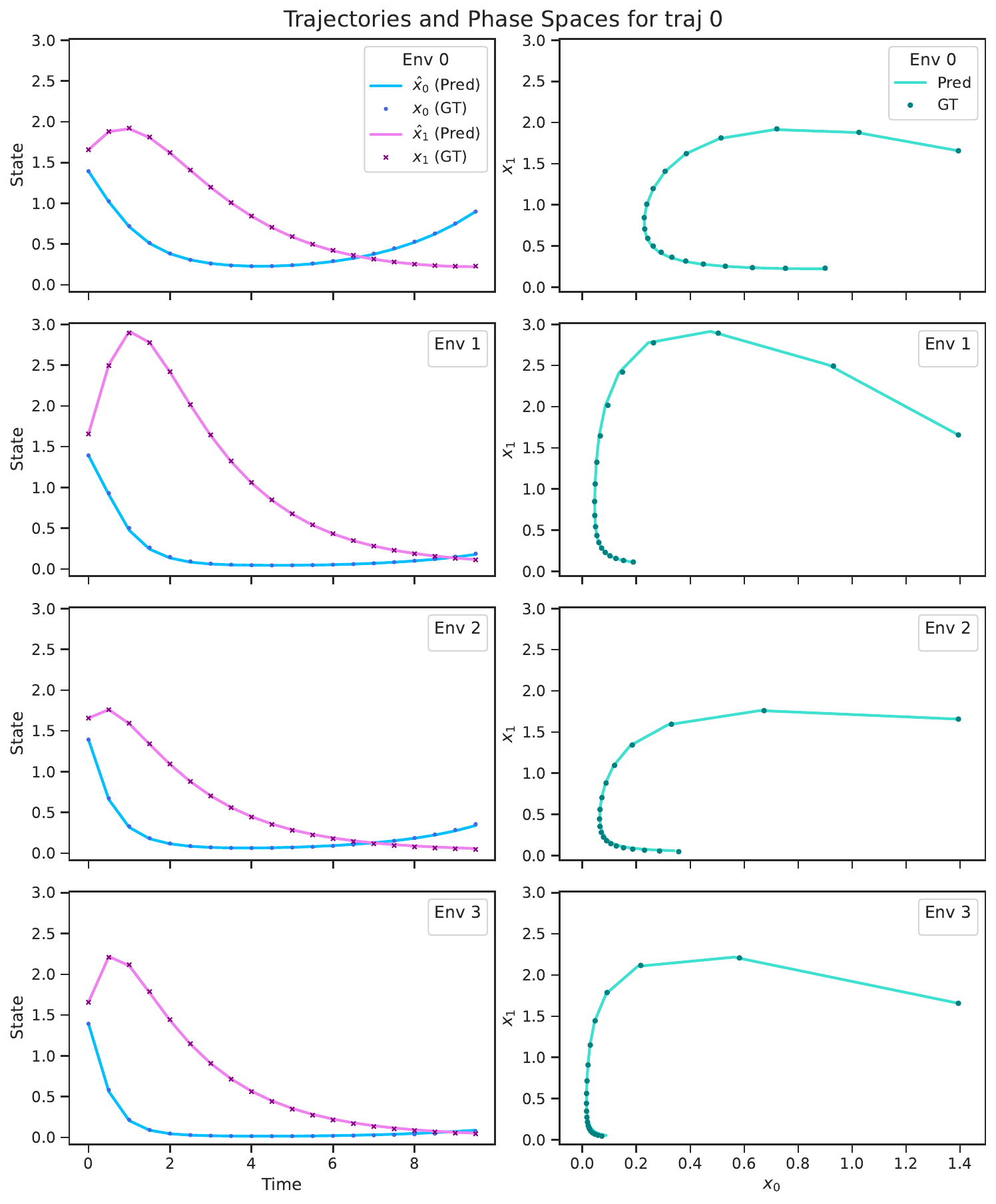} \hspace*{1cm} \includegraphics[width=.1885\columnwidth]{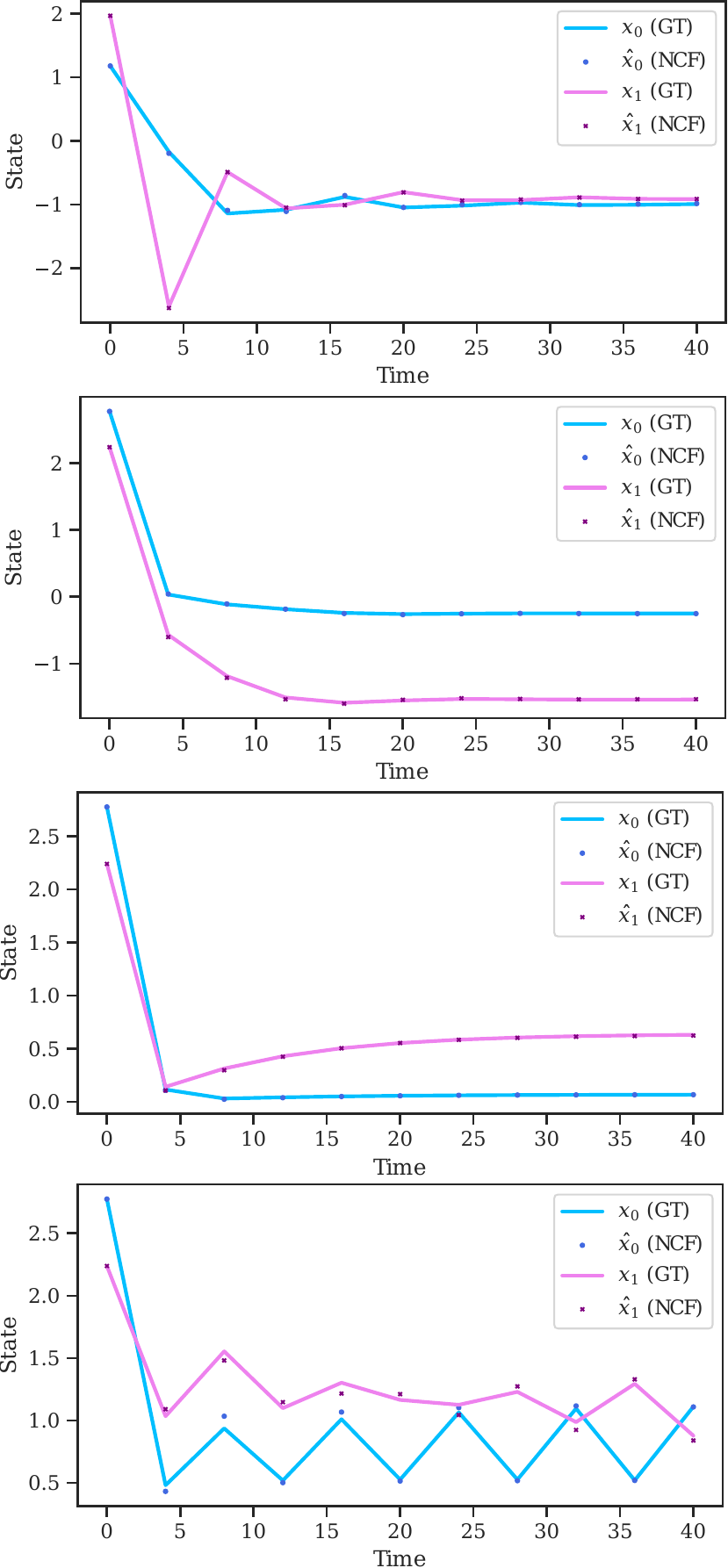}}
\caption{\textbf{(Right)} Visualizing the first ground truth and predicted testing trajectories and phase spaces in all 4 adaptation environments of the LV problem. The initial condition is the same across all 4 adaptation environments. \textbf{(Left)} \rebut{Visualizing the first ground truth testing 
trajectory in 4 meta-training environments found in various attractors of the SM problem: the first in (L1), the second and third in (E), and the fourth in (L2).}}
\label{fig:trajvis_lv}
\end{center}
\end{figure}

\begin{figure}[H]
\begin{center}
\centerline{\includegraphics[width=.9\columnwidth]{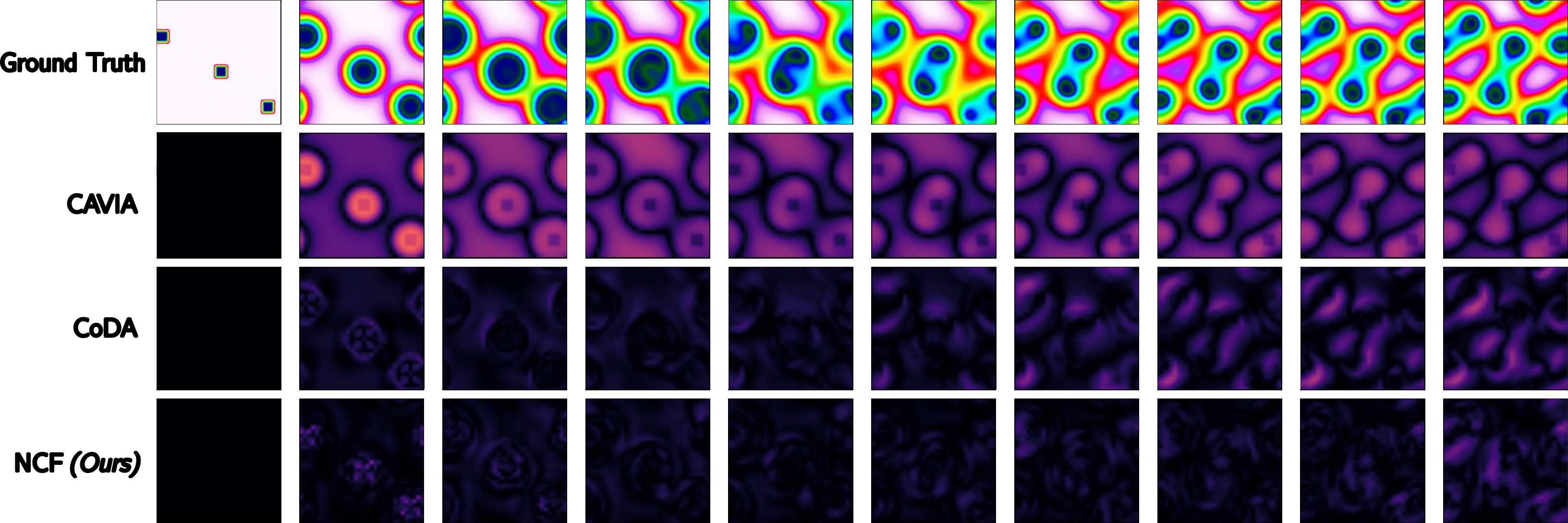}}
\caption{Sample absolute reconstruction error of a trajectory during adaptation of the Gray-Scott system with \ttwo. Trajectories begin at $t=0$ (left) and end at $t=400$. This describes the sole trajectory in the fourth adaptation environment.}
\label{fig:trajvis}
\end{center}
\end{figure}
\vspace*{-0.8cm}

\begin{figure}[H]
\begin{center}
\centerline{\includegraphics[width=.9\columnwidth]{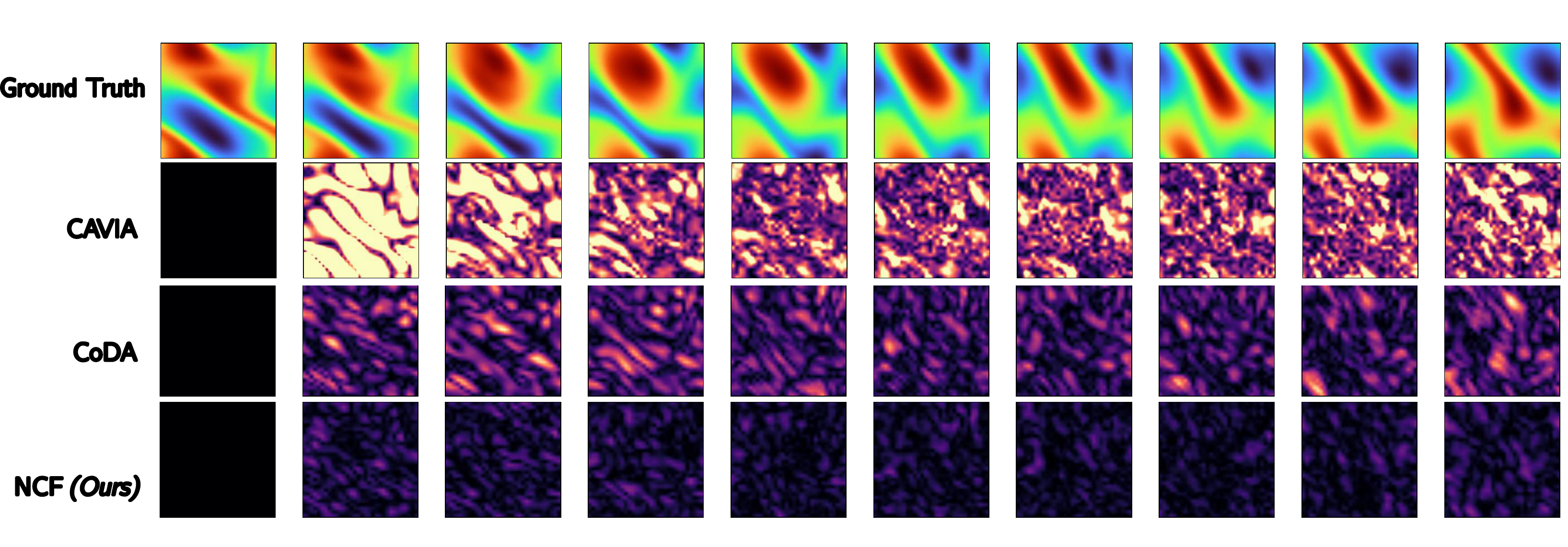}}
\caption{Sample absolute reconstruction error of a trajectory during adaptation of the Navier-Stokes system with \ttwo. Trajectories begin at $t=0$ (left) and end at $t=10$. The viscosity for the reported environment is $\nu = 1.15\times 10^{-3}$.}
\label{fig:trajvis_ns}
\end{center}
\end{figure}

\end{document}